\documentclass[letterpaper]{article}
\usepackage{aaai2027arxiv}
\usepackage[hyphens]{url}
\usepackage{graphicx}
\usepackage{natbib}
\usepackage{caption}
\nocopyright

\usepackage{amsmath}
\usepackage{amssymb}
\usepackage{amsthm}
\usepackage{mathtools}
\usepackage{bm}
\usepackage{algorithm}
\usepackage{algorithmic}
\usepackage{booktabs}
\usepackage{array}
\usepackage{float}
\usepackage[many]{tcolorbox}
\usepackage{colortbl}
\usepackage{enumitem}
\allowdisplaybreaks

\usepackage{titlesec}
\usepackage{titletoc}
\usepackage{fancyhdr}

\usepackage[colorlinks,linkcolor=black,citecolor=black,urlcolor=black,bookmarksnumbered,unicode]{hyperref}
\hypersetup{pdftitle={One Knob to Rule Them All: A Unified Optimal Transport View of Cold-Start Active Learning}}

\definecolor{suppnavy}{HTML}{244B66}
\definecolor{suppteal}{HTML}{26766F}
\definecolor{supplight}{HTML}{F2F6F8}
\definecolor{suppmid}{HTML}{DCE8ED}
\definecolor{supprule}{HTML}{9CB5C1}
\definecolor{suppgray}{HTML}{5E6870}
\definecolor{thmblue}{HTML}{244B66}
\definecolor{tabhead}{HTML}{E8F0F3}
\definecolor{tabours}{HTML}{D5E7EB}
\definecolor{accentblue}{HTML}{244B66}
\definecolor{gainpos}{HTML}{3B8C78}
\definecolor{gainposdk}{HTML}{1D5B4B}
\definecolor{gainneg}{HTML}{C76057}
\definecolor{gainnegdk}{HTML}{7A2924}
\definecolor{bestcell}{HTML}{F6EBC9}

\titleformat{\section}{\Large\bf\centering}{\thesection}{0.7em}{}
\titlespacing*{\section}{0pt}{2.0ex plus 0.5ex minus .2ex}{3pt plus 2pt minus 1pt}
\titleformat{\subsection}{\large\bf\raggedright}{\thesubsection}{0.7em}{}
\titlespacing*{\subsection}{0pt}{2.0ex plus 0.5ex minus .2ex}{3pt plus 2pt minus 1pt}
\titleformat{\subsubsection}[runin]{\normalsize\bf}{}{0pt}{}
\titlespacing*{\subsubsection}{0pt}{6pt plus 2pt minus 1pt}{1em}
\titleformat{\subparagraph}[runin]{\normalsize\bf}{}{0pt}{}
\titlespacing*{\subparagraph}{0pt}{6pt plus 2pt minus 1pt}{1em}
\titleformat{\paragraph}[runin]{\normalsize\bf}{}{0pt}{}
\titlespacing*{\paragraph}{0pt}{6pt plus 2pt minus 1pt}{1em}

\theoremstyle{plain}
\newtheorem{theorem}{Theorem}
\newtheorem{lemma}[theorem]{Lemma}
\newtheorem{proposition}[theorem]{Proposition}
\newtheorem{corollary}[theorem]{Corollary}
\theoremstyle{definition}
\newtheorem{definition}[theorem]{Definition}
\newtheorem{assumption}[theorem]{Assumption}
\theoremstyle{remark}
\newtheorem{remark}[theorem]{Remark}

\theoremstyle{plain}
\newtheorem{theoremApp}{Theorem}[section]
\newtheorem{lemmaApp}[theoremApp]{Lemma}
\newtheorem{propositionApp}[theoremApp]{Proposition}
\newtheorem{corollaryApp}[theoremApp]{Corollary}
\theoremstyle{definition}
\newtheorem{definitionApp}[theoremApp]{Definition}
\newtheorem{assumptionApp}[theoremApp]{Assumption}
\theoremstyle{remark}
\newtheorem{remarkApp}[theoremApp]{Remark}

\tcbset{
  graybox/.style={
    breakable,
    colback=black!10,
    colframe=white,
    width=\dimexpr\linewidth+10pt\relax,
    enlarge left by=-5pt,
    enlarge right by=-5pt,
    boxrule=0pt,
    left=0pt,right=0pt,top=0pt,bottom=0pt,
    sharp corners
  },
  thmbox/.style={
    breakable,
    enhanced jigsaw,
    colback=supplight,
    colframe=suppnavy!72,
    boxrule=0pt,
    leftrule=2.2pt,
    arc=1.5pt,
    outer arc=1.5pt,
    boxsep=0pt,
    left=7pt,
    right=7pt,
    top=6pt,
    bottom=6pt,
    before skip=0.9\topsep,
    after skip=0.9\topsep,
  }
}
\newcommand{\csalwraptcb}[2]{
  \expandafter\let\csname csal@wrap@#1\expandafter\endcsname\csname #1\endcsname
  \expandafter\let\csname csal@wrap@end#1\expandafter\endcsname\csname end#1\endcsname
  \renewenvironment{#1}[1][]{
    \begin{tcolorbox}[#2]
    \if\relax\detokenize{##1}\relax
      \csname csal@wrap@#1\endcsname
    \else
      \csname csal@wrap@#1\endcsname[##1]
    \fi
  }{
    \csname csal@wrap@end#1\endcsname
    \end{tcolorbox}
  }
}
\csalwraptcb{assumption}{graybox, boxsep=5pt, before skip=\topsep, after skip=\topsep}
\csalwraptcb{theorem}{graybox, boxsep=3pt, before skip=4pt, after skip=4pt}
\csalwraptcb{lemma}{graybox, boxsep=5pt, before skip=\topsep, after skip=\topsep}

\newenvironment{restated}[3]{
  \begin{tcolorbox}[thmbox]
  \noindent\textbf{#1~#2 (#3, restated).}\;\itshape\ignorespaces
}{
  \end{tcolorbox}
}

\newcommand{\R}{\mathbb{R}}
\newcommand{\E}{\mathbb{E}}
\newcommand{\KL}{\mathrm{KL}}

\newcommand{\OT}{\mathrm{OT}}
\newcommand{\eps}{\varepsilon}
\newcommand{\Hreg}{H_{\mathrm{reg}}}
\DeclareMathOperator*{\argmin}{arg\,min}
\DeclareMathOperator*{\argmax}{arg\,max}
\newcommand{\cmark}{\checkmark}
\newcommand{\xmark}{\ensuremath{\times}}

\newcommand{\mref}[1]{\ref{#1}}
\newcommand{\meqref}[1]{\eqref{#1}}
\newcommand{\mthm}[1]{\texorpdfstring{\hyperref[#1]{\textcolor{thmblue}{Theorem~\ref*{#1}}}}{Theorem~\ref{#1}}}
\newcommand{\mthmp}[1]{Theorem~\ref{#1}}
\newcommand{\clem}[1]{\hyperref[#1]{Lemma~\ref*{#1}}}
\newcommand{\cassum}[1]{\hyperref[#1]{Assumption~\ref*{#1}}}
\newcommand{\cdefn}[1]{\hyperref[#1]{Definition~\ref*{#1}}}

\newcommand*{\annotref}[1]{&\text{\footnotesize\textcolor{black!50}{\big(#1\big)}}&}
\let\annotnl\annot

\newcommand{\rowhead}{\rowcolor{tabhead}}
\newcommand{\rowours}{\rowcolor{tabours}}

\newcommand{\gain}[2]{\cellcolor{gainpos!#1!white}\textbf{\textcolor{gainposdk}{#2}}}
\newcommand{\drop}[2]{\cellcolor{gainneg!#1!white}\textbf{\textcolor{gainnegdk}{#2}}}

\title{One Knob to Rule Them All: A Unified Optimal Transport View of Cold-Start Active Learning}
\author{
    Ning Zhu\textsuperscript{\rm 1},
    Xiaochuan Ma\textsuperscript{\rm 2},
    Juntao Xu\textsuperscript{\rm 1},
    Jingze Liang\textsuperscript{\rm 1}, \\
    Mengfei Zhao,
    An Chen \textsuperscript{\rm 1},
    Liang-Jian Deng\textsuperscript{\rm 3}\thanks{Corresponding author.}    
}
\affiliations{
    \textsuperscript{\rm 1}Glasgow College, University of Electronic Science and Technology of China\\
    \textsuperscript{\rm 2}School of Mechanical and Electrical Engineering, University of Electronic Science and Technology of China\\
    \textsuperscript{\rm 3}School of Mathematical Sciences, University of Electronic Science and Technology of China
}

\begin{document}

\maketitle

\begin{abstract}
\textbf{Cold-Start Active Learning} (CSAL) aims to select a valuable
subset from an unlabeled pool without any prior knowledge or human assistance. Existing
methods take diverse routes based on typicality, coverage, or
diversity. Each rests on its own inductive bias and therefore performs
well on some tasks yet poorly on others. We argue that the real
challenge is not to design yet another selection heuristic, but to
make CSAL adapt automatically to the data and task at
hand. To this end, we revisit CSAL through the lens of optimal
transport. First, we propose a generalized transport selection
framework that reveals the shared allocation structure of existing
methods and exactly subsumes representative formulations. Second, we introduce a theoretical analysis that
characterizes the trade-off controlled by entropic regularization and
establishes a task-agnostic minimax bound for cold-start selection.
These results provide a principled foundation for adapting the
regularization strength to the unlabeled data. Third, we derive a
data-adaptive regularization rule and present a novel Sinkhorn-based
CSAL algorithm, termed \textbf{$\eps$-Adaptive Selection ($\eps$-AS)}.
Extensive experiments on six public datasets and multiple annotation
budgets show that $\eps$-AS consistently achieves state-of-the-art
performance. On ImageNet-1k, it improves the average accuracy over
ActiveFT by 1.29\% while reducing selection time by 56.2\%. Code will be released at \underline{\url{https://github.com/Z-yiwei/OT-CSAL}}.
\end{abstract}

\section{I. Introduction}
\label{sec:intro}

\begin{figure}[t]
\centering
\includegraphics[width=\columnwidth]{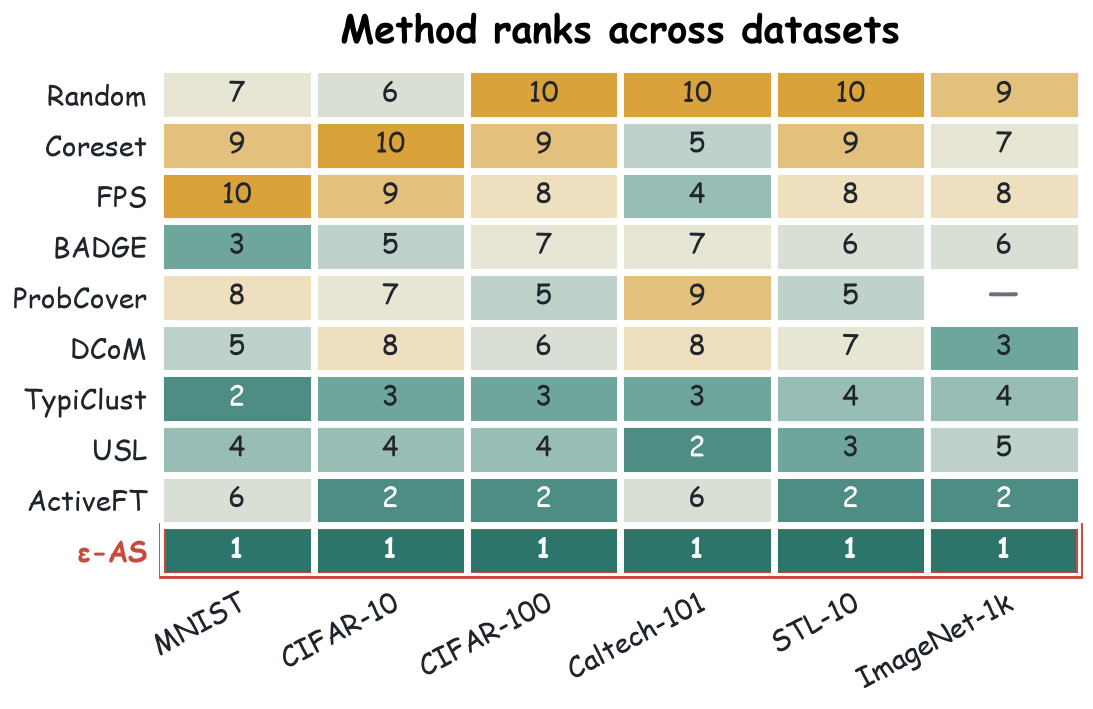}
\caption{\textbf{Fixed geometric selection rules do not transfer
consistently across datasets.} Each cell gives one method's accuracy
rank on a dataset. Baseline ranks vary greatly while
$\eps$-AS ranks first on all six datasets.}
\label{fig:fragility}
\end{figure}

Modern machine learning owes much of its success to large labeled
datasets~\cite{Peng2025unsupervised}. However, obtaining annotations is often expensive and slow~\cite{Ren2025deep}.
Warm-start \textbf{active learning} (AL) reduces this burden by querying
only the most informative samples~\citep{settles2009survey}. These
methods assume a partially trained task model and acquire labels over
many rounds of retraining and selection
~\citep{settles2009survey,ren2021survey}. This setting is fragile in two ways.
First, the multi-round loop repeatedly consumes expert effort and compute. 
Second, randomness in the initial labeled subset can lead to
suboptimal performance and instability
~\citep{hacohen2022typiclust,yehuda2022probcover}.

To this end, \textbf{Cold-Start AL} (CSAL) addresses a more demanding
setting in which the learner selects all informative samples in
one round without prior knowledge. This removes repeated
human interaction from the acquisition process, but it also removes the task feedback on which warm-start methods rely. Therefore, the learner must instead make its decision from the geometry of pretrained features alone~\cite{Zhu2024adversarial,zhu2025frequency}.

Existing CSAL methods follow three main routes. \textbf{Typicality-based}
methods select points that represent dense local regions
\citep{hacohen2022typiclust,wang2022usl}. \textbf{Coverage-based methods} spread
the annotation budget across feature space
\citep{yehuda2022probcover,sener2018coreset,mishal2024dcom}.
\textbf{Diversity-regularized continuous} methods optimize a set of separated
prototypes and then convert them into queries
\citep{xie2023activeft,mahmood2022lowbudget}. Each route commits to a different geometric preference. Typicality may
overlook sparse regions, aggressive coverage may favor atypical
points, and continuous diversity can impose a preference that does
not match the downstream task.

The CSAL setting provides no labels with which to choose among
these preferences. As a result, a method can perform strongly on one
dataset and poorly on another. As shown in Figure~\ref{fig:fragility}, the rankings of existing methods vary markedly across
datasets. We argue that the key problem is not a lack of selection heuristics. It is the use of a fixed geometric bias when
the appropriate balance between local representativeness and global
coverage changes with the pool and the budget. A robust CSAL method should adapt that balance from the unlabeled data itself.

In this paper, we take a first step by asking what the three routes have in common.
Despite their different objectives, they all allocate responsibility
for the unlabeled pool to a limited set of representatives and then
turn those representatives into annotation queries. We formalize this
shared structure through a generalized transport selection framework.
The framework exactly recovers the evaluated formulations of
\textbf{Typicality-based} TypiClust, \textbf{Coverage-based} ProbCover and \textbf{Diversity-based} ActiveFT while retaining their different
selection rules. It therefore provides a unified view for the three routes without treating them as identical algorithms. 

The unified view also makes it possible to ask a more principled
question: how should the selection behavior change with the data
rather than remain fixed by the choice of method? We answer this
question theoretically from two complementary perspectives. First, we
show that entropic regularization traces a monotone trade-off between
geometric fit and diffuse allocation. Second, we establish a
task-agnostic minimax bound that identifies the geometric resolution
governing the difficulty of CSAL. These
results provide theoretical support for choosing the regularization
strength adaptively instead of committing to a fixed geometric bias.
Building on this theory, we further propose \textbf{$\eps$-Adaptive Selection}
($\eps$-AS), a novel CSAL algorithm that adapts its
entropic regularization to the unlabeled pool and annotation budget.
The resulting method turns the unified transport view into a
practical selector without requiring the user to choose among
typicality, coverage, and continuous diversity objectives for each
new dataset. Extensive experiments on six public datasets and multiple annotation
budgets show that $\eps$-AS consistently achieves state-of-the-art
performance and obtains the highest point estimate in 25 of 26
settings overall. It improves over the most accurate baseline,
ActiveFT, by 1.29\% while reducing selection time by 56.2\%. The
contributions of this paper can be summarized as follows:
\begin{enumerate}
\item We introduce a generalized transport framework that exposes the
shared allocation structure of three representative CSAL routes and
subsumes their formulations.
\item We provide theoretical foundations for adaptive transport
selection by establishing a monotone entropic trade-off and a
task-agnostic minimax bound for CSAL.
\item We propose a novel Sinkhorn-based CSAL algorithm $\eps$-AS built on the unified
framework and theoretical results.
\item Extensive experiments across six datasets and multiple budgets
show that $\eps$-AS achieves state-of-the-art accuracy with practical
selection efficiency.
\end{enumerate}

\section{II. Related Work}
\label{sec:related}

\subsubsection{Warm-Start Active Learning}
Classical AL starts from a labeled seed and alternates task-model
training with label acquisition~\citep{settles2009survey,ren2021survey}.
Deep methods estimate predictive uncertainty
\citep{gal2017deepbayesian,kirsch2019batchbald}, combine uncertainty
with diversity in gradient or Fisher representations
\citep{ash2020badge,ash2021bait}, or optimize scalable batch
objectives~\citep{citovsky2021batch,coleman2020selection}. Their
acquisition scores depend on task feedback and can be sensitive to the
initial seed and training protocol
\citep{mittal2019parting,munjal2022robust}.

\subsubsection{Cold-Start Active Learning}
CSAL selects the initial annotation batch without target labels or a
trained task model. Existing work spans three geometric routes and a
fourth group of task-specific hybrids.
\textbf{Typicality-based} favors dense or prototypical
regions. CALR combines target-domain contrastive features,
hierarchical clustering, and information density
\citep{jin2022calr}, while TypiClust, USL, and foundation-model
clustering select representative instances from learned partitions
\citep{hacohen2022typiclust,wang2022usl,tan2024foundclust}.
\textbf{Coverage-based} spreads queries across feature space
through CoreSet, ProbCover, DCoM, generalized coverage, or
$\gamma$-Tube geometry
\citep{sener2018coreset,yehuda2022probcover,mishal2024dcom,
bae2024maxherding,cao2022gammatube}.
\textbf{Diversity-regularized continuous} methods optimizes continuous
representatives or distributional subset objectives. ActiveFT uses
attraction and repulsion between learned prototypes
\citep{xie2023activeft}, while Wasserstein selection matches a
discrete subset to the pool distribution
\citep{mahmood2022lowbudget}.
\textbf{Task-specific hybrids} combine geometry with auxiliary
label-free signals. DEUCE uses dual diversity and uncertainty for
text classification~\citep{jin2024deuce}, and related acquisition
rules extend to vision-language models~\citep{safaei2025alvlm}.
CSAL-3D and CSCS combine representativeness with self-supervised
uncertainty or difficulty for medical segmentation
\citep{zhao2025csal3d,hattat2026datasetaware}. MedCAL-Bench further
studies foundation models and selection strategies across medical
tasks~\citep{zhu2025medcalbench}. These routes differ in how they
balance local representativeness, global coverage, and auxiliary
difficulty signals.

\subsubsection{Optimal Transport for Selection}
Optimal transport compares measures through a minimum-cost coupling
\citep{villani2009optimal,peyre2019computational,su2025higher}. Entropic
regularization enables efficient Sinkhorn scaling
\citep{cuturi2013sinkhorn,altschuler2017nearlinear}. For data
selection, Wasserstein matching has been posed as discrete subset
optimization~\citep{mahmood2022lowbudget}, and AQOT uses
entropy-regularized transport to combine informativeness and
representativeness in iterative AL~\citep{zhu2023aqot}. Partial OT
can expose a support subset by relaxing a source marginal
\citep{riaz2023partialot}, while UniPROT selects uniformly weighted
prototypes through a partial-OT submodular reformulation
\citep{chanda2026uniprot}. These works use transport for active
querying, support reduction, or prototype construction. We use it to
expose allocation structure across representative CSAL formulations
and connect balanced entropic regularization to unlabeled geometric
resolution. Classical quantization and intrinsic-dimension estimation
characterize how this resolution changes with budget
\citep{lloyd1982leastsquares,levina2004mle}.

\section{III. Preliminaries}
\label{sec:prelim}

\textbf{Roadmap.} In this section, we first define the CSAL
problem. Then, we introduce the feature geometry and transport constraints used throughout the paper.

\subsubsection{Problem setup.}
Let $U=\{x_i\}_{i=1}^n$ be an unlabeled pool, $b\ll n$ an annotation
budget, and $\phi$ a frozen feature extractor. A selector
$\mathcal A$ observes $U$ and $\phi$ but no task labels, and returns
$S=\mathcal A(U)\subset[n]$ with $|S|=b$ in one round. The selected
samples are then annotated and used to train a predictor $h_S$. For a
fixed learner and loss, the ideal selector minimizes
\begin{equation}\label{eq:csal_objective}
 \mathcal A^\star
 \in
 \argmin_{\mathcal A}
 \E\!\left[\mathcal R(h_{\mathcal A(U)})\right].
\end{equation}
Equation~\eqref{eq:csal_objective} evaluates a selector through its
downstream predictor. However, the selector has no access to labels.
We therefore begin with the label-free allocation problem.

\subsubsection{Geometric resolution.}
We work with normalized features
$z_i:=\phi(x_i)/\|\phi(x_i)\|_2\in\mathbb S^{D-1}$. A sample and a
representative are compared using the cosine cost
\[
\begin{aligned}
 C(z,\mu)&:=1-\langle z,\mu\rangle,\\
 d(z,\mu)&:=\|z-\mu\|_2=\sqrt{2C(z,\mu)}.
\end{aligned}
\]
For a feature distribution $P$, define
\[
 m_P(b):=
 \inf_{\mu_1,\ldots,\mu_b\in\mathbb S^{D-1}}
 \E_{z\sim P}\min_{k\in[b]}C(z,\mu_k).
\]
The quantity $m_P(b)$ is the smallest average representation cost
achievable with $b$ anchors. It measures the geometric resolution at
budget $b$. A larger budget places more anchors in the feature space.
It therefore tends to reduce this cost. Our statistical results use
bounded normalized geometry and the quantization law
$m_P(b)=\Theta(b^{-2/d_{\mathrm{int}}})$,
finite-pool objective accuracy, and geometric fidelity of the decoded
selection.

\subsubsection{Transport constraints.}
A transport plan $\pi\in\R_+^{n\times b}$ records how the mass of each
pool sample is assigned to $b$ representatives. We use uniform
pool weights $a=(1/n)\mathbf1_n$ and representative weights
$u=(1/b)\mathbf1_b$. The relevant feasible sets are
\[
\begin{aligned}
 \Pi_a^{\mathrm{sr}}
 &:=
 \{\pi\in\R_+^{n\times b}:\pi\mathbf1_b=a\},\\
 \Pi(a,u)
 &:=
 \{\pi\in\Pi_a^{\mathrm{sr}}:\pi^\top\mathbf1_n=u\}.
\end{aligned}
\]
The semi-relaxed set fixes the mass leaving each pool sample. The
balanced set requires every representative to receive mass
$1/b$. This distinction lets the unified framework retain the
allocation rule of each method. For a fixed cost matrix $C$, the
balanced entropic optimal transport problem is
\begin{equation}\label{eq:eot}
 \pi_\eps(C)
 :=
 \argmin_{\pi\in\Pi(a,u)}
 \left\{
 \langle\pi,C\rangle
 +\eps\KL(\pi\|a\otimes u)
 \right\}.
\end{equation}
For $\eps>0$ the minimizer is unique and can be computed by Sinkhorn
scaling. The entropy convention $\Hreg(\pi)$ changes
\eqref{eq:eot} only by a constant under fixed marginals.

\section{IV. A Unified Transport Selection View}
\label{sec:theory}

\textbf{Roadmap.} In this section, we first propose a generalized transport framework
that subsumes previous CSAL methods. We then establish the
task-agnostic minimax resolution associated with the unlabeled
geometry. \footnote{Full proofs and additional mathematical details of all Theorems are provided in Appendix~\ref{app:proofs}.}

\paragraph{Generalized transport framework.}
Every selection method must answer three questions. It must define the
allowed representatives, assign the pool to those representatives,
and convert the resulting allocation into distinct queries. We record
these choices in a realization
\[
 \mathfrak R
 =
 (\Theta_{\mathfrak R},
   \Pi_{\mathfrak R},
   C_{\mathfrak R},
   \Omega_{\mathfrak R},
   \mathcal D_{\mathfrak R}).
\]
Its components specify feasible representatives and allocation
constraints. They also specify a cost, a support-only energy, and a
decoder. A realization may use the method's stated pool-dependent
preprocessing. It may also use fixed deterministic conventions.
Its objective is
\begin{equation}\label{eq:general_program}
\begin{aligned}
 \mathcal J_\eps^{\mathfrak R}(\theta,\pi)
 &:=
 \langle\pi,C_\theta\rangle
 +\eps\KL(\pi\|a\otimes u)
 +\Omega_{\mathfrak R}(\theta),\\
 &\hspace{5.5em}\pi\in\Pi_{\mathfrak R}(\theta).
\end{aligned}
\end{equation}
At $\eps=0$ the KL term is absent. We say that
\eqref{eq:general_program} \emph{variationally subsumes} a method
when minimizing out $\pi$ produces that method's objective up to a
positive affine transformation. The decoder must also recover the
method's selection rule. For a greedy method, the condition is applied
to each step. Table~\ref{tab:realizations} summarizes the concrete
parameterizations of three representative methods and $\eps$-AS. The
first three rows yield the following exact subsumption result, while
the final row records the $\eps$-AS realization developed later. Further parameterization details and proof are provided in Appendix~\ref{app:proof_unification}.
\begin{table*}[t]
\centering
\footnotesize
\setlength{\tabcolsep}{1.5pt}
\setlength{\arrayrulewidth}{0.4pt}
\renewcommand{\arraystretch}{1.05}
\arrayrulecolor{black!45}
\caption{Exact parameterizations of representative CSAL methods
within the unified transport-selection framework.}
\label{tab:realizations}
\begin{tabular}{@{}
>{\raggedright\arraybackslash}m{0.075\textwidth}
>{\centering\arraybackslash}m{0.195\textwidth}
>{\centering\arraybackslash}m{0.360\textwidth}
>{\centering\arraybackslash}m{0.310\textwidth}
@{}}
\hline
Method
& $(\Theta_{\mathfrak R},\Pi_{\mathfrak R})$
& $(C_{\mathfrak R},\Omega_{\mathfrak R},\eps)$
& Decoder $\mathcal D_{\mathfrak R}$ \\
\hline

TypiClust
&
\(\displaystyle
\begin{gathered}
\Theta_{\mathrm T}=\prod_{k=1}^{b}V_k,\;s_k\in V_k\\[-1pt]
\Pi_{\mathrm T}=\Pi(a,u)
\end{gathered}\)
&
\(\displaystyle
\begin{gathered}
C^{\mathrm T}_{ik}(s)=\ell_{s_k}\\[-1pt]
\Omega_{\mathrm T}=0,\;\eps=0
\end{gathered}\)
&
\(\displaystyle
\mathcal D_{\mathrm T}(s)=\{s_1,\ldots,s_b\}\)
\\
\hline

ProbCover
&
\(\displaystyle
\begin{gathered}
\Theta_{\mathrm P}=\{S\subseteq[n]:|S|=b\}\\[-1pt]
q_k=z_{s_k},\;\Pi_{\mathrm P}=\Pi_a^{\mathrm{sr}}
\end{gathered}\)
&
\(\displaystyle
\begin{gathered}
C^\delta_{ik}(S)=\mathbf 1\{d(z_i,q_k)\ge\delta\}\\[-1pt]
\Omega_{\mathrm P}=0,\;\eps=0
\end{gathered}\)
&
\(\displaystyle
\begin{gathered}
\Delta_jF_\delta(S)
=F_\delta(S\cup\{j\})-F_\delta(S)\\[-1pt]
S_0=\varnothing,\;
s_t=\argmax_{j\notin S_{t-1}}\Delta_jF_\delta(S_{t-1})\\[-1pt]
S_t=S_{t-1}\cup\{s_t\},\;
\mathcal D_{\mathrm P}(S_b)=S_b
\end{gathered}\)
\\
\hline

ActiveFT
&
\(\displaystyle
\begin{gathered}
\Theta_{\mathrm A}=(\mathbb S^{D-1})^b\\[-1pt]
\Pi_{\mathrm A}=\Pi_a^{\mathrm{sr}}
\end{gathered}\)
&
\(\displaystyle
\begin{gathered}
C^{\mathrm A}_{ik}(\Theta)
=1-\langle z_i,\theta_k\rangle\\[-1pt]
L_k=\sum_{l\ne k}
\exp\!\left(\langle\theta_k,\theta_l\rangle/\tau\right),\\
\Omega_{\mathrm A}(\Theta)
=(\lambda\tau/b)\sum_{k=1}^{b}\log L_k,\;\eps=0
\end{gathered}\)
&
\(\displaystyle
\begin{gathered}
S_{k-1}=\{s_1,\ldots,s_{k-1}\}\\[-1pt]
s_k=
\argmax_{i\notin S_{k-1}}
\langle z_i,\theta_k\rangle,\\
\mathcal D_{\mathrm A}(\Theta)
=\{s_1,\ldots,s_b\}
\end{gathered}\)
\\
\hline

\rowcolor{gray!15}
$\eps$-AS
&
\(\displaystyle
\begin{gathered}
\Theta_{\eps\text{-AS}}
=\{(\widehat\mu_1,\ldots,\widehat\mu_b)\}\\[-1pt]
\Pi_{\eps\text{-AS}}=\Pi(a,u)
\end{gathered}\)
&
\(\displaystyle
\begin{gathered}
C_{ik}=1-\langle z_i,\widehat\mu_k\rangle,\;
\Omega_{\eps\text{-AS}}=0\\[-1pt]
\widehat m(b)=(1/n)\sum_i\min_k C_{ik},\;
\eps^\star=c\widehat m(b)
\end{gathered}\)
&
\(\displaystyle
\begin{gathered}
\pi^\star=\pi_{\eps^\star}(C)\\[-1pt]
\widehat S=\operatorname{RoundRobin}_{k\in[b]}
\left[\argmax_{i\notin S}\pi^\star_{ik}\right]
\end{gathered}\)
\\
\hline
\end{tabular}
\end{table*}

\begin{theorem}[Exact variational subsumption]
\label{thm:unification}
For any finite pool and $2\le b\le n$, realizations in
Table~\ref{tab:realizations} recover the objectives and selection
rules of TypiClust, ProbCover, and
ActiveFT.
\end{theorem}

\paragraph{Geometric resolution and minimax risk.}
To state a model-independent downstream guarantee, fix $L>0$ and let
$\operatorname{Lip}_L(P)$ be the real-valued functions on
$\operatorname{supp}(P)$ that are $L$-Lipschitz under $d$. Consider
these noiseless tasks under integrated squared error. A
deterministic selector sees the feature distribution but not $f$,
chooses $S\subseteq\operatorname{supp}(P)$ with
$1\le|S|\le b$, and observes $f|_S$. A deterministic learner then
maps $(S,f|_S)$ to a measurable predictor
$\widehat f_{\mathcal A,\mathcal L,f}$. If $\mathfrak A_b$
denotes all such selector and learner pairs, define
\[
 \mathcal R_b^\star(P)
 :=
 \inf_{(\mathcal A,\mathcal L)\in\mathfrak A_b}
 \sup_{f\in\operatorname{Lip}_L(P)}
 \E_P[(\widehat f_{\mathcal A,\mathcal L,f}(z)-f(z))^2].
\]
Thus the task supremum is taken after the task-agnostic procedure is
fixed.

\begin{theorem}[Task-agnostic cold-start minimax resolution]
\label{thm:lowerbound}
For any $P$ with bounded feature geometry,
\begin{equation}\label{eq:lowerbound}
 \frac{L^2}{2}m_P(b)
 \le
 \mathcal R_b^\star(P)
 \le
 2\kappa_qL^2m_P(b),
\end{equation}
where the upper bound requires a selector with
$Q_P(S):=\E_P\min_{s\in S}C(z,s)\le\kappa_qm_P(b)$. If
$m_P(b)=\Theta(b^{-2/d_{\mathrm{int}}})$, then
\begin{equation}\label{eq:eps_saturation}
 \mathcal R_b^\star(P)
 =\Theta\!\left(b^{-2/d_{\mathrm{int}}}\right),
 \qquad
 b_\eta=\Theta\!\left(\eta^{-d_{\mathrm{int}}/2}\right).
\end{equation}
Here $b_\eta$ is the smallest budget with minimax risk at most
$\eta$.
\end{theorem}

The two bounds have the same geometric origin. For the lower
bound, the indistinguishable tasks
$f_0(z)=0$ and $f_1(z)=Ld(z,S)$ make the distance to the queried
set unavoidable. For the upper bound, nearest-neighbor extension
converts that same distance into an attainable prediction-error
bound. Consequently, both bounds are governed by the optimal
closest-anchor cost $m_P(b)$, which defines the population
resolution available at budget $b$. The next section
estimates this quantity from the unlabeled pool by $\widehat m(b)$
and uses it to set the entropic scale
$\eps^\star=c\widehat m(b)$.

\section{V. $\eps$-Adaptive Selection}
\label{sec:method}

In this section, we first specify the balanced entropic selection
rule. We then derive its data-adaptive regularization strength,
calibrate the dimensionless constant, and present the complete
$\eps$-AS algorithm.

\paragraph{Balanced Entropic Selection.}
$\eps$-AS uses $K=b$ $k$-means anchors and a balanced
transport. The resulting plan is computed by
\eqref{eq:eot} and converted into a subset by
cyclically allowing every anchor to claim its largest remaining plan
entry:
\begin{equation}\label{eq:roundrobin}
\begin{aligned}
 i_k(\eps,S)
 &:=
 \argmax_{i\notin S}\pi_\eps(C)_{ik},\\
 \widehat S(\eps)
 &:=
 \operatorname{RoundRobin}_{k\in[b]}
 \bigl[i_k(\eps,S)\bigr],
  |\widehat S(\eps)|=b.
\end{aligned}
\end{equation}
A fixed global index order resolves exact ties. Within this branch,
$\eps$ has a precise meaning. If $0<\eps_1<\eps_2$, then
\[
\begin{aligned}
 \KL(\pi_{\eps_2}\|a\otimes u)
 &\le\KL(\pi_{\eps_1}\|a\otimes u), \langle\pi_{\eps_1},C\rangle
 \le\langle\pi_{\eps_2},C\rangle.
\end{aligned}
\]
Moreover, cluster points of $\pi_\eps$ as $\eps\downarrow0$ solve unregularized
balanced OT, while $\pi_\eps\to a\otimes u$ as $\eps\to\infty$.
Thus $\eps$ controls assignment sharpness within the balanced
entropic realization used by $\eps$-AS.

\paragraph{Data-adaptive Regularization.}
Theorem~\ref{thm:lowerbound} identifies $m_P(b)$ as the population
resolution governing task-agnostic minimax risk. Because $P$ is
unknown, we estimate this scale from the unlabeled pool at the
current budget and use the estimate to normalize the transport
regularization. A fixed $\eps$ can otherwise represent different
smoothing regimes as the pool geometry or budget changes. After fitting
$\widehat\mu_1,\ldots,\widehat\mu_b$, define
$d_i:=\min_kC(z_i,\widehat\mu_k)$ and
$\widehat m(b):=n^{-1}\sum_i d_i$. Consider a temperature rule based
only on $(d_1,\ldots,d_n)$. We require it to be invariant to pool
order and to scale linearly with the costs. It must also average
consistently when two closest-cost lists are concatenated under a
common anchor and cost system. This last axiom applies only to cost
lists that share that system. Each raw pool is clustered before the
rule is evaluated. These requirements uniquely give a constant times
the arithmetic mean:
\begin{equation}\label{eq:eps_star}
\eps^\star
=
c\,\widehat m(b)
=
\frac{c}{n}\sum_{i=1}^n
\min_{k\in[b]} C(z_i,\widehat\mu_k).
\end{equation}

Whenever $\widehat m(b)>0$, this positive-temperature rule is exactly
scale-equivariant. For every $s>0$ and $\eps>0$,
\[
 \pi_{s\eps}(sC)=\pi_\eps(C),
 \qquad
 \widehat m_{sC}(b)=s\widehat m_C(b).
\]
Thus changing cost units changes $\eps^\star$ by the same factor and
leaves the plan and decoded subset unchanged.

\paragraph{Calibration of the dimensionless constant.}
Scale equivariance fixes the units of $\eps$ but not $c$. We calibrate
$c$ with an exactly solvable local robust game. For a nondegenerate
budget, set $m:=m_P(b)>0$. A $2\times2$
oversmoothing block with route gap $m$ has exact entropic bias
\[
 B_m(\eps)={m}/{(1+e^{m/\eps})},
\]
which strictly increases from $0$ to $m/2$. A near-tie block whose
two costs differ by $\kappa m$ has exact plan displacement and
penalty
\[
 S_m(\eps)
 =
 \lambda m\tanh\!\left({\kappa m}/{2\eps}\right),
\]
which strictly decreases from $\lambda m$ to $0$. For
general nonnegative transport costs these blocks allow any
$\kappa>0$. A unit-sphere cosine realization additionally requires
$\kappa m\le2$. For
$J_m(\eps):=\max\{A B_m(\eps),S_m(\eps)\}$ and $t=\eps/m$,
$J_m(tm)=m\psi(t)$. For every fixed
$A,\lambda,\kappa>0$, the increasing and decreasing terms cross once,
at the unique $c_{\mathrm{can}}>0$ satisfying
\[
 {A}/({1+e^{1/c_{\mathrm{can}}}})
 =
 \lambda\tanh\!\left({\kappa}/{2c_{\mathrm{can}}}\right).
\]
Hence $c_{\mathrm{can}}m$ is the unique global minimizer of this
canonical robust game. The result proves that the optimum is a
dimensionless multiple of the cost scale. Its numerical value depends
on the game weights. We therefore use one global calibration to
determine the remaining empirical multiplier.

\begin{theorem}[Unique scale-equivariant calibrated temperature]
\label{thm:eps_star}
Up to $c$, Equation~\eqref{eq:eps_star} is the unique nondegenerate
rule satisfying permutation, positive-scale, and mixture consistency.
The canonical robust game has the unique minimizer
$c_{\mathrm{can}}m_P(b)$. Under finite-pool accuracy,
\[
 \frac{\eps^\star}{m_P(b)}\longrightarrow c
 \quad\text{in probability}.
\]
Under intrinsic-dimension scaling,
$\eps^\star=\Theta_p(b^{-2/d_{\mathrm{int}}})$. With decoder
fidelity, $\eps$-AS attains the minimax rate in
Theorem~\ref{thm:lowerbound}.
\end{theorem}

Together, these results connect population resolution to finite-pool
scale estimation. They also connect that estimate to the decoded
subset. This establishes a direct path from unlabeled geometry to the
final selector.

\paragraph{Algorithm and Complexity.}
The complete procedure follows the same logic as the derivation.
First, it fits one anchor for each available annotation. Second, it
measures how well these anchors represent the pool and computes
$\eps^\star$ from Equation~\eqref{eq:eps_star}. Third, it uses
Sinkhorn scaling to obtain a balanced allocation. Finally, the
round-robin decoder selects distinct pool samples while giving every
anchor an opportunity to contribute. Algorithm~\ref{alg:eps-as}
summarizes these steps.  With
$T_{\mathrm{km}}$ Lloyd iterations, the leading time is
$\mathcal O(nbDT_{\mathrm{km}})$. The numerical floor must scale with the
cost and satisfy $\eps_{\min}=o(m_P(b))$ asymptotically.

\begin{algorithm}[!t]
\caption{$\eps$-AS}
\label{alg:eps-as}
\textbf{Input:} nonzero features $\{x_i\}_{i=1}^n$, budget $b$, calibrated
$c$, tolerance $\eta$, cap $T_{\mathrm{sk}}$, numerical floor
$\eps_{\min}$.\\
\textbf{Output:} $S\subseteq[n]$, $|S|=b$.
\begin{algorithmic}[1]
\STATE Set $z_i\leftarrow x_i/\|x_i\|_2$.
\STATE Run $K=b$ $k$-means with deterministic empty-cell repair.
Replace a zero-norm centroid by its cell medoid, then normalize
$\widehat\mu_k\leftarrow\widehat\mu_k/\|\widehat\mu_k\|_2$.
\STATE Set $C_{ik}\leftarrow1-\langle z_i,\widehat\mu_k\rangle$ and
$\widehat m\leftarrow n^{-1}\sum_i\min_k C_{ik}$.
\STATE Set $\eps^\star\leftarrow\max\{c\widehat m,\eps_{\min}\}$.
\STATE Set $a_i\leftarrow1/n$, $u_k\leftarrow1/b$, and initialize
dual potentials $f_i,g_k\leftarrow0$.
\FOR{$t=1,\ldots,T_{\mathrm{sk}}$}
\STATE Update
$f_i\leftarrow\eps^\star\log a_i
-\eps^\star\log\sum_k
\exp((g_k-C_{ik})/\eps^\star)$.
\STATE Update
$g_k\leftarrow\eps^\star\log u_k
-\eps^\star\log\sum_i
\exp((f_i-C_{ik})/\eps^\star)$.
\STATE Stop if the maximum absolute change of either dual potential
is below $\eta$.
\ENDFOR
\STATE Recover
$\pi_{ik}\leftarrow
\exp((f_i+g_k-C_{ik})/\eps^\star)$.
\STATE In cyclic anchor order, repeatedly add the unselected index
with largest $\pi_{ik}$, using one fixed index order for ties, until
$|S|=b$.
\STATE \textbf{return} $S$.
\end{algorithmic}
\end{algorithm}

\section{VI. Experiments}
\label{sec:experiments}
\textbf{Roadmap.} In this section, we first introduce the experimental setup. Then, we report the main results, component
ablations, sensitivity analysis, and selection efficiency results.

\subsubsection{Experimental Setup}
We evaluate on MNIST \citep{lecun1998mnist}, CIFAR-10, CIFAR-100
\citep{krizhevsky2009cifar}, STL-10 \citep{coates2011stl10},
Caltech-101 \citep{feifei2004caltech101}, and ImageNet-1k
\citep{deng2009imagenet}. Table~\ref{tab:datasets} summarizes these
datasets. We use the official data splits for all datasets except
Caltech-101. For Caltech-101, we create stratified $7{:}3$ splits.
Each image is encoded once with a frozen DINOv2-G/14 model
\citep{oquab2024dinov2}, which yields the feature dimension
$D=1536$. We train a linear classifier with one output per class
using only the selected labels. Each experimental cell runs on a
single NVIDIA H200 GPU. We compare $\eps$-AS with nine baselines:
Random selection, Coreset
\citep{sener2018coreset}, FPS \citep{jin2022oneshot}, cold-start
BADGE \citep{ash2020badge}, ProbCover \citep{yehuda2022probcover},
DCoM \citep{mishal2024dcom}, TypiClust
\citep{hacohen2022typiclust}, USL \citep{wang2022usl}, and ActiveFT
\citep{xie2023activeft}. The dense ProbCover implementation is not
feasible on ImageNet-1k and is omitted on that dataset.
The rule in Equation~\eqref{eq:eps_star} leaves one dimensionless
constant $c$ to be calibrated. We use MNIST as the sole calibration
domain and evaluate $c\in\{0.1,0.2,\ldots,0.9\}$. For each candidate, the calibration score is the mean top-1 accuracy over the four MNIST budgets and three seeds. This single search selects $c=0.4$. We then fix that value for
all budgets on CIFAR-10, CIFAR-100, STL-10, Caltech-101, and
ImageNet-1k. No labels from these five target datasets are used to
select or update $c$, and no dataset-specific retuning is performed.

\begin{table}[htbp]
\small
\centering
\setlength{\tabcolsep}{4pt}
\caption{Datasets and annotation budgets.}
\label{tab:datasets}
\resizebox{\columnwidth}{!}{
\begin{tabular}{@{}lrrrl@{}}
\toprule
Dataset & $n_{\mathrm{pool}}$ & $n_{\mathrm{test}}$ & Classes & Budget grid \\
\midrule
MNIST & 60{,}000 & 10{,}000 & 10 & 50, 100, 200, 500 \\
CIFAR-10 & 50{,}000 & 10{,}000 & 10 & 50, 100, 200, 500 \\
CIFAR-100 & 50{,}000 & 10{,}000 & 100 & 100, 250, 500, 1000 \\
STL-10 & 5{,}000 & 8{,}000 & 10 & 50, 100, 200, 500 \\
Caltech-101 & 6{,}067 & 2{,}612 & 101 & 100, 200, 500, 1000 \\
ImageNet-1k & 1{,}281{,}167 & 50{,}000 & 1{,}000 & 250, 500, 1000, 2000, 5000, 10000 \\
\bottomrule
\end{tabular}
}
\end{table}

\subsubsection{Main Results}
\label{subsec:exp_main}

\begin{figure*}[t]
\centering
\includegraphics[width=0.97\textwidth]{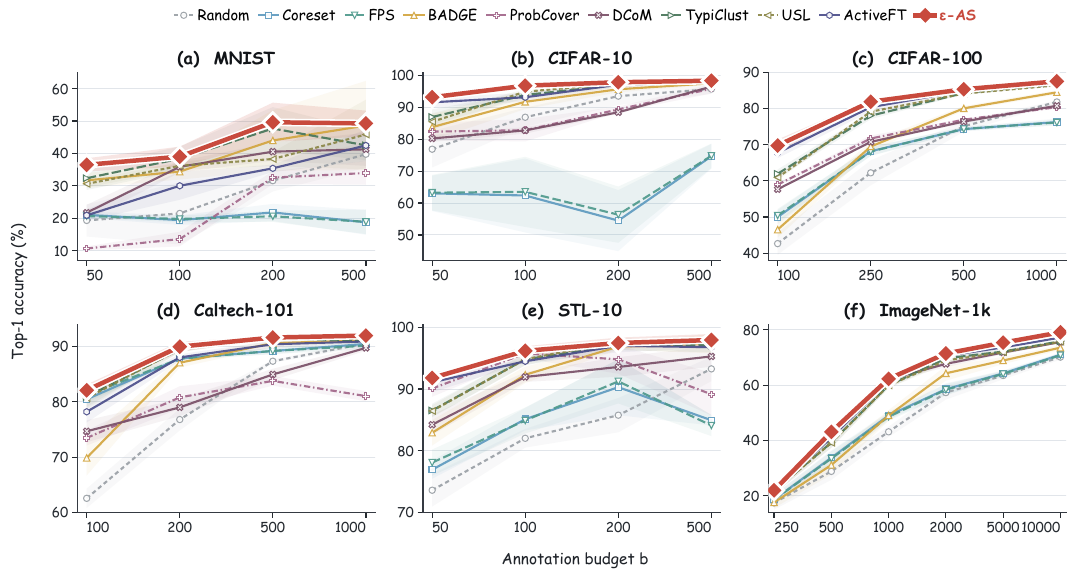}
\caption{\textbf{Top-1 accuracy on six benchmarks.} $\eps$-AS leads
in 25 of 26 dataset and budget settings.}
\label{fig:main}
\end{figure*}

Figure~\ref{fig:main} reports the complete comparison. The strongest
baseline changes with the dataset and annotation budget since different datasets and annotation budgets favor different geometric rules.
In contrast, $\eps$-AS has the highest point estimate in 25 of 26
settings. It also wins 21 of 22 settings after the MNIST calibration
domain is excluded. Its mean gain over the strongest baseline in
each setting is $0.98\%$. These consistent gains show that adapting
a shared allocation rule is more reliable than using a fixed
geometric bias for each new pool. Across the $22$ transfer cells,
$\eps$-AS also has the lowest average standard deviation at
$0.67\%$. The corresponding values are $0.81\%$ for USL and
$0.91\%$ for ActiveFT. On
ImageNet-1k, its standard deviation is the smallest at every budget.
The gains also transfer beyond the small calibration benchmark. On
CIFAR-10, CIFAR-100, and ImageNet-1k, the average improvements over
the strongest baseline are $1.66\%$, $1.18\%$, and $1.29\%$.
ImageNet-1k provides the most demanding test. It contains 1000
classes and more than 1.2 million pool images. The first two budgets
contain fewer queries than classes, which makes complete class
coverage impossible. $\eps$-AS ranks first at all six budgets. Its
advantage reaches $1.71\%$ and $1.89\%$ at budgets 5000 and 10000.
The improvement therefore persists as the budget grows. Overall, $\eps$-AS remains strong in the constrained cold-start
regime and retains its advantage as the scale and budget increase.

\subsubsection{Ablations and Sensitivity}
\label{subsec:exp_ablation}
Full $\eps$-AS combines a soft balanced Sinkhorn coupling with the
geometry-adaptive rule $\eps^\star=c\widehat m(b)$. We compare it
with two hard-selection controls that omit both components. The
$k$-means medoid variant queries the medoid of each cluster while
TypiClust queries a locally typical point from each cluster. The
remaining variants retain the balanced Sinkhorn coupling and
round-robin decoder but replace the adaptive rule with a fixed
$\eps\in\{10^{-3},1.0,0.05\}$ representing sharp, diffuse and
intermediate regimes. Table~\ref{tab:ablations} reports all variants.
On CIFAR-100, Medoid falls from $68.41\%$ at $b=100$ to $14.03\%$ at
$b=1000$, while Full $\eps$-AS rises from $69.77\%$ to $87.53\%$.
TypiClust avoids this collapse, but Full $\eps$-AS is more accurate at
every budget on both datasets. The full soft and adaptive realization therefore improves over both
hard-selection controls. The sharp setting is weak at small
budgets, while the other fixed settings are competitive only in
individual cases. No fixed value matches Full $\eps$-AS across both
datasets and all budgets. The adaptive rule thus provides a more
consistent allocation scale as geometry and budget change. We assess sensitivity to the global multiplier $c$ and Sinkhorn
iteration cap $T_{\mathrm{sk}}$. The former sets
$\eps^\star=c\widehat m(b)$, while the latter controls numerical
accuracy. We vary $c$ on STL-10 and Caltech-101 at every budget,
retaining the MNIST-selected $c=0.4$ without target-specific
retuning. We vary $T_{\mathrm{sk}}$ on Caltech-101 with $c=0.4$.
Figure~\ref{fig:robustness} shows that $c=0.4$ gives the highest
budget-averaged accuracy on both datasets, with gradual changes near
this value. Changing the default $T_{\mathrm{sk}}=200$ alters average
accuracy by at most $0.30\%$ and each budget result by at most
$0.58\%$. Its row-marginal residual is below $10^{-5}$. Performance
is therefore stable to calibration and solver perturbations.

\begin{table*}[htbp]
\small
\centering
\caption{Component ablations on CIFAR-100 and Caltech-101.}
\label{tab:ablations}
\footnotesize
\setlength{\tabcolsep}{2.5pt}
\renewcommand{\arraystretch}{1.0}
\begin{tabular}{@{}lcc*{8}{c}@{}}
\toprule
& & & \multicolumn{4}{c}{\textbf{CIFAR-100}} &
\multicolumn{4}{c}{\textbf{Caltech-101}} \\
\cmidrule(lr){4-7}\cmidrule(lr){8-11}
Variant & Soft & Adaptive $\eps$
& $b{=}100$ & $b{=}250$ & $b{=}500$ & $b{=}1000$
& $b{=}100$ & $b{=}200$ & $b{=}500$ & $b{=}1000$ \\
\midrule
\rowcolor{black!10}
Full & \cmark & \cmark
& \textbf{69.77$\pm$1.66} & \textbf{81.88$\pm$0.60} & 85.33$\pm$0.71 & \textbf{87.53$\pm$0.61}
& \textbf{82.08$\pm$1.04} & \textbf{89.97$\pm$0.89} & \textbf{91.62$\pm$0.24} & \textbf{91.97$\pm$0.66} \\
Medoid & \xmark & \xmark
& 68.41$\pm$1.30 & 63.27$\pm$0.91 & 37.77$\pm$2.35 & 14.03$\pm$2.65
& 81.19$\pm$1.78 & 82.03$\pm$1.67 & 80.59$\pm$2.11 & 74.65$\pm$1.94 \\
TypiClust & \xmark & \xmark
& 61.95$\pm$1.21 & 78.02$\pm$0.96 & 84.28$\pm$0.36 & 86.88$\pm$0.59
& 80.89$\pm$1.85 & 89.47$\pm$0.46 & 91.52$\pm$0.05 & 90.98$\pm$0.07 \\
$\eps{=}10^{-3}$ & \cmark & \xmark
& 49.90$\pm$1.00 & 68.88$\pm$1.55 & 81.19$\pm$1.52 & 86.06$\pm$0.63
& 78.02$\pm$1.48 & 88.81$\pm$0.30 & 90.36$\pm$0.81 & 91.48$\pm$0.28 \\
$\eps{=}1.0$ & \cmark & \xmark
& 68.48$\pm$1.39 & 81.59$\pm$0.34 & 85.27$\pm$0.46 & 87.29$\pm$0.62
& 81.99$\pm$2.03 & 89.48$\pm$0.63 & 91.61$\pm$0.14 & 91.14$\pm$0.15 \\
$\eps{=}0.05$ & \cmark & \xmark
& 68.21$\pm$0.67 & 81.65$\pm$0.33 & \textbf{85.39$\pm$0.68} & 87.45$\pm$0.89
& 81.65$\pm$2.47 & 89.37$\pm$0.89 & 91.57$\pm$0.40 & 91.39$\pm$0.59 \\
\bottomrule
\end{tabular}
\end{table*}

\begin{figure*}[t]
\centering
\includegraphics[width=\textwidth]{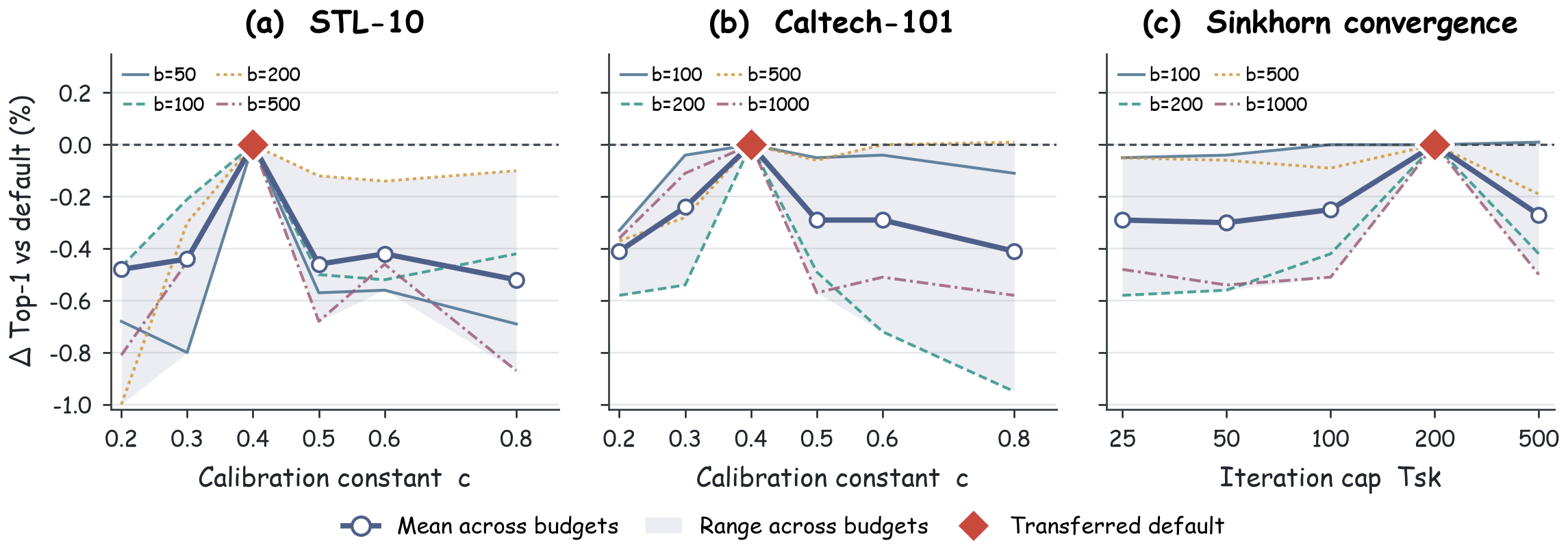}
\caption{\textbf{Sensitivity to the calibration constant and Sinkhorn
iteration cap.} The dark curve reports their mean, and the shaded band spans the
budget-level range. The figure shows changes in top-1 accuracy
relative to $c=0.4$ and $T_{\mathrm{sk}}=200$.}
\label{fig:robustness}
\end{figure*}

\subsubsection{Selection Efficiency}

Figure~\ref{fig:imagenet_pareto} pairs each method's mean ImageNet-1k
accuracy with its mean selection time. Both values are averaged over
six budgets. $\eps$-AS provides the strongest accuracy and efficiency
combination among the competitive methods. It improves the average
accuracy over ActiveFT by $1.29\%$. It also reduces selection time
from $55.41$ to $24.28$ seconds, which gives a $2.3\times$ speedup.
DCoM and USL require more time and achieve lower accuracy. TypiClust
is faster, but its average accuracy is $2.45\%$ lower. In summary, $\eps$-AS remains faster and more accurate
than the optimization-heavy baselines.

\begin{figure}[!b]
\centering
\includegraphics[width=\columnwidth]{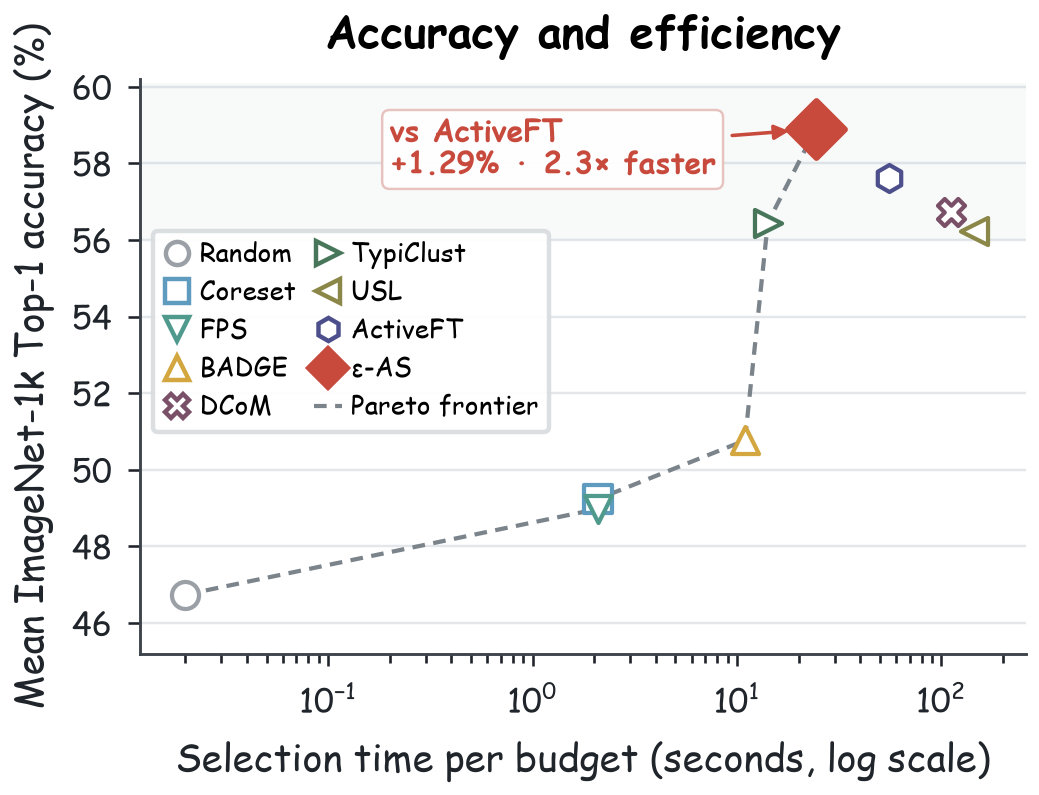}
\caption{\textbf{Accuracy--efficiency trade-off on ImageNet-1k} averaged over annotation budgets and seeds.}
\label{fig:imagenet_pareto}
\end{figure}

\subsubsection{Discussion}
\label{sec:discussion}
Theorem~\ref{thm:unification} changes the question posed by CSAL.
TypiClust, ProbCover, and ActiveFT are distinct realizations of one
allocation program. The central problem
is therefore not which named method to choose, but which part of the
realization should adapt when labels are unavailable. Their varying
behavior across datasets and budgets suggests that no fixed geometric
bias is uniformly appropriate. Our analysis identifies the allocation scale as an unlabeled and adaptable quantity. The closest-anchor cost $m_P(b)$ measures the
resolution supported by the pool at budget $b$, and
$\eps^\star=c\widehat m(b)$ matches entropic smoothing to that
resolution. This reframes CSAL from choosing among fixed heuristics to
adapting a common allocation geometry. More broadly, it suggests that
geometric scale should be adapted before additional task-specific
selection rules are introduced. Eliminating the remaining global
calibration and extending this principle beyond frozen representations
are important directions for future work.

\section{VII. Conclusion}
\label{sec:conclusion}

In this paper, we introduced a generalized transport framework that subsumes
representative CSAL methods. Within this framework, we proved two
complementary results. We showed that entropic regularization induces
a monotone trade-off between geometric fit and allocation
diffuseness. We also established a task-agnostic minimax bound that
identifies the closest-anchor cost as the resolution governing
cold-start selection. Guided by these results, we developed
$\eps$-AS, which adapts its regularization to the unlabeled pool and
budget. Experiments show that $\eps$-AS achieves the best overall accuracy while retaining practical
selection efficiency.

\onecolumn
\raggedbottom
\appendix
\setcounter{secnumdepth}{2}
\counterwithin{figure}{section}
\counterwithin{table}{section}
\counterwithin{equation}{section}
\let\assumption\assumptionApp   \let\endassumption\endassumptionApp
\let\theorem\theoremApp         \let\endtheorem\endtheoremApp
\let\lemma\lemmaApp             \let\endlemma\endlemmaApp
\let\proposition\propositionApp \let\endproposition\endpropositionApp
\let\corollary\corollaryApp     \let\endcorollary\endcorollaryApp
\let\definition\definitionApp   \let\enddefinition\enddefinitionApp
\let\remark\remarkApp           \let\endremark\endremarkApp
\csalwraptcb{assumption}{thmbox}
\csalwraptcb{theorem}{thmbox}
\csalwraptcb{lemma}{thmbox}
\csalwraptcb{proposition}{thmbox}
\csalwraptcb{corollary}{thmbox}
\csalwraptcb{definition}{thmbox}
\hypersetup{citecolor=suppteal,linkcolor=suppnavy,urlcolor=suppteal}
\titleformat{\section}
  {\Large\bfseries\color{suppnavy}}
  {\thesection}{0.75em}{}
  [\vspace{0.25em}{\color{supprule}\titlerule}]
\titlespacing*{\section}{0pt}{1.6em}{0.85em}
\titleformat{\subsection}
  {\normalsize\bfseries\color{suppnavy}}
  {\thesubsection}{0.7em}{}
\titlespacing*{\subsection}{0pt}{1.45em}{0.55em}
\titleformat{\subsubsection}
  {\normalsize\bfseries\color{suppgray}}
  {\thesubsubsection}{0.65em}{}
\titlespacing*{\subsubsection}{0pt}{1.2em}{0.45em}
\titleformat{\paragraph}[runin]
  {\bfseries\color{suppnavy}}
  {}{0pt}{}
\titlespacing*{\paragraph}{0pt}{0.9em}{0.5em}
\titlecontents{section}
  [0pt]
  {\addvspace{0.55em}\bfseries\color{suppnavy}}
  {\contentslabel{1.8em}}
  {}
  {\titlerule*[0.65pc]{.}\contentspage}
\titlecontents{subsection}
  [2.0em]
  {\small\color{black!82}}
  {\contentslabel{2.8em}}
  {}
  {\titlerule*[0.65pc]{.}\contentspage}
\captionsetup{
  font=small,
  labelfont={bf,color=suppnavy},
  textfont={color=black!88},
  justification=raggedright,
  singlelinecheck=false,
  skip=5pt
}
\captionsetup[table]{position=top}
\renewcommand{\topfraction}{0.95}
\renewcommand{\bottomfraction}{0.9}
\renewcommand{\textfraction}{0.05}
\renewcommand{\floatpagefraction}{0.7}
\setcounter{topnumber}{3}
\setcounter{bottomnumber}{2}
\setcounter{totalnumber}{5}
\renewcommand{\arraystretch}{1.08}
\setlist{nosep,leftmargin=1.7em}
\setlength{\emergencystretch}{2em}
\clubpenalty=10000
\widowpenalty=10000
\displaywidowpenalty=10000
\setlength{\headheight}{14pt}
\setlength{\headsep}{25pt}
\addtolength{\textheight}{-39pt}
\fancyhf{}
\fancyhead[L]{\footnotesize\color{suppgray}\textsc{A Unified Optimal Transport View of Cold-Start Active Learning}}
\fancyhead[R]{\footnotesize\color{suppgray}Supplementary Material}
\fancyfoot[C]{\footnotesize\color{suppgray}\thepage}
\renewcommand{\headrulewidth}{0.45pt}
\renewcommand{\footrulewidth}{0pt}
\makeatletter
\renewcommand{\headrule}{
  {\color{supprule}\hrule height \headrulewidth width\headwidth}
  \vskip-\headrulewidth}
\makeatother
\pagestyle{fancy}

\phantomsection
\addcontentsline{toc}{section}{Supplementary Material}

\begingroup
  \setlength{\parskip}{0.2em}
  \hypersetup{linkcolor=suppnavy}
  \renewcommand{\contentsname}{\large\bfseries\color{suppnavy}Contents}
  \tableofcontents
\endgroup

\vspace{0.45em}{\color{supprule}\hrule}\vspace{0.55em}
\clearpage

\section{Appendix Overview and Notation Table}
\label{app:overview}

This appendix collects the full proofs of the main-text theoretical
results, an extended discussion of related work, and the experimental
details omitted from the main text for space.

\paragraph{Notation table.}
\begin{center}\small
\setlength{\tabcolsep}{6pt}
\begin{tabular}{@{}ll@{}}
\toprule
\rowhead
Symbol & Meaning \\
\midrule
$\phi$, $z_i$ & frozen feature, unit-normalized feature \\
$\mu_U$ & empirical pool $\tfrac{1}{n}\sum_i\delta_{z_i}$ \\
$a$, $u_b$ & uniform pool masses $(1/n)_i$, uniform target masses $(1/b)_k$ \\
$\mathfrak R$ & a realization of the transport--selection program \\
$\Theta_{\mathfrak R}$ & feasible representatives in realization $\mathfrak R$ \\
$\Pi_{\mathfrak R}$ & feasible allocation plans in realization $\mathfrak R$ \\
$\mathcal D_{\mathfrak R}$ & decoder from a variational solution to a size-$b$ subset \\
$C_{ij}$ & cosine cost $1-\langle z_i,\mu_j\rangle\in[0,2]$ \\
$\Pi(\alpha,\beta)$ & couplings with marginals $\alpha,\beta$ \\
$\Hreg(\pi)$ & entropy $\sum_{ij}\pi_{ij}(\log\pi_{ij}-1)$ \\
$\OT_\eps$ & entropic optimal-transport cost \\
$d_{\min}(z)$ & $\min_k C(z,\mu_k)$ \\
$m_P(b)$ & optimal population $b$-representative cost \\
$Q_P(S)$ & population quantization cost of decoded set $S$ \\
$d_{\mathrm{int}}$ & intrinsic dimension via $m(b)\propto b^{-2/d_{\mathrm{int}}}$ \\
$\mathcal R_b^\star(P)$ & task-agnostic Lipschitz minimax risk \\
$\widehat m(b)$ & empirical closest-anchor cost \\
$\eps^\star$ & data-adaptive scale $c\cdot\widehat m(b)$ \\
$\widehat S(\eps)$ & selection from \meqref{eq:roundrobin} at $\eps$ \\
\bottomrule
\end{tabular}
\end{center}

\paragraph{Generic notation.}
Beyond the problem-specific symbols above, we use standard
mathematical notation throughout. We write $[n]=\{1,\dots,n\}$ for the
first $n$ integers and $\binom{[n]}{b}$ for the family of all size-$b$
subsets of $[n]$. For sets $A,B$, $A\setminus B$ is the set difference
and $A\triangle B=(A\setminus B)\cup(B\setminus A)$ is the symmetric
difference. Feature vectors live in the $D$-dimensional Euclidean space
$\R^D$, on which $\langle\cdot,\cdot\rangle$ is the inner product,
$\|\cdot\|$ the induced Euclidean norm, and
$\mathbb S^{D-1}=\{x\in\R^D:\|x\|=1\}$ the unit sphere. We write
$z\sim P$ to mean that $z$ is drawn from the distribution $P$, and
$\E$ for expectation. The symbol $\delta_z$ denotes the Dirac point
mass at $z$, that is, the probability measure placing all of its mass
at the single point $z$, so that $\tfrac1n\sum_i\delta_{z_i}$ is the
empirical distribution of the pool. For asymptotics we use the
Bachmann--Landau symbols: $f=O(g)$ means $f\le c\,g$ for some constant
$c$ and all large enough arguments, $f=o(g)$ means $f/g\to 0$,
$f=\Theta(g)$ means $c_1 g\le f\le c_2 g$ for constants
$0<c_1\le c_2$, and $f\propto g$ means $f=c\,g$ for a constant $c$.
Their in-probability analogs $O_p$ and $o_p$ are defined at their
first use. The limit $\eps\downarrow 0$ (equivalently $\eps\to 0^+$)
means $\eps$ decreases to $0$ through positive values, and
$\Gamma$-convergence denotes the standard variational notion of
functional convergence under which cluster points of minimizers of
the approximating functionals are minimizers of the limit.

\paragraph{Quantization notation.}
For unit vectors $z,\mu\in\mathbb S^{D-1}$, put
\[
 C(z,\mu):=1-\langle z,\mu\rangle,
 \qquad
 d(z,\mu):=\|z-\mu\|_2=\sqrt{2C(z,\mu)}.
\]
For an anchor tuple
$\boldsymbol\mu=(\mu_1,\ldots,\mu_b)\in(\mathbb S^{D-1})^b$, define
\[
 R_P(\boldsymbol\mu)
 :=
 \E_{z\sim P}\!\left[\min_{k\in[b]}C(z,\mu_k)\right],
 \qquad
 R_n(\boldsymbol\mu)
 :=
 \frac1n\sum_{i=1}^n\min_{k\in[b]}C(z_i,\mu_k),
\]
and
\begin{equation}\label{eq:population_resolution_app}
 m_P(b):=
 \inf_{\boldsymbol\mu\in(\mathbb S^{D-1})^b}
 R_P(\boldsymbol\mu).
\end{equation}
We write $m(b)$ when $P$ is clear.  Let
$\widehat{\boldsymbol\mu}
=(\widehat\mu_1,\ldots,\widehat\mu_b)$ be the normalized anchors
returned by the clustering implementation and set
$\widehat m(b):=R_n(\widehat{\boldsymbol\mu})$.

For a finite set $S\subset\mathbb S^{D-1}$, define
\[
 Q_P(S):=
 \E_{z\sim P}\!\left[\min_{s\in S}C(z,s)\right].
\]
For an anchor tuple $\boldsymbol\mu$, use the smallest anchor index to
break Voronoi ties and put
\[
 a_{\boldsymbol\mu}(z)
 :=
 \min\!\left\{
 k\in[b]:
 C(z,\mu_k)=\min_{\ell\in[b]}C(z,\mu_\ell)
 \right\},
 \qquad
 p_k(\boldsymbol\mu)
 :=
 P\!\left(a_{\boldsymbol\mu}(z)=k\right).
\]
For $t>0$, let $s_k(t)$ be the round-robin selection for anchor $k$ at
temperature $t\,\widehat m(b)$, with $s_k(t):=z_k$ if
$\widehat m(b)=0$.  Thus
\[
 \widehat S(t\,\widehat m(b))
 =
 \{s_1(t),\ldots,s_b(t)\}.
\]

\begin{definition}[Variational subsumption]
\label{def:variational_subsumption}
Let a selection method $M$ be specified by an objective, a feasible
set, and a deterministic rule that maps an optimizer to a subset.
A transport--selection program \emph{variationally subsumes} $M$
if it has a realization $\mathfrak R_M$ for which:
\begin{enumerate}[label=(\roman*),leftmargin=2.1em]
\item the feasible variables of $\mathfrak R_M$ are exactly the
      variables optimized by $M$, possibly together with allocation
      variables that can be minimized out in closed form;
\item after those allocation variables are minimized out, the
      program objective equals the objective of $M$ up to a
      positive multiplicative constant and an additive constant
      independent of the decision variables; and
\item the program decoder $\mathcal D_{\mathfrak R_M}$ is the
      selection rule of $M$.
\end{enumerate}
Positive rescaling and decision-independent translation preserve the
set of minimizers, so this definition implies equality of the
selected subsets whenever the common deterministic tie convention is
used.
\end{definition}

\clearpage
\subsection{Theoretical Results and Logical Dependencies}
\label{app:theory_map}

The theoretical development contains exact finite-pool identities,
nonasymptotic minimax bounds, and conditional asymptotic conclusions.
Table~\ref{tab:theory_dependencies} records the logical status of each
result and the conditions used at that stage.

\begin{table}[H]
\small
\centering
\caption{Logical status and dependencies of the theoretical results.}
\label{tab:theory_dependencies}
\setlength{\tabcolsep}{3.5pt}
\renewcommand{\arraystretch}{1.15}
\begin{tabular}{@{}
>{\raggedright\arraybackslash}p{0.155\linewidth}
>{\raggedright\arraybackslash}p{0.175\linewidth}
>{\raggedright\arraybackslash}p{0.275\linewidth}
>{\raggedright\arraybackslash}p{0.305\linewidth}
@{}}
\toprule
\rowhead
Result & Logical status & Conditions & Conclusion \\
\midrule
\mthm{thm:unification}
& Exact finite-pool identity
& A finite pool, $2\le b\le n$, and fixed repair and tie conventions
& The reduced objectives and deterministic decoders recover the
evaluated TypiClust, ProbCover, and ActiveFT formulations. \\
\clem{lem:anchorspacing}
& Exact finite-dimensional result
& A fixed finite cost matrix, balanced marginals, and $\eps>0$
& The EOT plan is unique, its KL decreases and its transport cost
increases with $\eps$, and both temperature endpoints are identified. \\
\addlinespace[2pt]
\mthm{thm:lowerbound}, lower bound
& Nonasymptotic minimax bound
& \cassum{assumption:feature_geometry}
& Every deterministic task-agnostic selector and learner incurs risk
at least $(L^2/2)m_P(b)$ on the Lipschitz task class. \\
\mthm{thm:lowerbound}, upper bound
& Nonasymptotic attainable bound
& A selected set with $Q_P(S)\le\kappa_qm_P(b)$ and
nearest-neighbor readout
& The worst-case risk is at most $2\kappa_qL^2m_P(b)$. \\
\addlinespace[2pt]
\mthm{thm:eps_star}, scale statistic
& Exact finite-pool characterization
& The permutation, positive-scale, mixture-consistency, and
normalization axioms in \cdefn{def:scale_rule_app}
& The admissible closest-anchor statistic is uniquely
$c\,\widehat m(b)$. \\
Canonical calibration game
& Exact normalized optimization
& Fixed dimensionless weights $A,\lambda,\kappa>0$ and $m_P(b)>0$
& The game has a unique interior minimizer
$c_{\mathrm{can}}(A,\lambda,\kappa)m_P(b)$. \\
\addlinespace[2pt]
\mthm{thm:eps_star}, consistency
& Conditional asymptotic conclusion
& Assumptions~\ref{assumption:feature_geometry}
and~\ref{assumption:sampling}--\ref{assumption:approximate_clustering}
& $\eps^\star/m_P(b)\to c$ in probability, with
$\eps^\star=\Theta_p(b^{-2/d_{\mathrm{int}}})$ under
\cassum{assumption:intrinsic_dim}. \\
$\eps$-AS decoded rate
& Conditional asymptotic conclusion
& The preceding statistical conditions,
\cassum{assumption:localized_decoding}, and
$c\in[t_-,t_+]$
& The decoded set is quantization-faithful with high probability and,
with nearest-neighbor readout, attains the minimax rate. \\
\bottomrule
\end{tabular}
\end{table}

The first two rows are algebraic properties of finite optimization
problems and require no distributional approximation. The two
minimax rows separate the unavoidable geometric resolution from the
property sufficient to attain it. The remaining rows separate the
choice of scale statistic, the dimensionless calibration, its
population consistency, and the final decoder-dependent rate.
Accordingly, the logical path is
\[
\text{structural realization}
\ \longrightarrow\
\text{geometric resolution}
\ \longrightarrow\
\text{empirical scale}
\ \longrightarrow\
\text{decoded selection}.
\]

\paragraph{Connection to experimental evaluation.}
The minimax analysis and the experiments assess complementary levels
of the selection problem.  Theorem~\mref{thm:lowerbound} uses
nearest-neighbor readout as a constructive learner to establish an
attainable geometric rate over the Lipschitz task class.  The
experiments instead apply the same linear-probe protocol to every
selected subset and measure its utility for multiclass
classification.  Neither the probe labels nor its predictions enter
the selector.  Thus $\eps$-AS adapts to the geometry of the unlabeled
pool and the annotation budget while remaining task agnostic.  In
parallel, Theorem~\mref{thm:eps_star} identifies the
scale-equivariant form $c\,\widehat m(b)$, while the single
dimensionless constant used in the experiments is fixed by the MNIST
calibration protocol.

\clearpage
\section{Assumptions and Justifications}
\label{app:assumptions}

The assumptions below concern the feature distribution, the accuracy
of the empirical clustering objective, and the distance between each
anchor and its decoded point.  Their geometric and statistical
justifications follow the six statements.

\begin{assumption}[Normalized feature geometry and nondegeneracy]
\label{assumption:feature_geometry}
For every distribution $P$ under consideration, the normalized
feature extractor satisfies $\|\phi(x)\|_2>0$ almost surely, and
$z=\phi(x)/\|\phi(x)\|_2$ is supported on
$\mathbb S^{D-1}$.  For every admissible budget $b$, the distribution
$P$ is not supported on a set of at most $b$ points.
\end{assumption}

\begin{assumption}[Intrinsic-dimension scaling]
\label{assumption:intrinsic_dim}
There are constants $0<c_1\le c_2<\infty$ and
$d_{\mathrm{int}}>0$, independent of $b$, such that
\begin{equation}\label{eq:m_scaling}
 c_1b^{-2/d_{\mathrm{int}}}
 \le m_P(b)
 \le c_2b^{-2/d_{\mathrm{int}}}
\end{equation}
throughout the budget range in which the theory is invoked.
\end{assumption}

\begin{assumption}[Sampling regime]
\label{assumption:sampling}
For each $n$, the unlabeled pool
$U_n=(z_1,\ldots,z_n)$ is drawn from $P^n$.  Any clustering
initialization is independent of $U_n$.  Probability limits are taken
as $n\to\infty$ along a budget sequence $b=b_n$ with
$2\le b_n\le n$.  Budget-rate statements additionally use
$b_n\to\infty$.
\end{assumption}

\begin{assumption}[Uniform empirical approximation]
\label{assumption:uniform_accuracy}
There is a deterministic sequence $r_{n,b}\ge0$ such that
\begin{equation}\label{eq:uniform_empirical_accuracy}
 \Pr\!\left[
   \sup_{\boldsymbol\mu\in(\mathbb S^{D-1})^b}
   |R_n(\boldsymbol\mu)-R_P(\boldsymbol\mu)|
   \le r_{n,b}
 \right]\longrightarrow1,
 \qquad
 \frac{r_{n,b}}{m_P(b)}\longrightarrow0.
\end{equation}
\end{assumption}

\begin{assumption}[Approximate clustering]
\label{assumption:approximate_clustering}
There is a deterministic sequence $\xi_{n,b}\ge0$ such that
\begin{equation}\label{eq:clustering_accuracy}
 \Pr\!\left[
 R_n(\widehat{\boldsymbol\mu})
 \le
 \inf_{\boldsymbol\mu\in(\mathbb S^{D-1})^b}
 R_n(\boldsymbol\mu)+\xi_{n,b}
 \right]\longrightarrow1,
 \qquad
 \frac{\xi_{n,b}}{m_P(b)}\longrightarrow0.
\end{equation}
\end{assumption}

\begin{assumption}[Local anchor-to-query displacement]
\label{assumption:localized_decoding}
There are constants $0<t_-<t_+<\infty$ and
$0\le\kappa_d<\infty$, independent of $b$, such that
\begin{equation}\label{eq:localized_decoding}
 \Pr\!\left[
 \sup_{t\in[t_-,t_+]}
 \frac{
 \sum_{k=1}^b
 p_k(\widehat{\boldsymbol\mu})
 C(\widehat\mu_k,s_k(t))
 }{m_P(b)}
 \le\kappa_d
 \right]\longrightarrow1.
\end{equation}
\end{assumption}

\paragraph{Justification of the assumptions.}
Feature normalization gives the spherical support in
\cassum{assumption:feature_geometry}.  The nondegeneracy condition requires
the feature distribution to contain more than $b$ support points.  The
objective in \eqref{eq:population_resolution_app} is continuous on the
compact set $(\mathbb S^{D-1})^b$, so its infimum is attained.  On the
unit sphere, this attained value is zero exactly when the distribution
is supported on at most $b$ points.  Thus the stated nondegeneracy
condition is equivalent to $m_P(b)>0$.

\cassum{assumption:intrinsic_dim} is the standard squared-distance
quantization behavior for a regular low-dimensional distribution. A
sufficient geometric setting is a compact
$d_{\mathrm{int}}$-dimensional support whose probability mass in
small metric balls is bounded above and below by constant multiples
of the radius to the power $d_{\mathrm{int}}$. The relation
$C(z,\mu)=d(z,\mu)^2/2$ then gives the exponent
$2/d_{\mathrm{int}}$. Appendix~\ref{app:intrinsic_dim} reports the
corresponding label-free finite-budget fits.

\cassum{assumption:sampling} is the standard random-design model for
an unlabeled pool.  Independence of the clustering initialization
separates algorithmic randomness from sampling randomness.

For fixed $b$ and $D$, boundedness of $C$ and compactness of
$(\mathbb S^{D-1})^b$ give uniform convergence of $R_n$ to $R_P$.
\cassum{assumption:uniform_accuracy} extends this standard property to
the stated growing-budget regime by requiring its error to be smaller
than the population resolution.

\cassum{assumption:approximate_clustering} separates clustering
optimization from statistical approximation.  An exact empirical
optimizer has $\xi_{n,b}=0$, while the stated condition also covers a
finite-run solver whose objective gap is negligible relative to
$m_P(b)$.

\cassum{assumption:localized_decoding} is a local geometric condition
on the discretization of the anchors.  It involves only the
probability mass of each anchor cell and the distance from that anchor
to its decoded point, without reference to labels, a downstream
predictor, or a task loss.  A simple sufficient condition is
\[
 \Pr\!\left[
 \sup_{t\in[t_-,t_+]}\max_{k\in[b]}
 \frac{C(\widehat\mu_k,s_k(t))}{m_P(b)}
 \le \kappa_d
 \right]\longrightarrow1,
\]
because the cell probabilities sum to one.  Such locality is natural
for a sufficiently dense pool when every anchor has a distinct nearby
candidate that remains available at its round-robin turn.  The
probability-weighted assumption is weaker than this pointwise
condition and permits larger displacements in cells with smaller
population mass.  It therefore isolates the geometric loss incurred
when continuous anchors are converted into distinct pool queries.

\section{Extended Discussion of Related Work}
\label{app:related_works}
The CSAL problem sits at the intersection of several lines of research. We extend detailed discussion of related work in this section. They are organized into warm-start active learning, cold-start
active learning, optimal transport for selection and the geometry of
intrinsic dimension.

\subsection{Warm-Start Active Learning}
Classical active learning assumes an initial labeled seed and a
trained task model, and proceeds in rounds that query the points
expected to be most informative for that
model~\citep{settles2009survey,ren2021survey}. The dominant paradigm
is uncertainty sampling, which queries the points on which the current
model is least confident, using least-confidence or margin
scores~\citep{lewis1994sequential,scheffer2001margin} or predictive
entropy. Bayesian and disagreement-based variants replace the point
estimate by a posterior: BALD and its deep and batch extensions score
points by the expected information gain about the model
parameters~\citep{gal2017deepbayesian,kirsch2019batchbald}, and
deep ensembles approximate the same disagreement
signal~\citep{beluch2018ensembles}. A second family targets
representativeness and diversity rather than uncertainty: the core-set
approach selects a batch that covers the feature
space~\citep{sener2018coreset}, BADGE mixes predictive uncertainty
with diversity in gradient space~\citep{ash2020badge}, and further
methods learn a loss or an adversarial embedding to drive
selection~\citep{yoo2019learningloss,sinha2019vaal,agarwal2020contextdiv}. These methods require either a labeled seed or a trained task model to
evaluate acquisition scores. Their measured gains can also be
sensitive to the initial seed and training protocol, particularly at
small budgets~\citep{mittal2019parting,munjal2022robust}. Cold-start
selection instead constructs the initial batch before task feedback is
available.

\subsection{Cold-Start Active Learning}
Cold-start active learning selects the initial annotation batch from
an unlabeled pool without target labels or a trained task model.
Following the taxonomy in the main text, existing methods can be
organized into three general geometric routes and a fourth group of
task-specific hybrids. \textbf{Typicality-based methods} favor dense or prototypical regions
of the feature space. CALR learns contrastive features on the target
pool, performs agglomerative hierarchical clustering, and selects
high-density representatives~\citep{jin2022calr}. TypiClust applies
$k$-means and selects a locally typical point from each cluster
\citep{hacohen2022typiclust}. USL combines $k$-means anchors,
$k$NN-density typicality, and a regularized selection rule
\citep{wang2022usl}. Self-supervised representation learning and
foundation-model clustering further strengthen the feature geometry
used to identify representative instances
\citep{bengar2021reducing,tan2024foundclust}. \textbf{Coverage-based methods} spread the annotation budget across
feature space so that few regions remain far from the selected set.
CoreSet and farthest-point traversal greedily reduce the distance from
pool points to their nearest selected point~\citep{sener2018coreset}.
ProbCover covers a feature graph with fixed-radius neighborhoods
\citep{yehuda2022probcover}, while DCoM adapts the coverage radius
\citep{mishal2024dcom}. Generalized coverage and herding provide
related objectives for low-budget selection
\citep{bae2024maxherding,bae2025uherding}. $\gamma$-Tube uses the
disagreement structure of tube and ball regions derived from minimum
enclosing balls to guide cold-start sampling
\citep{cao2022gammatube}. \textbf{Diversity-regularized continuous methods} use global objectives
to distribute representatives across the pool. ActiveFT learns
continuous prototypes with an attraction term that represents the
pool and a repulsion term that separates the prototypes
\citep{xie2023activeft}. Wasserstein subset selection optimizes a
discrete support whose empirical distribution matches the pool
\citep{mahmood2022lowbudget}. These objectives promote global
diversity without relying on a local typicality score or a prescribed
coverage radius. \textbf{Task-specific hybrids} combine feature geometry with
modality-dependent label-free signals. ALPS uses self-supervised
language-model representations for cold-start text selection
\citep{yuan2020alps}. DEUCE constructs a dual-neighbor graph for
textual and class diversity and propagates predictive uncertainty to
select difficult representatives~\citep{jin2024deuce}. Related
acquisition rules extend to vision-language models and foundation models
\citep{safaei2025alvlm, ma2025sugfw}. For medical segmentation, CSAL-3D combines
cluster representativeness with self-supervised uncertainty, while
SUGFW applies uncertainty-guided feature weighting
\citep{zhao2025csal3d,li2025sugfw}. CSCS combines local typicality and
reconstruction difficulty through a dataset-aware pacing rule that
depends on pool structure and annotation budget
\citep{hattat2026datasetaware}. MedCAL-Bench evaluates
foundation-model representations and CSAL strategies across medical
classification and segmentation tasks
\citep{zhu2025medcalbench}. The four groups encode different selection preferences.
Typicality emphasizes dense local regions but can underrepresent
sparse structure. Aggressive coverage can favor atypical points.
Continuous diversity spreads the selected support globally, although
its regularization may not match the downstream task. Task-specific
hybrids introduce additional signals whose usefulness depends on the
modality and application.
Because CSAL provides no target labels or trained task model, it
cannot use task feedback to determine which preference is best for a
given pool and budget. Within this taxonomy, \mthm{thm:unification} exactly instantiates the
structural choices of the evaluated $K=b$ TypiClust specialization,
the deterministic ProbCover solver, and the published ActiveFT
objective. These instances represent the typicality, coverage, and
continuous-diversity routes, respectively. The framework exposes
their shared allocation structure while retaining their distinct
selection rules. Its balanced entropic branch defines $\eps$-AS,
while CoreSet and farthest-point selection provide additional
geometry-based coverage rules.

\subsection{Optimal Transport for Selection}
Optimal transport provides a geometry for comparing and coupling
probability measures~\citep{villani2009optimal,peyre2019computational}.
Solving the Kantorovich problem exactly is costly, and entropic
regularization with Sinkhorn scaling has made it practical at
scale~\citep{cuturi2013sinkhorn,altschuler2017nearlinear}, which is
the solver we use.  Variational results show that, as the regularizer
tends to zero, entropic optima approach unregularized transport
optima~\citep{carlier2017convergence,leonard2014survey,braides2002gamma,dalmaso1993introduction}. For the fixed-support balanced branch used by $\eps$-AS, our proof
gives the finite-dimensional properties needed here directly:
existence and uniqueness at positive temperature, monotone transport
cost and KL along the entropic path, and both endpoint limits.
Complementary statistical work analyzes finite-sample estimation of
entropic transport quantities
~\citep{genevay2019sample,mena2019statistical}. Our downstream minimax
theorem is derived from an explicit pair of indistinguishable
Lipschitz tasks. Optimal transport has several roles in data selection.
\citet{mahmood2022lowbudget} formulate low-budget selection as an
integer program that matches a target distribution in Wasserstein
distance. AQOT combines model-derived informativeness and
representativeness through entropy-regularized transport during
iterative active learning~\citep{zhu2023aqot}.
\citet{riaz2023partialot} study semi-relaxed partial transport, in
which relaxing the source marginal can concentrate mass on a support
subset. UniPROT seeks a uniformly weighted prototype distribution and
reformulates its partial-transport objective as a submodular problem
with a greedy approximation guarantee~\citep{chanda2026uniprot}.
Together, these formulations show how support variables and marginal
constraints determine the behavior of transport-based selection. Our
generalized program makes these structural choices explicit. Its
balanced entropic realization uses a cost-equivariant rule to adapt
the operating scale to the unlabeled pool and annotation budget.

\subsection{Intrinsic Dimension and Manifold Geometry}
Our sample-complexity characterization is governed by the intrinsic
dimension of the feature manifold rather than the ambient dimension,
entering through the squared-distance scaling of the closest-anchor
cost $m(b)\propto b^{-2/d_{\mathrm{int}}}$. The maximum-likelihood estimator
of intrinsic dimension from nearest-neighbor distances is due to
\citet{levina2004mle}. Our closest-anchor scaling is a budget-driven
analog in which the anchor count plays the role of the neighborhood
size.
\section{Proofs of Theorems and Lemmas}
\label{app:proofs}

\subsection{Proof of \mthmp{thm:unification}}
\label{app:proof_unification}

We first formulate a transport--selection program that accommodates
the structural choices made by existing cold-start selectors,
including both balanced fixed-support and semi-relaxed
learnable-support problems.  Each method is represented by a
realization with its own feasible set, cost, support energy, and
decoder, enabling exact variational identities across these
realizations.

\paragraph{The transport--selection program.}
Fix integers $n$ and $b$ with $2\le b\le n$.  Let
\[
 a:=\frac1n\mathbf 1_n,
 \qquad
 u:=\frac1b\mathbf 1_b,
\]
and define
\begin{align}
 \Pi_a^{\mathrm{sr}}
 &:=
 \left\{
 \pi\in\R_+^{n\times b}:
 \pi\mathbf 1_b=a
 \right\},\nonumber\\
 \Pi(a,u)
 &:=
 \left\{
 \pi\in\Pi_a^{\mathrm{sr}}:
 \pi^\top\mathbf 1_n=u
 \right\}.\label{eq:balanced_app}
\end{align}
The superscript ``sr'' means semi-relaxed: the mass of every pool
point is fixed, whereas the masses received by the representatives
are free.  A realization is a tuple
\[
 \mathfrak R
 =
 (\Theta_{\mathfrak R},
   \Pi_{\mathfrak R},
   C_{\mathfrak R},
   \Omega_{\mathfrak R},
   \mathcal D_{\mathfrak R}),
\]
where $\Theta_{\mathfrak R}$ is a feasible support set,
$\Pi_{\mathfrak R}(\theta)$ is an allocation-plan set,
$C_\theta$ is a cost matrix, $\Omega_{\mathfrak R}$ is a
support-only energy, and $\mathcal D_{\mathfrak R}$ is a decoder.  The
construction of this tuple may include the method's stated
pool-dependent preprocessing with its seed, repair, and tie
conventions fixed.  This includes, for example, constructing
TypiClust's cells before defining its local-support set.
For $\eps\ge0$, define
\begin{align}
 \mathcal J_{\eps}^{\mathfrak R}(\theta,\pi)
 &:=
 \langle\pi,C_\theta\rangle
 +\eps\KL\!\left(\pi\,\middle\|\,a\otimes u\right)
 +\Omega_{\mathfrak R}(\theta),
 \label{eq:general_program_app}\\
 \overline{\mathcal J}_{\eps}^{\mathfrak R}(\theta)
 &:=
 \inf_{\pi\in\Pi_{\mathfrak R}(\theta)}
 \mathcal J_{\eps}^{\mathfrak R}(\theta,\pi).
 \label{eq:reduced_program_app}
\end{align}
At $\eps=0$, the KL term is absent.  Thus no absolute-continuity
convention is needed at the zero-temperature endpoint.  Throughout
the proof, one total order $\prec$ on pool indices is fixed before
any optimization and is used by every method to resolve a genuine
tie.

For the finite-pool statement we also fix an explicit clustering
repair.  Process empty cells in increasing cluster-index order.  For
the current empty cell, consider every point whose present cell has
at least two indices, choose a point of largest distance from its
present center, break a tie by $\prec$, and move that point to the
empty cell.  Such a donor point exists whenever an empty cell remains
because $b\le n$.  If the arithmetic mean of a nonempty cell has zero
norm, choose as its medoid an index minimizing the sum of squared
Euclidean distances to the other points in that cell, break a tie by
$\prec$, and use that unit-norm pool point as the representative.
This rule is deterministic, terminates after at most $b-1$ repairs,
and always produces $b$ nonempty cells and nonzero representatives.

\begin{definition}[Cold-start TypiClust specialization evaluated here]
\label{def:typiclust}
Let $V_1,\ldots,V_b$ be the nonempty cells returned by the $K=b$
clustering stage, viewed as a partition of the pool indices.  For
$i\in[n]$, let $\mathcal N_i$ be that method's in-cell nearest-neighbor
set and define
\[
 \ell_i:=
 \begin{cases}
 \displaystyle
 \frac1{|\mathcal N_i|}
 \sum_{j\in\mathcal N_i}\|z_i-z_j\|_2,
 &|\mathcal N_i|>0,\\[6pt]
 0,&|\mathcal N_i|=0,
 \end{cases}
 \qquad
 \rho_i:=\ell_i^{-1}.
\]
Here $\mathcal N_i$ contains the $K_t=20$ nearest in-cell indices
distinct from $i$ used by TypiClust, or every available one if there
are fewer than $K_t$.
The second case covers a singleton cell.  Throughout we use the
extended-real convention $1/0:=+\infty$.
TypiClust chooses, from each $V_k$, an index minimizing $\ell_i$,
equivalently maximizing $\rho_i$.
\end{definition}

\begin{definition}[Deterministic probability coverage and ProbCover]
\label{def:probcover}
For a radius $\delta>0$ and an index set $S\subseteq[n]$, define
\[
 F_\delta(S):=
 \sum_{i=1}^n
 a_i\,
 \mathbf 1
 \left\{
 \min_{j\in S}d(z_i,z_j)<\delta
 \right\}.
\]
The minimum over the empty set is $+\infty$, so
$F_\delta(\varnothing)=0$.
The maximum-coverage objective is
$\max_{|S|=b}F_\delta(S)$.  The deterministic no-repeat ProbCover
specialization used here is its residual greedy solver: at each step
it adds an as-yet unselected point with the largest increase in
$F_\delta$, and uses $\prec$ if all remaining candidates tie.
\end{definition}

\begin{definition}[Published ActiveFT variational objective]
\label{def:activeft}
For $\tau>0$, $\lambda>0$, and
$\Theta=(\theta_1,\ldots,\theta_b)
\in(\mathbb S^{D-1})^b$, put
$s_{ik}:=\langle z_i,\theta_k\rangle$ and
$c(i):=\argmax_{k\in[b]}s_{ik}$, choosing the smallest prototype
index in a tie.  The published ActiveFT variational objective
\citep[Eq.~(11)]{xie2023activeft} is
\[
 \mathcal L_{\mathrm{AFT}}(\Theta)
 :=
 -\frac1{n\tau}\sum_{i=1}^n s_{i,c(i)}
 +\frac{\lambda}{b}\sum_{k=1}^b
 \log\sum_{\substack{l=1\\l\ne k}}^b
 \exp\!\left(\frac{\langle\theta_k,\theta_l\rangle}{\tau}\right)
\]
and decodes the prototypes in their fixed order, assigning each one
to its nearest as-yet unselected pool index.
\end{definition}

\begin{restated}{Theorem}{\mref{thm:unification}}{Exact variational subsumption of cold-start selectors}
For every finite unlabeled pool and every budget $2\le b\le n$, the
transport--selection program
\eqref{eq:general_program_app}--\eqref{eq:reduced_program_app}
variationally subsumes
\emph{(i)} the $K=b$, one-point-per-repaired-cell cold-start
TypiClust specialization evaluated here, as a cell-constrained
local-support realization,
\emph{(ii)} deterministic maximum probability coverage as a
truncated-cost free-support realization, with the no-repeat ProbCover
rule of Definition~\ref{def:probcover} as its exact residual greedy
solver, and
\emph{(iii)} the published ActiveFT variational objective with its
ordered no-repeat decoder, as a learnable-support semi-relaxed
realization.
The proposed $\eps$-AS method is the fixed-support balanced entropic
realization of the same program.  All objective identities hold
without a distributional assumption.  Clustering cells and all ties
are handled by the deterministic repair and ordering fixed above.
\end{restated}

\begin{proof}
We prove four objective identities.  Before doing so, we record an
elementary row-minimization identity used twice below.  Let
$v=(v_1,\ldots,v_b)\in\R^b$ and let
$q=(q_1,\ldots,q_b)\in\R_+^b$ satisfy
$\sum_kq_k=a_i$.  Then
\begin{align}
 \sum_{k=1}^b q_kv_k
 &\ge
 \sum_{k=1}^b q_k\min_{l\in[b]}v_l
 \annotnl{By $v_k\ge\min_l v_l$ and $q_k\ge0$}
 \nonumber\\
 &=
 \left(\sum_{k=1}^bq_k\right)\min_{l\in[b]}v_l
 \annotnl{By factoring out the term independent of $k$}
 \nonumber\\
 &=a_i\min_{l\in[b]}v_l
 \annotref{By $\sum_kq_k=a_i$}.
 \label{eq:row_lower_app}
\end{align}
If $k_i^\star$ is the $\prec$-first minimizer of $v_k$, the choice
$q_{k_i^\star}=a_i$ and $q_k=0$ for $k\ne k_i^\star$ attains equality
in \eqref{eq:row_lower_app}.  Consequently,
\begin{equation}\label{eq:row_identity_app}
 \inf_{\substack{q\ge0\\\sum_kq_k=a_i}}
 \sum_{k=1}^bq_kv_k
 =
 a_i\min_{k\in[b]}v_k.
\end{equation}

\par\medskip\noindent\textbf{Step 1: the evaluated TypiClust specialization is a
cell-constrained realization.}\quad
The clustering stage itself has a semi-relaxed transport
representation.  For candidate centers
$\mu_1,\ldots,\mu_b\in\R^D$ and Euclidean clustering distance
$d_{\mathrm{cl}}(z,\mu):=\|z-\mu\|_2$, apply
\eqref{eq:row_identity_app} to each row:
\begin{align}
 &\inf_{\alpha\in\Pi_a^{\mathrm{sr}}}
 \sum_{i=1}^n\sum_{k=1}^b
 \alpha_{ik}d_{\mathrm{cl}}(z_i,\mu_k)^2
 \nonumber\\
 &\quad=
 \sum_{i=1}^n
 \inf_{\substack{\alpha_{i\cdot}\ge0\\
                 \sum_k\alpha_{ik}=a_i}}
 \sum_{k=1}^b
 \alpha_{ik}d_{\mathrm{cl}}(z_i,\mu_k)^2
 \annotnl{By separability of the row constraints}
 \nonumber\\
 &\quad=
 \sum_{i=1}^n
 a_i\min_{k\in[b]}d_{\mathrm{cl}}(z_i,\mu_k)^2
 \annotref{By Equation~\ref{eq:row_identity_app}}.
 \label{eq:kmeans_transport_app}
\end{align}
Minimizing \eqref{eq:kmeans_transport_app} over the centers is exactly
the $K=b$ clustering objective and induces its cells
$V_1,\ldots,V_b$.  For a finite-run implementation, the repaired
cells returned by the clustering stage are fixed before defining the
readout realization below.  The resulting readout identity applies to
the returned clustering solution.

For the readout stage, take
\[
 \Theta_{\mathrm T}:=
 \prod_{k=1}^b V_k,
 \quad
 \Pi_{\mathrm T}:=\Pi(a,u),
 \quad
 C^{\mathrm T}_{ik}(s):=\ell_{s_k},
 \quad
 \Omega_{\mathrm T}:=0,
 \quad
 \eps:=0,
\]
where $s=(s_1,\ldots,s_b)$ and $s_k\in V_k$.  For every
$\pi\in\Pi(a,u)$,
\begin{align}
 \left\langle\pi,C^{\mathrm T}(s)\right\rangle
 &=
 \sum_{i=1}^n\sum_{k=1}^b
 \pi_{ik}\ell_{s_k}
 \annotnl{By the definition of the matrix inner product}
 \nonumber\\
 &=
 \sum_{k=1}^b
 \ell_{s_k}\sum_{i=1}^n\pi_{ik}
 \annotnl{By regrouping the finite sums}
 \nonumber\\
 &=
 \sum_{k=1}^b\ell_{s_k}u_k
 \annotnl{By Equation~\ref{eq:balanced_app}}
 \nonumber\\
 &=
 \frac1b\sum_{k=1}^b\ell_{s_k}
 \annotref{By $u_k=1/b$}.
 \label{eq:typi_objective_app}
\end{align}
The right-hand side of \eqref{eq:typi_objective_app} does not depend
on $\pi$, so minimizing out $\pi$ leaves that expression unchanged.
Moreover,
\begin{align*}
 \argmin_{s\in\Theta_{\mathrm T}}
 \overline{\mathcal J}^{\mathrm T}_0(s)
 &=
 \argmin_{\substack{s_k\in V_k\\k\in[b]}}
 \frac1b\sum_{k=1}^b\ell_{s_k}
 \annotnl{By Equation~\ref{eq:typi_objective_app}}
 \nonumber\\
 &=
 \prod_{k=1}^b
 \argmin_{i\in V_k}\ell_i
 \annotnl{By separability over $s_k$}
 \nonumber\\
 &=
 \prod_{k=1}^b
 \argmax_{i\in V_k}\rho_i
 \annotnl{Since $x\mapsto1/x$ is decreasing on $(0,\infty)$}.
\end{align*}
The identity decoder $\mathcal D_{\mathrm T}(s)=\{s_1,\ldots,s_b\}$
therefore returns exactly the TypiClust selection.  Notice that no
equal-cell-size assumption was used.

\par\medskip\noindent\textbf{Step 2: probability coverage is a truncated-cost
realization, and ProbCover solves every residual step exactly.}\quad
Let
$\Theta_{\mathrm P}:=\{S\subseteq[n]:|S|=b\}$.  Enumerate each
$S\in\Theta_{\mathrm P}$ as $(s_1,\ldots,s_b)$ using the fixed order
$\prec$, put $q_k:=z_{s_k}$, let
$\Pi_{\mathrm P}:=\Pi_a^{\mathrm{sr}}$, and define
\[
 C^\delta_{ik}(S):=
 \mathbf1\{d(z_i,q_k)\ge\delta\},
 \qquad
 \Omega_{\mathrm P}:=0,
 \qquad
 \eps:=0,
\]
for $S\in\Theta_{\mathrm P}$.  For fixed $S$,
\begin{align}
 \inf_{\pi\in\Pi_a^{\mathrm{sr}}}
 \left\langle\pi,C^\delta(S)\right\rangle
 &=
 \sum_{i=1}^n
 a_i\min_{k\in[b]}
 \mathbf1\{d(z_i,q_k)\ge\delta\}
 \annotnl{By Equation~\ref{eq:row_identity_app}}
 \nonumber\\
 &=
 \sum_{i=1}^n
 a_i\,
 \mathbf1\{d(z_i,z_j)\ge\delta\text{ for every }j\in S\}
 \annotnl{By minimizing the binary indicators}
 \nonumber\\
 &=
 \sum_{i=1}^na_i
 -
 \sum_{i=1}^n
 a_i\,
 \mathbf1\{\text{$z_i$ is covered by at least one $z_j$, $j\in S$}\}
 \annotnl{By complementary indicators}
 \nonumber\\
 &=1-F_\delta(S)
 \annotref{By Definition~\ref{def:probcover}}.
 \label{eq:probcover_exact_app}
\end{align}
It follows directly that
\begin{align}
 \argmin_{S\in\Theta_{\mathrm P}}
 \overline{\mathcal J}^{\mathrm P}_0(S)
 &=
 \argmin_{\substack{S\subseteq[n]\\|S|=b}}
 [1-F_\delta(S)]
 \annotnl{By Equation~\ref{eq:probcover_exact_app}}
 \nonumber\\
 &=
 \argmax_{\substack{S\subseteq[n]\\|S|=b}}
 F_\delta(S)
 \annotref{By adding $1$ and multiplying by $-1$}.
 \label{eq:maxcover_exact_app}
\end{align}

ProbCover applies a greedy solver to
\eqref{eq:maxcover_exact_app}.  We now show that each greedy step is
an exact residual transport minimization.  Let $S_{t-1}$ be
the set chosen before step $t$, let $U_{t-1}$ be the points not
covered by $S_{t-1}$, and put
$A_{t-1}:=\sum_{i\in U_{t-1}}a_i$.  For a candidate
$j\in[n]\setminus S_{t-1}$, define
\[
 R_t(j):=
 \sum_{i\in U_{t-1}}
 a_i\mathbf1\{d(z_i,z_j)\ge\delta\}.
\]
This is itself the reduced objective of an exact residual
realization.  Take
\[
 \Theta_t:=[n]\setminus S_{t-1},
 \qquad
 \Pi_t:=\Pi_a^{\mathrm{sr}},
 \qquad
 C^{(t,j)}_{ik}:=
 \mathbf1\{i\in U_{t-1}\}
 \mathbf1\{d(z_i,z_j)\ge\delta\},
\]
use $\Omega_t:=0$ and $\eps:=0$, and decode the support variable
$j$ itself.  The $b$ columns of $C^{(t,j)}$ are identical.  Hence,
for every $\pi\in\Pi_a^{\mathrm{sr}}$,
\begin{align*}
 \left\langle\pi,C^{(t,j)}\right\rangle
 &=
 \sum_{i=1}^n\sum_{k=1}^b
 \pi_{ik}
 \mathbf1\{i\in U_{t-1}\}
 \mathbf1\{d(z_i,z_j)\ge\delta\}
 \annotnl{By the definition of the matrix inner product}
 \nonumber\\
 &=
 \sum_{i=1}^n
 \mathbf1\{i\in U_{t-1}\}
 \mathbf1\{d(z_i,z_j)\ge\delta\}
 \sum_{k=1}^b\pi_{ik}
 \annotnl{By factoring out the indicators independent of $k$}
 \nonumber\\
 &=
 \sum_{i=1}^n
 a_i\mathbf1\{i\in U_{t-1}\}
 \mathbf1\{d(z_i,z_j)\ge\delta\}
 \annotnl{By the row constraint in $\Pi_a^{\mathrm{sr}}$}
 \nonumber\\
 &=
 \sum_{i\in U_{t-1}}
 a_i\mathbf1\{d(z_i,z_j)\ge\delta\}
 \annotnl{By removing the previously selected indices}
 \nonumber\\
 &=R_t(j).
 \annotnl{By the definition of $R_t(j)$}
\end{align*}
The final expression is independent of $\pi$, so minimizing out the
plan gives
$\overline{\mathcal J}^{\,t}_0(j)=R_t(j)$ for every candidate $j$.
Thus the residual program and the greedy update have exactly the
same support objective, including the case $U_{t-1}=\varnothing$,
where all candidates tie and $\prec$ resolves the tie.

Then
\begin{align}
 R_t(j)
 &=
 \sum_{i\in U_{t-1}}a_i
 -
 \sum_{i\in U_{t-1}}
 a_i\mathbf1\{d(z_i,z_j)<\delta\}
 \annotnl{By complementary indicators}
 \nonumber\\
 &=
 A_{t-1}
 -
 \left[
 F_\delta(S_{t-1}\cup\{j\})-F_\delta(S_{t-1})
 \right]
 \annotref{By the definition of newly covered points}.
 \label{eq:residual_cover_app}
\end{align}
Since $A_{t-1}$ is independent of $j$,
\begin{align*}
 \argmin_{j\in[n]\setminus S_{t-1}}R_t(j)
 &=
 \argmax_{j\in[n]\setminus S_{t-1}}
 \left[
 F_\delta(S_{t-1}\cup\{j\})-F_\delta(S_{t-1})
 \right]
 \annotnl{By Equation~\ref{eq:residual_cover_app}}.
\end{align*}
The right-hand side is precisely the ProbCover update.  Thus the
method is subsumed at the level at which it is defined: an exact
residual greedy solver of the truncated-cost realization.

We next establish the monotonicity and submodularity underlying the
approximation guarantee.  For
$j\in[n]$ and $A\subseteq[n]$, define
\[
 N_\delta(j):=\{i\in[n]:d(z_i,z_j)<\delta\},
 \qquad
 \operatorname{Cov}(A):=\bigcup_{j\in A}N_\delta(j).
\]
By Definition~\ref{def:probcover},
\begin{equation}\label{eq:coverage_set_app}
 F_\delta(A)
 =
 \sum_{i\in\operatorname{Cov}(A)}a_i.
\end{equation}
If $A\subseteq B$, then every set in the union defining
$\operatorname{Cov}(A)$ also occurs in the union defining
$\operatorname{Cov}(B)$.  Hence
\[
 \operatorname{Cov}(A)\subseteq\operatorname{Cov}(B),
\]
and therefore
\begin{align}
 F_\delta(A)
 &=
 \sum_{i\in\operatorname{Cov}(A)}a_i
 \annotnl{By Equation~\ref{eq:coverage_set_app}}
 \nonumber\\
 &\le
 \sum_{i\in\operatorname{Cov}(B)}a_i
 \annotnl{By set inclusion and $a_i\ge0$}
 \nonumber\\
 &=F_\delta(B).
 \annotref{By Equation~\ref{eq:coverage_set_app}}
 \label{eq:coverage_monotone_app}
\end{align}
This proves monotonicity.  Moreover, for $j\notin B$,
\[
 N_\delta(j)\setminus\operatorname{Cov}(B)
 \subseteq
 N_\delta(j)\setminus\operatorname{Cov}(A),
\]
because $\operatorname{Cov}(A)\subseteq\operatorname{Cov}(B)$.
Consequently,
\begin{align}
 F_\delta(B\cup\{j\})-F_\delta(B)
 &=
 \sum_{i\in N_\delta(j)\setminus\operatorname{Cov}(B)}a_i
 \annotnl{By cancellation outside the newly covered set}
 \nonumber\\
 &\le
 \sum_{i\in N_\delta(j)\setminus\operatorname{Cov}(A)}a_i
 \annotnl{By set inclusion and $a_i\ge0$}
 \nonumber\\
 &=
 F_\delta(A\cup\{j\})-F_\delta(A).
 \annotref{By cancellation outside the newly covered set}
 \label{eq:coverage_submodular_app}
\end{align}
This is the diminishing-returns definition of submodularity.
Finally, enumerate any finite
$T\setminus A=\{j_1,\ldots,j_r\}$ and put
$A_l:=A\cup\{j_1,\ldots,j_l\}$.  Telescoping and
\eqref{eq:coverage_submodular_app} give
\begin{align}
 F_\delta(A\cup T)-F_\delta(A)
 &=
 \sum_{l=1}^r
 \left[F_\delta(A_l)-F_\delta(A_{l-1})\right]
 \annotnl{By telescoping}
 \nonumber\\
 &\le
 \sum_{l=1}^r
 \left[F_\delta(A\cup\{j_l\})-F_\delta(A)\right].
 \annotref{By Equation~\ref{eq:coverage_submodular_app}}
 \label{eq:coverage_union_bound_app}
\end{align}

For completeness, this realization also recovers the standard greedy
coverage guarantee.  Let $S^\star$ maximize $F_\delta$ among sets of
size at most $b$.  Monotonicity and submodularity give
\begin{align}
 F_\delta(S^\star)-F_\delta(S_t)
 &\le
 F_\delta(S_t\cup S^\star)-F_\delta(S_t)
 \annotnl{By Equation~\ref{eq:coverage_monotone_app}}
 \nonumber\\
 &\le
 \sum_{j\in S^\star\setminus S_t}
 \left[
 F_\delta(S_t\cup\{j\})-F_\delta(S_t)
 \right]
 \annotnl{By Equation~\ref{eq:coverage_union_bound_app}}
 \nonumber\\
 &\le
 b\max_{j\in[n]\setminus S_t}
 \left[
 F_\delta(S_t\cup\{j\})-F_\delta(S_t)
 \right]
 \annotnl{By bounding the sum with at most $b$ terms}
 \nonumber\\
 &=
 b\left[
 F_\delta(S_{t+1})-F_\delta(S_t)
 \right]
 \annotref{By the greedy choice}.
 \label{eq:cover_gap_step_app}
\end{align}
Rearranging \eqref{eq:cover_gap_step_app} yields
\[
 F_\delta(S^\star)-F_\delta(S_{t+1})
 \le
 \left(1-\frac1b\right)
 \left[F_\delta(S^\star)-F_\delta(S_t)\right].
\]
Starting from $F_\delta(S_0)=0$ and applying this inequality $b$
times gives
\begin{align*}
 F_\delta(S^\star)-F_\delta(S_b)
 &\le
 \left(1-\frac1b\right)^bF_\delta(S^\star)
 \annotnl{By induction on $t$}
 \nonumber\\
 F_\delta(S_b)
 &\ge
 \left[1-\left(1-\frac1b\right)^b\right]F_\delta(S^\star)
 \annotnl{By rearranging the preceding inequality}
 \nonumber\\
 &\ge
 (1-e^{-1})F_\delta(S^\star)
 \annotnl{By $(1-1/b)^b\le e^{-1}$}.
\end{align*}

\par\medskip\noindent\textbf{Step 3: the published ActiveFT objective is a
learnable-support semi-relaxed realization.}\quad
Take
\[
 \Theta_{\mathrm A}:=(\mathbb S^{D-1})^b,
 \quad
 \Pi_{\mathrm A}:=\Pi_a^{\mathrm{sr}},
 \quad
 C^{\mathrm A}_{ik}(\Theta):=1-s_{ik},
 \quad
 \eps:=0,
\]
and
\[
 \Omega_{\mathrm A}(\Theta):=
 \frac{\lambda\tau}{b}
 \sum_{k=1}^b
 \log\sum_{\substack{l=1\\l\ne k}}^b
 \exp\!\left(\frac{\langle\theta_k,\theta_l\rangle}{\tau}\right).
\]
For fixed $\Theta$, row minimization gives
\begin{align}
 \inf_{\pi\in\Pi_a^{\mathrm{sr}}}
 \left\langle\pi,C^{\mathrm A}(\Theta)\right\rangle
 &=
 \sum_{i=1}^n
 a_i\min_{k\in[b]}(1-s_{ik})
 \annotnl{By Equation~\ref{eq:row_identity_app}}
 \nonumber\\
 &=
 \sum_{i=1}^n
 a_i\left(1-\max_{k\in[b]}s_{ik}\right)
 \annotnl{By $\min_k(1-s_{ik})=1-\max_ks_{ik}$}
 \nonumber\\
 &=
 \sum_{i=1}^na_i
 -
 \sum_{i=1}^na_i\max_{k\in[b]}s_{ik}
 \annotnl{By distributing the finite weighted sum}
 \nonumber\\
 &=
 1-\frac1n\sum_{i=1}^ns_{i,c(i)}
 \annotref{By $a_i=1/n$ and the definition of $c(i)$}.
 \label{eq:aft_row_app}
\end{align}
Adding the support energy and using Definition~\ref{def:activeft},
\begin{align}
 \overline{\mathcal J}^{\mathrm A}_0(\Theta)
 &=
 1-\frac1n\sum_{i=1}^ns_{i,c(i)}
 +\frac{\lambda\tau}{b}
 \sum_{k=1}^b
 \log\sum_{l\ne k}
 \exp\!\left(\frac{\langle\theta_k,\theta_l\rangle}{\tau}\right)
 \annotnl{By Equation~\ref{eq:aft_row_app} and the definition of $\Omega_{\mathrm A}$}
 \nonumber\\
 &=
 1+\tau\mathcal L_{\mathrm{AFT}}(\Theta)
 \annotref{By multiplying $\mathcal L_{\mathrm{AFT}}$ by $\tau$}.
 \label{eq:aft_exact_app}
\end{align}
Because $\tau>0$, the positive affine transformation in
\eqref{eq:aft_exact_app} preserves all global minimizers.  The
minimizing hard plan sends row $i$ to $c(i)$,
\[
 \pi^{\mathrm A}_{ik}
 =
 a_i\mathbf1\{k=c(i)\},
\]
and its column marginal is
\begin{align*}
 r_k
 &:=
 \sum_{i=1}^n\pi^{\mathrm A}_{ik}
 =
 \sum_{i=1}^na_i\mathbf1\{k=c(i)\}
 \annotnl{By the definition of the column marginal}
 \nonumber\\
 &=
 \frac{|\{i:c(i)=k\}|}{n}
 \annotnl{By $a_i=1/n$}.
\end{align*}
Thus the published ActiveFT objective is exactly semi-relaxed, with
data-dependent prototype masses $r_k$.  Finally,
\[
 s_k
 :=
 \argmax_{i\in[n]\setminus\{s_1,\ldots,s_{k-1}\}}
 \langle z_i,\theta_k\rangle,
 \qquad k=1,\ldots,b,
\]
and set
$\mathcal D_{\mathrm A}(\Theta):=\{s_1,\ldots,s_b\}$.
This is the ordered no-repeat decoder used by ActiveFT.  Together,
the semi-relaxed allocation and this decoder establish variational
subsumption of both its objective and its readout.

\par\medskip\noindent\textbf{Step 4: $\eps$-AS is the fixed-support balanced entropic
realization.}\quad
Fix the $b$ anchors returned by the stated clustering stage, take
\[
 \Theta_{\eps\text{-AS}}:=\{(\mu_1,\ldots,\mu_b)\},
 \quad
 \Pi_{\eps\text{-AS}}:=\Pi(a,u),
 \quad
 C_{ik}:=1-\langle z_i,\mu_k\rangle,
 \quad
 \Omega_{\eps\text{-AS}}:=0,
\]
and let $\mathcal D_{\eps\text{-AS}}$ be the round-robin decoder
in \meqref{eq:roundrobin}.  For every $\pi\in\Pi(a,u)$,
\begin{align}
 \KL(\pi\|a\otimes u)
 &=
 \sum_{i,k}\pi_{ik}
 \log\frac{\pi_{ik}}{a_iu_k}
 \annotnl{By the definition of KL divergence}
 \nonumber\\
 &=
 \sum_{i,k}\pi_{ik}\log\pi_{ik}
 -
 \sum_{i,k}\pi_{ik}\log\frac1{nb}
 \annotnl{By $a_iu_k=1/(nb)$}
 \nonumber\\
 &=
 \sum_{i,k}\pi_{ik}\log\pi_{ik}
 +\log(nb)\sum_{i,k}\pi_{ik}
 \annotnl{By $-\log(1/(nb))=\log(nb)$}
 \nonumber\\
 &=
 \sum_{i,k}\pi_{ik}\log\pi_{ik}+\log(nb)
 \annotnl{Since every feasible coupling has unit mass}
 \nonumber\\
 &=
 \Hreg(\pi)+1+\log(nb)
 \annotref{By the definition of $\Hreg(\pi)$}.
 \label{eq:kl_hreg_app}
\end{align}
The last two terms in \eqref{eq:kl_hreg_app} are independent of
$\pi$.  Consequently \eqref{eq:general_program_app} and
\meqref{eq:eot} have exactly the same minimizer.  This is the
fixed-support, balanced, entropic realization claimed in the theorem.
Combining Steps 1--4 completes the proof.
\end{proof}

\paragraph{Structural interpretation.}
Theorem~\mref{thm:unification} is a finite-pool structural statement.
Under the stated repair and tie conventions, it gives an exact
objective-level realization and decoder-level equivalence for each
evaluated method. The common program organizes the selectors along
five design axes: feasible representatives, allocation constraints,
transport costs, support regularization, and discrete decoding. Each
method retains its own realization on these axes, while the reduced
objective and decoder preserve its original selection rule. Within
this common coordinate system, $\eps$-AS occupies the fixed-support,
balanced branch and adapts the entropic scale of its allocation.
Performance and statistical consequences are established separately
by Theorems~\mref{thm:lowerbound} and~\mref{thm:eps_star}.

\subsection{The entropic regularizer orders assignment sharpness}

\begin{lemma}[Monotone entropic path and its two endpoints]
\label{lem:anchorspacing}
Fix a finite cost matrix $C$ and the balanced polytope $\Pi(a,u)$.
For $\eps>0$, let
\[
 \pi_\eps:=
 \argmin_{\pi\in\Pi(a,u)}
 \left\{
 \langle\pi,C\rangle
 +\eps\KL(\pi\|a\otimes u)
 \right\}.
\]
Then:
\begin{enumerate}[label=(\roman*),leftmargin=2.1em]
\item $\pi_\eps$ is unique;
\item if $0<\eps_1<\eps_2$, then
\[
 \KL(\pi_{\eps_2}\|a\otimes u)
 \le
 \KL(\pi_{\eps_1}\|a\otimes u),
 \qquad
 \langle\pi_{\eps_1},C\rangle
 \le
 \langle\pi_{\eps_2},C\rangle;
\]
\item every cluster point as $\eps\downarrow0$ is an unregularized
balanced OT minimizer, whereas
$\pi_\eps\to a\otimes u$ in $\ell_1$ as $\eps\to\infty$.
\end{enumerate}
Thus $\eps$ is a genuine assignment-sharpness knob within the
balanced fixed-support branch.
\end{lemma}

\begin{proof}
\noindent\textbf{Item (i).}\quad
The polytope $\Pi(a,u)$ is convex.  The function
$\pi\mapsto\langle\pi,C\rangle$ is affine, and
$\pi\mapsto\KL(\pi\|a\otimes u)$ is strictly convex because every
entry of $a\otimes u$ is positive.  A positive multiple of a strictly
convex function plus an affine function is strictly convex.
Therefore it has at most one minimizer on a convex set.  Compactness
of $\Pi(a,u)$ and continuity with the convention $0\log0=0$ give
existence, hence the minimizer is unique.

\par\medskip\noindent\textbf{Item (ii).}\quad
Write
\[
 L_j:=\langle\pi_{\eps_j},C\rangle,
 \qquad
 K_j:=\KL(\pi_{\eps_j}\|a\otimes u),
 \qquad j\in\{1,2\}.
\]
Optimality of $\pi_{\eps_1}$ when the coefficient is $\eps_1$ gives
\begin{align}
 L_1+\eps_1K_1
 &\le L_2+\eps_1K_2
 \annotnl{Since $\pi_{\eps_2}$ is feasible at $\eps_1$}
 \nonumber\\
 L_1-L_2
 &\le \eps_1(K_2-K_1)
 \annotref{By rearranging}.
 \label{eq:path_first_app}
\end{align}
Optimality of $\pi_{\eps_2}$ when the coefficient is $\eps_2$ gives
\begin{align}
 L_2+\eps_2K_2
 &\le L_1+\eps_2K_1
 \annotnl{Since $\pi_{\eps_1}$ is feasible at $\eps_2$}
 \nonumber\\
 L_1-L_2
 &\ge \eps_2(K_2-K_1)
 \annotref{By rearranging}.
 \label{eq:path_second_app}
\end{align}
Combining \eqref{eq:path_first_app} and
\eqref{eq:path_second_app},
\begin{equation}\label{eq:path_sandwich_app}
 \eps_2(K_2-K_1)
 \le L_1-L_2
 \le \eps_1(K_2-K_1).
\end{equation}
Suppose for contradiction that $K_2-K_1>0$.  Dividing the outer
inequality in \eqref{eq:path_sandwich_app} by this positive number
would give $\eps_2\le\eps_1$, contradicting
$\eps_1<\eps_2$.  Hence $K_2-K_1\le0$, which proves $K_2\le K_1$.
Substituting this sign into \eqref{eq:path_first_app} gives
\[
 L_1-L_2
 \le \eps_1(K_2-K_1)
 \le0,
\]
so $L_1\le L_2$.  Both claimed monotonicities follow.

\par\medskip\noindent\textbf{Item (iii), the zero-temperature endpoint.}\quad
Let $\pi_0$ be any minimizer of
$\langle\pi,C\rangle$ on $\Pi(a,u)$.  Optimality of $\pi_\eps$ gives
\begin{align}
 \langle\pi_\eps,C\rangle
 +\eps\KL(\pi_\eps\|a\otimes u)
 &\le
 \langle\pi_0,C\rangle
 +\eps\KL(\pi_0\|a\otimes u)
 \annotnl{Since $\pi_0$ is feasible}
 \nonumber\\
 \langle\pi_\eps,C\rangle-\langle\pi_0,C\rangle
 &\le
 \eps\left[
 \KL(\pi_0\|a\otimes u)
 -
 \KL(\pi_\eps\|a\otimes u)
 \right]
 \annotnl{By rearranging the preceding inequality}
 \nonumber\\
 &\le
 \eps\KL(\pi_0\|a\otimes u)
 \annotref{By nonnegativity of KL divergence}.
 \label{eq:zero_endpoint_app}
\end{align}
The left-hand side of \eqref{eq:zero_endpoint_app} is nonnegative
because $\pi_0$ minimizes the unregularized cost.  Hence it converges
to zero with $\eps$.  The coupling polytope is compact, so every
sequence $\eps_j\downarrow0$ has a convergent subsequence.  By
continuity of $\pi\mapsto\langle\pi,C\rangle$ and
\eqref{eq:zero_endpoint_app}, the limit has the same cost as
$\pi_0$ and is therefore an unregularized OT minimizer.

\par\medskip\noindent\textbf{Item (iii), the infinite-temperature endpoint.}\quad
Let $q:=a\otimes u$.  This independent coupling belongs to
$\Pi(a,u)$ and satisfies $\KL(q\|q)=0$.  Optimality gives
\begin{align}
 \langle\pi_\eps,C\rangle
 +\eps\KL(\pi_\eps\|q)
 &\le
 \langle q,C\rangle
 +\eps\KL(q\|q)
 \annotnl{Since $q$ is feasible}
 \nonumber\\
 \eps\KL(\pi_\eps\|q)
 &\le
 \langle q,C\rangle-\langle\pi_\eps,C\rangle
 \annotnl{By $\KL(q\|q)=0$ and rearrangement}
 \nonumber\\
 &\le
 \langle q,C\rangle-\min_{\pi\in\Pi(a,u)}\langle\pi,C\rangle
 \annotref{Since $\pi_\eps$ is feasible}
 \nonumber\\
 &=:M_C<\infty.
 \label{eq:large_endpoint_app}
\end{align}
Dividing \eqref{eq:large_endpoint_app} by $\eps$ gives
$\KL(\pi_\eps\|q)\le M_C/\eps\to0$.  Pinsker's inequality then yields
\[
 \|\pi_\eps-q\|_1
 \le\sqrt{2\KL(\pi_\eps\|q)}
 \le\sqrt{\frac{2M_C}{\eps}}
 \longrightarrow0.
\]
This proves the second endpoint and completes the lemma.
\end{proof}

\subsection{Proof of \mthmp{thm:lowerbound}}
\label{app:proof_lowerbound}

The minimax statement is established for the task-agnostic Lipschitz
prediction class defined below.  This formulation isolates the
dependence of attainable prediction risk on feature-space coverage
and permits matching lower and upper bounds.

Fix $L>0$.  For a nonempty finite
$S\subseteq\operatorname{supp}(P)$, recall
\[
 Q_P(S):=
 \E_{z\sim P}\!\left[\min_{s\in S}C(z,s)\right].
\]
Let $\operatorname{Lip}_L(P)$ be all real functions on
$\operatorname{supp}(P)$ satisfying
$|f(z)-f(z')|\le Ld(z,z')$.  A deterministic cold-start procedure
first chooses $S\subseteq\operatorname{supp}(P)$ with
$1\le|S|\le b$ without seeing $f$, then observes the noiseless values
of $f$ on those points, and finally returns a measurable predictor
$\widehat f_{\mathcal A,\mathcal L;f}
:=\mathcal L(S,f|_S)$.  Let $\mathfrak A_b$ be the class of deterministic
selector--learner pairs satisfying these conditions, and define
\begin{equation*}
 \mathcal R_b^\star(P)
 :=
 \inf_{(\mathcal A,\mathcal L)\in\mathfrak A_b}
 \sup_{f\in\operatorname{Lip}_L(P)}
 \E_{z\sim P}
 \left[
   (\widehat f_{\mathcal A,\mathcal L;f}(z)-f(z))^2
 \right].
\end{equation*}
The order $\inf\sup$ makes explicit that the downstream task may be
chosen after the task-agnostic procedure has been fixed.
When a pool-based selector is random through the sampled pool or its
initialization, we evaluate the realized selector conditionally.  For
a realized nonempty selection $S$ and a specified learner
$\mathcal L$, define the random-design conditional risk
\[
 \mathcal R_{\mathrm{cond}}(S,\mathcal L)
 :=
 \sup_{f\in\operatorname{Lip}_L(P)}
 \E_{z\sim P}\!\left[
  \bigl(\mathcal L(S,f|_S)(z)-f(z)\bigr)^2
 \right].
\]
The expectation here is only over a fresh $z\sim P$. The realized
pool, clustering initialization, and $S$ are held fixed.  Statements
``in probability'' below refer to this conditional-risk random
variable, not to an unconditional
$\sup_f\E_{U_n,z}[\cdot]$ criterion.

\begin{lemma}[Anchor-to-query quantization transfer]
\label{lem:anchor_query_transfer}
Fix an anchor tuple
$\boldsymbol\mu=(\mu_1,\ldots,\mu_b)$ and points
$s_1,\ldots,s_b\in\mathbb S^{D-1}$.  For
$S=\{s_1,\ldots,s_b\}$,
\begin{equation}\label{eq:anchor_query_transfer}
 Q_P(S)
 \le
 2R_P(\boldsymbol\mu)
 2\sum_{k=1}^b
 p_k(\boldsymbol\mu)C(\mu_k,s_k).
\end{equation}
\end{lemma}

\begin{proof}
For a feature $z$, let $k=a_{\boldsymbol\mu}(z)$.  Since
$s_k\in S$, the triangle inequality and
$(a+b)^2\le2a^2+2b^2$ give
\begin{align*}
 \min_{s\in S}C(z,s)
 &\le C(z,s_k)
 =\frac12d(z,s_k)^2\\
 &\le d(z,\mu_k)^2+d(\mu_k,s_k)^2\\
 &=2C(z,\mu_k)+2C(\mu_k,s_k).
\end{align*}
Integrating this inequality over the Voronoi cells induced by
$a_{\boldsymbol\mu}$ gives
\eqref{eq:anchor_query_transfer}.
\end{proof}

\begin{restated}{Theorem}{\mref{thm:lowerbound}}{Unavoidable cold-start resolution and a matching rate}
Under \cassum{assumption:feature_geometry}, every deterministic
task-agnostic selector and learner obeys
\[
 \sup_{f\in\operatorname{Lip}_L(P)}
 \E_P[(\widehat f(z)-f(z))^2]
 \ge\frac{L^2}{2}m_P(b).
\]
If a selector returns a set $S$ satisfying
$Q_P(S)\le\kappa_qm_P(b)$, nearest-neighbor readout satisfies
\[
 \sup_{f\in\operatorname{Lip}_L(P)}
 \E_P[(\widehat f(z)-f(z))^2]
 \le2\kappa_qL^2m_P(b).
\]
Consequently, under Assumptions~\ref{assumption:feature_geometry}
and~\ref{assumption:intrinsic_dim}, the minimax rate is
$\Theta(b^{-2/d_{\mathrm{int}}})$ whenever such a
quantization-faithful selector is available.  Under
Assumptions~\ref{assumption:feature_geometry}
and~\ref{assumption:sampling}--\ref{assumption:localized_decoding},
for every fixed $\kappa_q>2(1+\kappa_d)$, the decoded set produced by
$\eps$-AS satisfies
\[
 \sup_{t\in[t_-,t_+]}
 \frac{Q_P(\widehat S(t\,\widehat m(b)))}{m_P(b)}
 \le\kappa_q
\]
with probability tending to one.  If
\cassum{assumption:intrinsic_dim} also holds, pairing these selections
with nearest-neighbor readout gives the stated conditional worst-case
risk order in probability.
\end{restated}

\begin{proof}
\noindent\textbf{Lower bound.}\quad
Fix an arbitrary deterministic cold-start procedure.  Because it
does not see the downstream task before querying, its selected set
$S$ is the same for every $f\in\operatorname{Lip}_L(P)$.  Define
\[
 d(z,S):=\min_{s\in S}d(z,s),
 \qquad
 f_0(z):=0,
 \qquad
 f_1(z):=L\,d(z,S).
\]
We first prove that $z\mapsto d(z,S)$ is $1$-Lipschitz.  For any
$z,z'$ and every $s\in S$, the triangle inequality gives
\[
 d(z,s)\le d(z,z')+d(z',s).
\]
Taking the infimum over $s\in S$ on both sides gives
\[
 d(z,S)\le d(z,z')+d(z',S),
\]
and therefore
\[
 d(z,S)-d(z',S)\le d(z,z').
\]
Interchanging $z$ and $z'$ gives
\[
 d(z',S)-d(z,S)\le d(z',z)=d(z,z').
\]
The last two inequalities imply
$|d(z,S)-d(z',S)|\le d(z,z')$.  Hence $f_1$ is $L$-Lipschitz, and
$f_0$ is also $L$-Lipschitz.

For every queried $s\in S$,
\[
 f_0(s)=0,
 \qquad
 f_1(s)=L\,d(s,S)=0,
\]
because $s\in S$ implies $d(s,S)=0$.  The learner therefore receives
identical observations under $f_0$ and $f_1$ and must output the same
function, denoted $g$.  Put $u(z):=Ld(z,S)$ and let
\[
 R_0:=\int g(z)^2\,dP(z),
 \qquad
 R_1:=\int(g(z)-u(z))^2\,dP(z).
\]
Then
\begin{align}
 \max\{R_0,R_1\}
 &\ge\frac{R_0+R_1}{2}
 \annotnl{By $\max\{x,y\}\ge(x+y)/2$}
 \nonumber\\
 &=
 \frac12\int
 \left[g(z)^2+(g(z)-u(z))^2\right]\,dP(z)
 \annotnl{By the definitions of $R_0,R_1$ and linearity}
 \nonumber\\
 &=
 \frac12\int
 \left[
 2\left(g(z)-\frac{u(z)}2\right)^2
 +\frac{u(z)^2}{2}
 \right]\,dP(z)
 \annotnl{By completing the square}
 \nonumber\\
 &\ge
 \frac14\int u(z)^2\,dP(z)
 \annotnl{By nonnegativity of the square}
 \nonumber\\
 &=
 \frac{L^2}{4}\int d(z,S)^2\,dP(z)
 \annotnl{By the definition of $u(z)$}
 \nonumber\\
 &=
 \frac{L^2}{2}
 \int\min_{s\in S}C(z,s)\,dP(z)
 \annotnl{By $d(z,s)^2=2C(z,s)$}
 \nonumber\\
 &=
 \frac{L^2}{2}Q_P(S)
 \annotnl{By the definition of $Q_P(S)$}
 \nonumber\\
 &\ge
 \frac{L^2}{2}m_P(b).
 \annotref{By padding $S$ to a feasible $b$-tuple}
 \label{eq:lipschitz_lower_app}
\end{align}
The last line is valid because repeating any point of the nonempty
set $S$ until there are exactly $b$ representatives does not change
the closest-representative cost.  The padded tuple is feasible in the
infimum defining $m_P(b)$, so $Q_P(S)\ge m_P(b)$.
Since both $f_0$ and $f_1$ belong to the task class, the supremum over
that class is at least the left-hand side of
\eqref{eq:lipschitz_lower_app}.  This proves the lower bound.

\par\medskip\noindent\textbf{Upper bound.}\quad
Now suppose a selector returns $S$ with
$Q_P(S)\le\kappa_qm_P(b)$.  For every $z$, let $s(z)$ be the
$\prec$-first nearest point in $S$ and define the nearest-neighbor
predictor $\widehat f(z):=f(s(z))$.  For every
$f\in\operatorname{Lip}_L(P)$,
\begin{align}
 |\widehat f(z)-f(z)|^2
 &=|f(s(z))-f(z)|^2
 \annotnl{By the definition of $\widehat f$}
 \nonumber\\
 &\le L^2d(s(z),z)^2
 \annotnl{Since $f$ is $L$-Lipschitz}
 \nonumber\\
 &=2L^2C(z,s(z))
 \annotnl{By $d^2=2C$}
 \nonumber\\
 &=2L^2\min_{s\in S}C(z,s)
 \annotref{Since $s(z)$ is a nearest point}.
 \label{eq:nn_pointwise_app}
\end{align}
Integrating \eqref{eq:nn_pointwise_app} gives
\begin{align}
 \E_P[|\widehat f(z)-f(z)|^2]
 &\le2L^2Q_P(S)
 \annotnl{By the definition of $Q_P(S)$}
 \nonumber\\
 &\le2\kappa_qL^2m_P(b)
 \annotref{By the stated quantization condition}.
 \label{eq:lipschitz_upper_app}
\end{align}
The right-hand side is independent of $f$, so it also bounds the
supremum over $\operatorname{Lip}_L(P)$.

\par\medskip\noindent\textbf{Budget rate and label complexity.}\quad
Under \cassum{assumption:intrinsic_dim}, the proved lower bound and
the attainable upper bound become
\begin{equation}\label{eq:minimax_rate_app}
 \frac{L^2c_1}{2}\,b^{-2/d_{\mathrm{int}}}
 \le
 \mathcal R_b^\star(P)
 \le
 2\kappa_qL^2c_2\,b^{-2/d_{\mathrm{int}}},
\end{equation}
where the upper inequality refers to a selector satisfying the stated
fidelity condition.  If a target risk $\eta>0$ is required, the lower
inequality in \eqref{eq:minimax_rate_app} implies the necessary
condition
\[
 b
 \ge
 \left\lceil
 \left(\frac{L^2c_1}{2\eta}\right)^{d_{\mathrm{int}}/2}
 \right\rceil,
\]
because otherwise $(L^2c_1/2)b^{-2/d_{\mathrm{int}}}>\eta$.
Likewise, the upper inequality is at most $\eta$ whenever
\[
 b
 \ge
 \left\lceil
 \left(\frac{2\kappa_qL^2c_2}{\eta}\right)^{d_{\mathrm{int}}/2}
 \right\rceil.
\]
Both integer thresholds are asserted only within the budget range of
\cassum{assumption:intrinsic_dim}.

\par\medskip\noindent\textbf{Decoder fidelity of $\eps$-AS.}\quad
Fix any $\kappa_q>2(1+\kappa_d)$ and a realization in the
intersection of the events in
\eqref{eq:uniform_empirical_accuracy},
\eqref{eq:clustering_accuracy}, and
\eqref{eq:localized_decoding}.  The first two events give
\begin{align*}
 R_P(\widehat{\boldsymbol\mu})
 &\le R_n(\widehat{\boldsymbol\mu})+r_{n,b}
 \annotnl{By uniform empirical approximation}
 \nonumber\\
 &\le \inf_{\boldsymbol\mu}R_n(\boldsymbol\mu)
       +\xi_{n,b}+r_{n,b}
 \annotnl{By approximate clustering}
 \nonumber\\
 &\le m_P(b)+2r_{n,b}+\xi_{n,b}.
 \annotref{At a population minimizer}
\end{align*}
They also give
\[
 \widehat m(b)
 =R_n(\widehat{\boldsymbol\mu})
 \ge\inf_{\boldsymbol\mu}R_n(\boldsymbol\mu)
 \ge m_P(b)-r_{n,b}>0
\]
for all sufficiently large $n$.  Applying
\clem{lem:anchor_query_transfer} to
$\widehat{\boldsymbol\mu}$ and
$s_1(t),\ldots,s_b(t)$, then using
\eqref{eq:localized_decoding}, gives uniformly over
$t\in[t_-,t_+]$
\begin{align*}
 Q_P\!\left(\widehat S(t\widehat m(b))\right)
 &\le
 2R_P(\widehat{\boldsymbol\mu})
 +2\sum_{k=1}^b
 p_k(\widehat{\boldsymbol\mu})
 C(\widehat\mu_k,s_k(t))
 \nonumber\\
 &\le
 2(1+\kappa_d)m_P(b)+4r_{n,b}+2\xi_{n,b}.
\end{align*}
Since $r_{n,b}/m_P(b)\to0$ and
$\xi_{n,b}/m_P(b)\to0$, the last expression is at most
$\kappa_qm_P(b)$ for all sufficiently large $n$.  The bound is
uniform over $t\in[t_-,t_+]$.  The intersection of the three events
has probability tending to one, so
\begin{equation}\label{eq:decoded_fidelity_derived}
 \Pr\!\left[
 \widehat m(b)>0
 \quad\text{and}\quad
 \sup_{t\in[t_-,t_+]}
 \frac{
 Q_P\!\left(\widehat S(t\,\widehat m(b))\right)
 }{m_P(b)}
 \le\kappa_q
 \right]\longrightarrow1.
\end{equation}
On this event,
\eqref{eq:lipschitz_upper_app} applies conditionally to the realized
selection and yields conditional worst-case risk
$O_p(m_P(b))$.  This establishes the stated high-probability rate and
completes the proof.
\end{proof}

\subsection{Proof of \mthmp{thm:eps_star}}
\label{app:proof_eps_star}

We separate three statements that play different logical roles.
First, a finite-pool axiomatization shows that the empirical mean
closest-anchor cost is the unique admissible scale statistic.
Second, exact scale equivariance shows that the multiplier of that
statistic must be dimensionless.  Third, an exactly solvable
two-route robust game gives a unique interior multiplier.

\begin{definition}[Admissible closest-anchor scale rule]
\label{def:scale_rule_app}
For every $n\ge1$, let $T_n:\R_+^n\to\R_+$ map closest-anchor costs
$(d_1,\ldots,d_n)$ to a transport temperature.  The rule is
admissible if:
\begin{enumerate}[label=(\roman*),leftmargin=2.1em]
\item it is invariant under permutations of the $n$ costs;
\item $T_n(sd_1,\ldots,sd_n)=sT_n(d_1,\ldots,d_n)$ for every $s>0$;
\item for every $n,q\ge1$, $d\in\R_+^n$, and $e\in\R_+^q$,
\[
 T_{n+q}(d\oplus e)
 =
 \frac{nT_n(d)+qT_q(e)}{n+q};
\]
\item $T_1(1)=c$ for a fixed $c>0$.
\end{enumerate}
Condition (iii) applies to the concatenation of two closest-cost
lists evaluated under a common, fixed anchor and cost system.
\end{definition}

\begin{lemma}[Unique empirical cost scale]
\label{lem:scale_unique_app}
Every admissible closest-anchor scale rule has the unique form
\[
 T_n(d_1,\ldots,d_n)
 =
 \frac c n\sum_{i=1}^nd_i.
\]
In particular, for $d_i=\min_kC(z_i,\widehat\mu_k)$, it equals
$c\,\widehat m(b)$.
\end{lemma}

\begin{proof}
Positive homogeneity and normalization first give, for every $d_i>0$,
\begin{align}
 T_1(d_i)
 &=T_1(d_i\cdot1)
 \annotnl{By writing the scalar input as $d_i\cdot1$}
 \nonumber\\
 &=d_iT_1(1)
 \annotnl{By positive homogeneity}
 \nonumber\\
 &=cd_i
 \annotref{By $T_1(1)=c$}.
 \label{eq:singleton_scale_app}
\end{align}
For $d_i=0$, positive homogeneity gives
$T_1(0)=T_1(2\cdot0)=2T_1(0)$.  Subtracting $T_1(0)$ from both sides
forces $T_1(0)=0$.  Thus \eqref{eq:singleton_scale_app} holds for all
$d_i\ge0$.

We now use induction on $n$.  The assertion for $n=1$ is
\eqref{eq:singleton_scale_app}.  Suppose it holds for $n-1$.  Split
$(d_1,\ldots,d_n)$ into a pool of size $n-1$ and a singleton.
Mixture consistency gives
\begin{align}
 T_n(d_1,\ldots,d_n)
 &=
 \frac{(n-1)T_{n-1}(d_1,\ldots,d_{n-1})+T_1(d_n)}{n}
 \annotnl{By Definition~\ref{def:scale_rule_app}(iii)}
 \nonumber\\
 &=
 \frac{(n-1)\left[\frac c{n-1}
       \sum_{i=1}^{n-1}d_i\right]+cd_n}{n}
 \annotnl{By the induction hypothesis and Equation~\ref{eq:singleton_scale_app}}
 \nonumber\\
 &=
 \frac c n\sum_{i=1}^nd_i
 \annotref{By cancellation and regrouping}.
 \label{eq:unique_mean_app}
\end{align}
This proves the formula for every $n$.  Permutation invariance
guarantees that the value is independent of the order in which the
pool is decomposed.  Substituting the empirical closest-anchor costs
into \eqref{eq:unique_mean_app} gives $T_n=c\widehat m(b)$.
\end{proof}

\begin{restated}{Theorem}{\mref{thm:eps_star}}{Unique scale-equivariant and robustly calibrated temperature}
The rule
\[
 \eps^\star=c\,\widehat m(b)
\]
is the unique closest-anchor temperature, up to the dimensionless
constant $c$, satisfying the axioms in
Definition~\ref{def:scale_rule_app}.  Whenever
$\widehat m(b)>0$, its positive-temperature EOT plan and decoded set
are exactly equivariant to a positive rescaling of the transport
cost.  Under \cassum{assumption:feature_geometry}, $m_P(b)>0$, and the
canonical two-route robust calibration game defined below assigns
every fixed $A,\lambda,\kappa>0$ a unique
$c_{\mathrm{can}}(A,\lambda,\kappa)\in(0,\infty)$.  Its unique global
minimizer is
$\eps^\star_{\mathrm{can}}=c_{\mathrm{can}}m_P(b)$.  For any fixed
calibrated $c$, under Assumptions~\ref{assumption:feature_geometry}
and~\ref{assumption:sampling}--\ref{assumption:approximate_clustering},
$\eps^\star/m_P(b)\to c$ in probability.  Under the same conditions,
if \cassum{assumption:intrinsic_dim} also holds, then
$\eps^\star=\Theta_p(b^{-2/d_{\mathrm{int}}})$.  If, in addition,
\cassum{assumption:localized_decoding} holds and $c$ lies in its
normalized-temperature window, the conditional worst-case risk of
the resulting $\eps$-AS selector paired with nearest-neighbor readout
attains, in probability, the minimax rate of
\mthm{thm:lowerbound}.
\end{restated}

\begin{proof}
The uniqueness of the empirical statistic was proved in
\clem{lem:scale_unique_app}.  We prove the remaining claims in four
steps.

\noindent\textbf{Step 1: exact scale equivariance of balanced EOT.}\quad
For a cost matrix $C$, write
\[
 F_{C,\eps}(\pi):=
 \langle\pi,C\rangle
 +\eps\KL(\pi\|a\otimes u).
\]
For any $s>0$, $\eps>0$, and every feasible $\pi$,
\begin{align*}
 F_{sC,s\eps}(\pi)
 &=
 \langle\pi,sC\rangle
 +s\eps\KL(\pi\|a\otimes u)
 \annotnl{By the definition of $F$ at $(sC,s\eps)$}
 \nonumber\\
 &=
 s\langle\pi,C\rangle
 +s\eps\KL(\pi\|a\otimes u)
 \annotnl{By linearity of the matrix inner product}
 \nonumber\\
 &=
 sF_{C,\eps}(\pi)
 \annotnl{By factoring out $s>0$}.
\end{align*}
Multiplication of all objective values by the same positive number
does not change their ordering.  Since the entropic minimizer is
unique by \clem{lem:anchorspacing}(i),
\begin{equation}\label{eq:plan_scale_equivariance_app}
 \pi_{s\eps}(sC)=\pi_\eps(C).
\end{equation}
The round-robin decoder sees the same plan on both sides and therefore
returns the same subset under the fixed tie order.

If $d_i(C):=\min_kC_{ik}$, then
\begin{align}
 d_i(sC)
 &=\min_k(sC_{ik})
 \annotnl{By the definition of the closest-anchor cost}
 \nonumber\\
 &=s\min_kC_{ik}
 \annotref{Since $s>0$ preserves order}
 \nonumber\\
 &=sd_i(C).
 \label{eq:closest_scale_app}
\end{align}
Averaging \eqref{eq:closest_scale_app} gives
$\widehat m_{sC}(b)=s\widehat m_C(b)$, hence
$c\widehat m_{sC}(b)=s\,c\widehat m_C(b)$.  Together with
\eqref{eq:plan_scale_equivariance_app}, this proves exact
scale equivariance of both the temperature and the selected subset.

\par\medskip\noindent\textbf{Step 2: exact solution of one two-route block.}\quad
Take uniform $2\times2$ marginals.  Every feasible coupling has the
form
\[
 \pi(p):=
 \frac12
 \begin{pmatrix}
 p&1-p\\
 1-p&p
 \end{pmatrix},
 \qquad p\in[0,1].
\]
Indeed, the first row and first column each sum to $1/2$, so after
choosing $\pi_{11}=p/2$, the other three entries are forced to be
$(1-p)/2,(1-p)/2,p/2$.

For
\[
 C_\Delta^+:=
 \begin{pmatrix}
 0&\Delta\\
 \Delta&0
 \end{pmatrix},
 \qquad \Delta>0,
\]
the cost term is
\begin{align}
 \langle\pi(p),C_\Delta^+\rangle
 &=
 \frac p2\cdot0
 +\frac{1-p}{2}\Delta
 +\frac{1-p}{2}\Delta
 +\frac p2\cdot0
 \annotnl{By the entrywise matrix inner product}
 \nonumber\\
 &=\Delta(1-p).
 \annotref{By combining the off-diagonal terms}
 \label{eq:block_cost_app}
\end{align}
The independent coupling has four entries $1/4$, and therefore
\begin{align}
 \KL(\pi(p)\|a\otimes u)
 &=
 2\frac p2\log\frac{p/2}{1/4}
 +2\frac{1-p}{2}\log\frac{(1-p)/2}{1/4}
 \annotnl{By the two diagonal and two off-diagonal entries}
 \nonumber\\
 &=
 p\log(2p)+(1-p)\log(2(1-p)).
 \annotref{By cancellation and simplification}
 \label{eq:block_kl_app}
\end{align}
By \eqref{eq:block_cost_app}--\eqref{eq:block_kl_app}, the scalar
objective is
\[
 F_{\Delta,\eps}(p)
 =
 \Delta(1-p)
 +\eps\left[
 p\log(2p)+(1-p)\log(2(1-p))
 \right].
\]
For $p\in(0,1)$,
\begin{align}
 F'_{\Delta,\eps}(p)
 &=
 -\Delta
 +\eps\left[
 \log(2p)+1-\log(2(1-p))-1
 \right]
 \annotnl{By direct differentiation}
 \nonumber\\
 &=
 -\Delta+\eps\log\frac p{1-p},
 \annotref{By cancellation and the logarithm rule}
 \label{eq:block_first_derivative_app}\\
 F''_{\Delta,\eps}(p)
 &=
 \eps\left(\frac1p+\frac1{1-p}\right)
 \annotnl{By direct differentiation}
 \nonumber\\
 &=
 \frac{\eps}{p(1-p)}
 >0.
 \annotnl{Since $p\in(0,1)$ and $\eps>0$}
 \nonumber
\end{align}
Thus $F'_{\Delta,\eps}$ is strictly increasing on $(0,1)$.  We now
solve for its stationary point:
\begin{align}
 F'_{\Delta,\eps}(p)=0
 &\Longleftrightarrow
 \log\frac p{1-p}=\frac\Delta\eps
 \annotnl{By Equation~\ref{eq:block_first_derivative_app}}
 \nonumber\\
 &\Longleftrightarrow
 \frac p{1-p}=e^{\Delta/\eps}
 \annotnl{By monotonicity of the exponential}
 \nonumber\\
 &\Longleftrightarrow
 p=e^{\Delta/\eps}(1-p)
 \annotnl{By multiplying by $1-p>0$}
 \nonumber\\
 &\Longleftrightarrow
 p(1+e^{\Delta/\eps})=e^{\Delta/\eps}
 \annotnl{By rearranging}
 \nonumber\\
 &\Longleftrightarrow
 p=
 \frac{e^{\Delta/\eps}}{1+e^{\Delta/\eps}}
 =
 \frac1{1+e^{-\Delta/\eps}}.
 \annotref{By dividing by the positive denominator}
 \label{eq:block_logistic_app}
\end{align}
Denote the displayed solution by $p_\eps^\star$.  Strict increase of
$F'_{\Delta,\eps}$ and
$F'_{\Delta,\eps}(p_\eps^\star)=0$ imply
\[
 F'_{\Delta,\eps}(p)<0
 \quad\text{for }0<p<p_\eps^\star,
 \qquad
 F'_{\Delta,\eps}(p)>0
 \quad\text{for }p_\eps^\star<p<1.
\]
Therefore $F_{\Delta,\eps}$ strictly decreases up to
$p_\eps^\star$ and strictly increases after it.  Continuity at
$p=0,1$ under the convention $0\log0=0$ then shows that
$p_\eps^\star$ is the unique global minimizer on $[0,1]$.

\par\medskip\noindent\textbf{Step 3: the canonical robust calibration game has one
global minimizer.}\quad
Set $m:=m_P(b)>0$.  The first route measures oversmoothing with
$\Delta=m$.  Relative to the zero-temperature diagonal route, its
transport-cost bias is, by
\eqref{eq:block_cost_app} and \eqref{eq:block_logistic_app},
\begin{align}
 B_m(\eps)
 &:=
 m(1-p_\eps)
 \nonumber\\
 &=
 m\left[
 1-\frac{e^{m/\eps}}{1+e^{m/\eps}}
 \right]
 \annotnl{By Equation~\ref{eq:block_logistic_app} with $\Delta=m$}
 \nonumber\\
 &=
 \frac{m}{1+e^{m/\eps}}.
 \annotref{By using a common denominator}
 \label{eq:oversmoothing_app}
\end{align}
Its derivative is
\begin{align}
 B_m'(\eps)
 &=
 m\left[
 -\frac{e^{m/\eps}}{(1+e^{m/\eps})^2}
 \right]
 \left(-\frac m{\eps^2}\right)
 \annotnl{By the chain rule}
 \nonumber\\
 &=
 \frac{m^2e^{m/\eps}}
 {\eps^2(1+e^{m/\eps})^2}
 >0.
 \annotref{Since all factors are positive}
 \label{eq:bias_increasing_app}
\end{align}
Moreover,
\[
 \lim_{\eps\downarrow0}B_m(\eps)=0,
 \qquad
 \lim_{\eps\to\infty}B_m(\eps)=\frac m2,
\]
because $e^{m/\eps}\to\infty$ in the first limit and
$e^{m/\eps}\to1$ in the second.

The second route measures sensitivity to a near tie.  Let
$r:=\kappa m$ for a dimensionless $\kappa>0$ and define
\[
 C_r^+:=
 \begin{pmatrix}0&r\\r&0\end{pmatrix},
 \qquad
 C_r^-:=
 \begin{pmatrix}r&0\\0&r\end{pmatrix}.
\]
These are abstract nonnegative transport-cost blocks.  The robust
game is valid for general transport costs. If one additionally
requires every entry to be realizable as a unit-sphere cosine cost,
then one restricts to $r=\kappa m\le2$.
Their entrywise distance is
$\|C_r^+-C_r^-\|_\infty=r=\kappa m$.  Applying
\eqref{eq:block_logistic_app} to $C_r^+$ gives
\[
 p_+=\frac1{1+e^{-r/\eps}}.
\]
For $C_r^-$, direct substitution gives
\begin{align*}
 \langle\pi(p),C_r^-\rangle
 &=rp
 \annotnl{Since only the diagonal entries have cost $r$}
 \nonumber\\
 &=r\left[1-(1-p)\right].
 \annotnl{By direct algebra}
\end{align*}
The KL expression in \eqref{eq:block_kl_app} is unchanged when
$p$ is replaced by $1-p$.  Therefore the full $C_r^-$ objective at
$p$ equals the $C_r^+$ objective at $1-p$.  Since the latter has the
unique minimizer $p_+$,
\begin{align*}
 p_-
 &=1-p_+
 \annotnl{By the substitution $q=1-p$}
 \nonumber\\
 &=1-\frac1{1+e^{-r/\eps}}
 \annotnl{By substituting the formula for $p_+$}
 \nonumber\\
 &=\frac{e^{-r/\eps}}{1+e^{-r/\eps}}
 \annotnl{By using a common denominator}
 \nonumber\\
 &=\frac1{1+e^{r/\eps}}.
 \annotnl{By multiplying by $e^{r/\eps}/e^{r/\eps}$}
\end{align*}
For any $p,q\in[0,1]$, direct subtraction of the four entries gives
\begin{align}
 \|\pi(p)-\pi(q)\|_1
 &=
 4\frac{|p-q|}{2}
 \annotref{Since each entrywise difference equals $|p-q|/2$}
 \nonumber\\
 &=2|p-q|.
 \label{eq:block_l1_app}
\end{align}
Therefore,
\begin{align*}
 \frac12
 \|\pi_\eps(C_r^+)-\pi_\eps(C_r^-)\|_1
 &=
 |p_+-p_-|
 \annotnl{By Equation~\ref{eq:block_l1_app}}
 \nonumber\\
 &=
 2p_+-1
 \annotnl{By $p_-=1-p_+$ and $p_+>1/2$}
 \nonumber\\
 &=
 \frac{1-e^{-r/\eps}}{1+e^{-r/\eps}}
 \annotnl{By substituting $p_+=1/(1+e^{-r/\eps})$}
 \nonumber\\
 &=
 \tanh\!\left(\frac r{2\eps}\right)
 \annotnl{By the identity for $\tanh(x/2)$}
 \nonumber\\
 &=
 \tanh\!\left(\frac{\kappa m}{2\eps}\right).
 \annotnl{By $r=\kappa m$}
 \nonumber
\end{align*}
If rerouted local mass incurs a normalized penalty $\lambda m$, the
exact instability loss is
\begin{equation}\label{eq:instability_app}
 S_m(\eps):=
 \lambda m
 \tanh\!\left(\frac{\kappa m}{2\eps}\right),
 \qquad \lambda>0.
\end{equation}
Differentiation gives
\begin{align}
 S_m'(\eps)
 &=
 \lambda m\,
 \operatorname{sech}^2\!\left(\frac{\kappa m}{2\eps}\right)
 \left(-\frac{\kappa m}{2\eps^2}\right)
 \annotnl{By the chain rule}
 \nonumber\\
 &<0,
 \annotref{By the signs of the three factors}
 \label{eq:instability_decreasing_app}
\end{align}
and
\[
 \lim_{\eps\downarrow0}S_m(\eps)=\lambda m,
 \qquad
 \lim_{\eps\to\infty}S_m(\eps)=0,
\]
because $\tanh x\to1$ as $x\to\infty$ and
$\tanh x\to0$ as $x\downarrow0$.

For fixed $A,\lambda,\kappa>0$, define the canonical robust objective
\[
 J_m(\eps):=
 \max\{A B_m(\eps),S_m(\eps)\}.
\]
With $t:=\eps/m>0$, equations
\eqref{eq:oversmoothing_app} and \eqref{eq:instability_app} give
\begin{align}
 J_m(tm)
 &=
 m\,
 \max\left\{
 \frac{A}{1+e^{1/t}},
 \lambda\tanh\!\left(\frac{\kappa}{2t}\right)
 \right\}
 \annotref{By substituting $\eps=tm$}
 \nonumber\\
 &=:m\,\psi(t).
 \label{eq:normalized_game_app}
\end{align}
Put
\[
 f(t):=\frac{A}{1+e^{1/t}},
 \qquad
 g(t):=\lambda\tanh\!\left(\frac{\kappa}{2t}\right).
\]
By \eqref{eq:bias_increasing_app}, $f$ is continuous and strictly
increasing; by \eqref{eq:instability_decreasing_app}, $g$ is
continuous and strictly decreasing.  Their endpoint limits are
\[
 \lim_{t\downarrow0}f(t)=0,
 \quad
 \lim_{t\to\infty}f(t)=\frac A2,
 \qquad
 \lim_{t\downarrow0}g(t)=\lambda,
 \quad
 \lim_{t\to\infty}g(t)=0.
\]
Consequently, $h(t):=f(t)-g(t)$ is continuous and strictly
increasing, with
\[
 \lim_{t\downarrow0}h(t)=-\lambda<0,
 \qquad
 \lim_{t\to\infty}h(t)=\frac A2>0.
\]
The intermediate value theorem gives a $c_{\mathrm{can}}>0$ with
$h(c_{\mathrm{can}})=0$, and strict monotonicity makes it unique.
Equivalently, $c_{\mathrm{can}}$ is the unique
solution of
\begin{equation*}
 \frac{A}{1+e^{1/c_{\mathrm{can}}}}
 =
 \lambda\tanh\!\left(\frac{\kappa}{2c_{\mathrm{can}}}\right).
\end{equation*}
If $t<c_{\mathrm{can}}$, then $f(t)<g(t)$, so
$\psi(t)=g(t)$ and $\psi$ strictly decreases.  If
$t>c_{\mathrm{can}}$, then
$f(t)>g(t)$, so $\psi(t)=f(t)$ and $\psi$ strictly increases.
Therefore $c_{\mathrm{can}}$ is the unique global minimizer of
$\psi$.  Since
$m>0$, \eqref{eq:normalized_game_app} shows that
\[
 \eps^\star_{\mathrm{can}}=c_{\mathrm{can}}m
\]
is the unique global minimizer of $J_m$.  The constants
$A,\lambda,\kappa$, and hence $c_{\mathrm{can}}$, are dimensionless.
$m$ carries all cost units and all budget dependence.  The game
therefore proves a unique normalized optimum conditional on its
dimensionless weights. The MNIST calibration selects the empirical
global multiplier used for transfer in the experiments.

\par\medskip\noindent\textbf{Step 4: empirical consistency, intrinsic rate, and
minimax matching.}\quad
Consider the intersection of the events in
\eqref{eq:uniform_empirical_accuracy} and
\eqref{eq:clustering_accuracy}.  Since
$m_P(b)=\inf_{\boldsymbol\mu}R_P(\boldsymbol\mu)$, uniform deviation
gives
\[
 m_P(b)-r_{n,b}
 \le
 \inf_{\boldsymbol\mu}R_n(\boldsymbol\mu)
 \le
 m_P(b)+r_{n,b}.
\]
The empirical optimization condition then yields
\[
 m_P(b)-r_{n,b}
 \le
 \widehat m(b)
 \le
 m_P(b)+r_{n,b}+\xi_{n,b}.
\]
Therefore, with probability tending to one,
\begin{equation}\label{eq:empirical_scale_bound_app}
 |\widehat m(b)-m_P(b)|
 \le r_{n,b}+\xi_{n,b}.
\end{equation}
On this event,
\begin{align}
 \left|
 \frac{\eps^\star}{m_P(b)}-c
 \right|
 &=
 \left|
 \frac{c\widehat m(b)}{m_P(b)}-c
 \right|
 \annotnl{By the definition of $\eps^\star$}
 \nonumber\\
 &=
 c\left|
 \frac{\widehat m(b)-m_P(b)}{m_P(b)}
 \right|
 \annotnl{By factoring out $c$}
 \nonumber\\
 &\le
 c\,\frac{r_{n,b}+\xi_{n,b}}{m_P(b)}
 \annotnl{By Equation~\ref{eq:empirical_scale_bound_app}}
 \nonumber\\
 &\longrightarrow0.
 \annotref{By Assumptions~\ref{assumption:uniform_accuracy} and
 \ref{assumption:approximate_clustering}}
 \label{eq:eps_consistency_app}
\end{align}
This proves $\eps^\star/m_P(b)\to c$ in probability.

Under \cassum{assumption:intrinsic_dim},
\begin{align*}
 \eps^\star
 &=c\,m_P(b)\,[1+o_p(1)]
 \annotnl{By Equation~\ref{eq:eps_consistency_app}}
 \nonumber\\
 &=\Theta_p\!\left(b^{-2/d_{\mathrm{int}}}\right)
 \annotnl{By Equation~\ref{eq:m_scaling} and $c>0$}.
\end{align*}
Finally, suppose $c\in[t_-,t_+]$ and fix any
$\kappa_q>2(1+\kappa_d)$.
Equation~\eqref{eq:decoded_fidelity_derived} gives, with probability
tending to one,
\[
 Q_P\!\left(\widehat S(c\,\widehat m(b))\right)
 \le\kappa_qm_P(b).
\]
Because $\eps^\star=c\,\widehat m(b)$, this is precisely
$Q_P(\widehat S(\eps^\star))\le\kappa_qm_P(b)$.  The upper bound
proved in Theorem~\mref{thm:lowerbound} applies.  Its lower bound
holds for every deterministic task-agnostic selector.  Hence, with
probability tending to one, $\eps$-AS at $\eps^\star$, paired with
nearest-neighbor readout, attains the common
$\Theta(m_P(b))=\Theta(b^{-2/d_{\mathrm{int}}})$ minimax rate.
This proves the theorem.
\end{proof}

\paragraph{Interpretation of the temperature derivation.}
The derivation separates the form of the scale statistic from the
value of its global multiplier. The finite-pool axioms identify the
functional dependence as $c\,\widehat m(b)$ within the admissible
closest-anchor rule class. Exact cost-scale equivariance makes $c$
dimensionless and preserves both the EOT plan and the decoded subset
when the cost units change. For fixed dimensionless weights
$A,\lambda,\kappa$, the canonical game supplies the unique normalized
interior optimum $c_{\mathrm{can}}(A,\lambda,\kappa)$. These weights
encode the relative importance of oversmoothing and near-tie
sensitivity. The experimental protocol instantiates the remaining
dimensionless constant as $c=0.4$ through one global MNIST
calibration and applies it across datasets. Pool-specific and
budget-specific scaling enters through $\widehat m(b)$. Finally,
empirical objective accuracy links $\widehat m(b)$ to $m_P(b)$, and
localized decoding carries this scale to the selected subset and its
conditional minimax rate.

\clearpage
\section{Experimental Supplementary Details}
\label{app:exp_supp}

This section provides the complete experimental protocol, per-dataset
results, numerical diagnostics, runtime measurements, and additional
ablations.

\subsection{Datasets, Backbone, and Baseline Implementations}
\label{app:assets}

To support reproducibility we list every external asset used in this
study. Table~\ref{tab:assets_data} gives each evaluation dataset with its
class count, license, and cited source, and Table~\ref{tab:assets_code}
gives the frozen backbone together with the reference implementation
adopted for each baseline. The assets are available from the cited
sources under their respective access and usage terms. Licenses are
reported to the best of our knowledge, and those sources give the
authoritative terms. Pool and test sizes and the per-dataset budget grids
are listed in the main paper.

\begin{table}[H]
\small\centering
\setlength{\tabcolsep}{5pt}
\caption{Evaluation datasets, licenses, and cited sources.}
\label{tab:assets_data}
\begin{tabular}{@{}llll@{}}
\toprule
\rowhead
Dataset & Classes & License & Source \\
\midrule
MNIST       & 10       & CC BY-SA 3.0    & \url{http://yann.lecun.com/exdb/mnist/} \\
CIFAR-10    & 10       & Research use    & \url{https://www.cs.toronto.edu/~kriz/cifar.html} \\
CIFAR-100   & 100      & Research use    & \url{https://www.cs.toronto.edu/~kriz/cifar.html} \\
STL-10      & 10       & Research use    & \url{https://cs.stanford.edu/~acoates/stl10/} \\
Caltech-101 & 101      & CC BY 4.0       & \url{https://data.caltech.edu/records/mzrjq-6wc02} \\
ImageNet-1k & 1{,}000  & ImageNet terms  & \url{https://www.image-net.org/} \\
\bottomrule
\end{tabular}
\end{table}

\begin{table}[H]
\small\centering
\setlength{\tabcolsep}{5pt}
\caption{Frozen backbone and baseline reference implementations.}
\label{tab:assets_code}
\resizebox{\textwidth}{!}{
\begin{tabular}{@{}lll@{}}
\toprule
\rowhead
Component & Reference & Repository / Source \\
\midrule
DINOv2-G/14 (backbone) & \citet{oquab2024dinov2}     & \url{https://github.com/facebookresearch/dinov2} \\
Coreset / $k$-Center   & \citet{sener2018coreset}    & \url{https://github.com/ozansener/active_learning_coreset} \\
FPS                    & \citet{jin2022oneshot}      & \url{https://github.com/fdujay/One-shot-AL} \\
BADGE                  & \citet{ash2020badge}        & \url{https://github.com/JordanAsh/badge} \\
ProbCover              & \citet{yehuda2022probcover} & \url{https://github.com/avihu111/TypiClust} \\
DCoM                   & \citet{mishal2024dcom}      & follows the paper (cold-start branch) \\
TypiClust              & \citet{hacohen2022typiclust}& \url{https://github.com/avihu111/TypiClust} \\
USL                    & \citet{wang2022usl}         & \url{https://github.com/TonyLianLong/UnsupervisedSelectiveLabeling} \\
ActiveFT               & \citet{xie2023activeft}     & \url{https://github.com/yichen928/ActiveFT} \\
$\eps$-AS (ours)       & this paper                   & released upon acceptance \\
\bottomrule
\end{tabular}
}
\end{table}

\subsection{Feature Extraction and Preprocessing}
\label{app:preprocessing}

All methods operate on a shared frozen feature cache, which we build
once per dataset. Each image is resized to $224\times 224$ by bilinear
upsampling on the small-resolution datasets (MNIST, CIFAR-10,
CIFAR-100, and STL-10), and by resizing the shorter side to $256$
followed by a $224$ center crop on Caltech-101 and ImageNet-1k. We
replicate the single MNIST channel three times before applying the
same preprocessing as for the RGB datasets. We apply the standard
ImageNet channel normalization throughout. Each
image is then passed through a frozen DINOv2-G/14
backbone~\citep{oquab2024dinov2}, and we take the $\ell_2$-normalized
class token as the $D=1536$ dimensional feature. We reuse the same
features for every method, budget, and seed so that all methods receive
the same input representation.

\subsection{Baseline Configurations}
\label{app:baselines}

We document below the exact configuration of every baseline. All
methods read the same $\ell_2$-normalized DINOv2-G/14 feature cache
and use the same selection seeds.

\noindent\textbf{Random.}\;We sample $b$ indices uniformly at random
without replacement.

\noindent\textbf{Coreset.}\;We
follow the cold-start variant of \citet{sener2018coreset} and apply
farthest-first traversal under cosine distance, initialized from a
random pool point.

\noindent\textbf{FPS.}\;Farthest-Point Sampling~\citep{jin2022oneshot}
uses the same traversal initialized from the medoid of the pool.

\noindent\textbf{BADGE-cold.}\;The cold-start specialization of
\citet{ash2020badge} reduces to running $k$-means$++$ on the frozen
features with $K=b$ centers, returning the $b$ chosen seeds. We use
the standard $D^2$ sampling probability without re-projection.

\noindent\textbf{ProbCover.}\;We follow \citet{yehuda2022probcover}
and build a $\delta$-ball graph in cosine distance. The threshold
$\delta$ is selected by binary search to satisfy a purity criterion
of $0.95$ on $K^\star$-means pseudo-labels, as in the original
implementation. Here $K^\star$ is set to the known number of
benchmark classes. Thus the procedure uses class-count metadata but
no sample labels. We report this class-count prior explicitly as part
of the baseline configuration.

\noindent\textbf{DCoM.}\;We follow the cold-start branch of
\citet{mishal2024dcom}, retaining its own acquisition and
implementation path. With an empty labeled set, its competence term
is inactive and the remaining score is coverage-driven. We use the
per-dataset $\delta_{\ell_2}$ value from Table~6 of that reference
when one is available and otherwise use the same purity-based
automatic radius rule described above. DCoM and ProbCover are timed
using their respective implementations, preserving their
method-specific memory and runtime behavior.

\noindent\textbf{TypiClust.}\;We follow
\citet{hacohen2022typiclust}. We run $k$-means with $K=b$
clusters on the features, using $k$-means$++$ initialization and
twenty-five Lloyd iterations. Inside each cluster we select the
sample with the highest typicality score, which is the inverse of
the mean distance to its twenty nearest neighbors in the cluster.

\noindent\textbf{USL.}\;We follow the reference implementation of
\citet{wang2022usl}. We run $k$-means with $K=b$
clusters, compute a
kNN-density score with $K_{\mathrm{NN}}=20$ neighbors under the
squared $\ell_2$ distance, and then run ten iterations of the
selection regularizer with weight $W=0.5$, momentum $0.9$,
horizon $4.0$, and $\alpha=0.5$, matching the official CIFAR
configuration.

\noindent\textbf{ActiveFT.}\;We follow \citet{xie2023activeft}. We
use temperature $\tau=0.07$, diversity weight $\lambda=1$, AdamW
with learning rate $10^{-3}$, and three hundred inner gradient steps
from a random anchor initialization.

\noindent\textbf{$\eps$-AS (ours).}\;The closed form
\meqref{eq:eps_star} with $c=0.4$, $k$-means with $K=b$ centroids,
log-domain Sinkhorn with at most 200 iterations, tolerance
$10^{-6}$ on the dual potentials, and an entropic floor
$\eps_{\min}=10^{-4}$ for numerical stability. We use the
round-robin rule \meqref{eq:roundrobin} to extract the discrete
selection. We treat MNIST as the sole calibration domain: $c=0.4$
is chosen once from the grid
$\{0.1,0.2,\ldots,0.9\}$ by the MNIST calibration score and is then
fixed. The other five datasets evaluate transfer without
target-dataset labels or per-dataset retuning.

\subsection{Linear-Probe Protocol}
\label{app:probe}

We train one linear head per (method, dataset, budget, seed) tuple
on top of the frozen DINOv2-G/14 features. The head is a single
fully connected layer mapping $\R^D$ to the class count, and the
training objective is standard multiclass cross-entropy. The optimizer is AdamW with
weight decay $10^{-4}$. The learning rate is $10^{-3}$ on CIFAR-10,
CIFAR-100, STL-10, Caltech-101, and ImageNet-1k, and
$5\!\cdot\!10^{-3}$ on MNIST when $b\le100$. The batch size is $1024$ on
all datasets except for MNIST under $b\le 100$, where we drop it
to $128$ to keep at least one mini-batch per epoch. We train for
50 epochs on the small datasets, 100 epochs on
ImageNet-1k, and 500 epochs on the small-budget MNIST
cells. A linear warm-up over five epochs is followed by a cosine
schedule down to one hundredth of the peak learning rate. The same
protocol is applied to every method.

\subsection{Full Main-Comparison Results}
\label{app:full_main}

The tables below report mean top-1 accuracy and standard deviation
over three seeds. The Avg column averages the budget-level means, and
the final row compares $\eps$-AS with the strongest baseline. All
$\eps$-AS results use the single calibration constant $c=0.4$.

\begin{table}[H]
\small
\centering
\caption{MNIST main comparison
(top-1 \%, mean$\pm$std over $3$ seeds).}
\label{tab:main_mnist}
\setlength{\tabcolsep}{3pt}
\begin{tabular}{@{}lccccc@{}}
\toprule
\rowhead
Method & $b{=}50$ & $b{=}100$ & $b{=}200$ & $b{=}500$ & Avg \\
\midrule
Random & 19.34$\pm$4.97 & 21.49$\pm$9.58 & 31.60$\pm$3.33 & 39.76$\pm$4.18 & 28.05 \\
Coreset & 21.05$\pm$1.04 & 19.43$\pm$1.48 & 21.80$\pm$2.64 & 18.76$\pm$3.83 & 20.26 \\
FPS & 20.68$\pm$0.59 & 19.79$\pm$1.02 & 20.61$\pm$1.72 & 18.92$\pm$3.79 & 20.00 \\
BADGE & 31.70$\pm$2.26 & 34.39$\pm$2.44 & 44.00$\pm$9.36 & 48.61$\pm$13.76 & 39.68 \\
ProbCover & 10.69$\pm$0.65 & 13.58$\pm$2.25 & 32.54$\pm$1.98 & 33.96$\pm$3.62 & 22.69 \\
DCoM & 21.73$\pm$2.41 & 35.88$\pm$1.74 & 40.55$\pm$2.79 & 41.27$\pm$5.10 & 34.86 \\
TypiClust & 32.36$\pm$3.63 & 38.36$\pm$3.74 & 47.68$\pm$6.61 & 42.49$\pm$4.59 & 40.22 \\
USL & 30.61$\pm$1.13 & 36.18$\pm$1.17 & 38.28$\pm$2.19 & 45.83$\pm$10.76 & 37.73 \\
ActiveFT & 20.91$\pm$1.77 & 30.02$\pm$4.32 & 35.40$\pm$1.13 & 42.46$\pm$4.87 & 32.20 \\
\rowours
\textbf{$\eps$-AS} & \textbf{36.48$\pm$2.03} & \textbf{39.02$\pm$3.08} & \textbf{49.61$\pm$6.05} & \textbf{49.25$\pm$3.96} & \textbf{43.59} \\
\midrule
$\Delta$ vs.\ best & \gain{82}{+4.12} & \gain{32}{+0.66} & \gain{48}{+1.93} & \gain{32}{+0.64} & \gain{65}{+3.37} \\
\bottomrule
\end{tabular}
\end{table}

\begin{table}[H]
\small
\centering
\caption{CIFAR-10 main comparison (top-1 \%, mean$\pm$std).}
\label{tab:main_cifar10}
\setlength{\tabcolsep}{3pt}
\begin{tabular}{@{}lccccc@{}}
\toprule
\rowhead
Method & $b{=}50$ & $b{=}100$ & $b{=}200$ & $b{=}500$ & Avg \\
\midrule
Random & 76.90$\pm$5.74 & 86.92$\pm$2.99 & 93.55$\pm$2.81 & 95.76$\pm$1.52 & 88.28 \\
Coreset & 63.04$\pm$5.54 & 62.37$\pm$11.6 & 54.49$\pm$9.57 & 74.66$\pm$3.86 & 63.64 \\
FPS & 63.25$\pm$5.58 & 63.51$\pm$11.0 & 56.35$\pm$8.82 & 74.83$\pm$3.83 & 64.49 \\
BADGE & 83.82$\pm$4.11 & 91.76$\pm$2.27 & 95.70$\pm$0.96 & 97.59$\pm$0.45 & 92.22 \\
ProbCover & 82.41$\pm$2.73 & 82.86$\pm$3.53 & 89.37$\pm$0.88 & 95.87$\pm$1.49 & 87.63 \\
DCoM & 80.27$\pm$1.09 & 82.79$\pm$1.02 & 88.48$\pm$0.47 & 96.53$\pm$0.24 & 87.02 \\
TypiClust & 87.00$\pm$2.37 & 93.90$\pm$1.02 & 97.07$\pm$0.66 & 98.33$\pm$0.29 & 94.08 \\
USL & 85.50$\pm$2.64 & 94.94$\pm$0.76 & 97.47$\pm$0.53 & 98.30$\pm$0.12 & 94.05 \\
ActiveFT & 91.61$\pm$1.69 & 93.19$\pm$1.34 & 97.13$\pm$0.81 & 97.75$\pm$0.28 & 94.92 \\
\rowours
\textbf{$\eps$-AS} & \textbf{93.20$\pm$1.20} & \textbf{96.83$\pm$0.27} & \textbf{97.91$\pm$0.47} & \textbf{98.39$\pm$0.17} & \textbf{96.58} \\
\midrule
$\Delta$ vs.\ best & \gain{48}{+1.59} & \gain{48}{+1.89} & \gain{20}{+0.44} & \gain{20}{+0.06} & \gain{48}{+1.66} \\
\bottomrule
\end{tabular}
\end{table}

\begin{table}[H]
\small
\centering
\caption{CIFAR-100 main comparison (top-1 \%, mean$\pm$std).}
\label{tab:main_cifar100}
\setlength{\tabcolsep}{3pt}
\begin{tabular}{@{}lccccc@{}}
\toprule
\rowhead
Method & $b{=}100$ & $b{=}250$ & $b{=}500$ & $b{=}1000$ & Avg \\
\midrule
Random & 42.65$\pm$3.08 & 62.22$\pm$3.45 & 75.11$\pm$1.51 & 81.82$\pm$0.39 & 65.45 \\
Coreset & 50.00$\pm$1.93 & 68.19$\pm$1.21 & 74.31$\pm$0.58 & 76.27$\pm$0.94 & 67.19 \\
FPS & 50.43$\pm$1.81 & 68.10$\pm$1.22 & 74.35$\pm$0.56 & 76.14$\pm$0.75 & 67.26 \\
BADGE & 46.52$\pm$2.23 & 69.49$\pm$1.33 & 80.06$\pm$0.63 & 84.49$\pm$0.56 & 70.14 \\
ProbCover & 58.96$\pm$0.43 & 71.71$\pm$0.86 & 76.91$\pm$0.29 & 80.40$\pm$0.39 & 72.00 \\
DCoM & 57.67$\pm$0.12 & 70.77$\pm$0.30 & 76.41$\pm$0.52 & 80.77$\pm$0.48 & 71.41 \\
TypiClust & 61.95$\pm$1.21 & 78.02$\pm$0.96 & 84.28$\pm$0.36 & 86.88$\pm$0.59 & 77.78 \\
USL & 60.92$\pm$0.78 & 79.12$\pm$0.37 & 84.14$\pm$0.39 & 86.43$\pm$0.21 & 77.65 \\
ActiveFT & 67.76$\pm$1.46 & 80.37$\pm$1.18 & 84.91$\pm$0.42 & 86.75$\pm$0.61 & 79.95 \\
\rowours
\textbf{$\eps$-AS} & \textbf{69.77$\pm$1.66} & \textbf{81.88$\pm$0.60} & \textbf{85.33$\pm$0.71} & \textbf{87.53$\pm$0.61} & \textbf{81.13} \\
\midrule
$\Delta$ vs.\ best & \gain{55}{+2.01} & \gain{48}{+1.51} & \gain{20}{+0.42} & \gain{32}{+0.65} & \gain{48}{+1.18} \\
\bottomrule
\end{tabular}
\end{table}

\begin{table}[H]
\small
\centering
\caption{ImageNet-1k main comparison (top-1 \%, mean$\pm$std over
three seeds). ProbCover is omitted because its dense implementation
is infeasible at $n=1.28$M.}
\label{tab:main_imagenet}
\setlength{\tabcolsep}{4pt}
\begin{tabular}{@{}lccccccc@{}}
\toprule
\rowhead
Method & $b{=}250$ & $b{=}500$ & $b{=}1000$ & $b{=}2000$ & $b{=}5000$ & $b{=}10000$ & Avg \\
\midrule
Random & 17.44$\pm$2.41 & 28.84$\pm$3.07 & 43.15$\pm$1.86 & 57.33$\pm$2.53 & 63.51$\pm$1.37 & 70.16$\pm$0.92 & 46.74 \\
Coreset & 18.71$\pm$1.73 & 33.78$\pm$2.58 & 49.05$\pm$1.29 & 58.57$\pm$2.11 & 64.28$\pm$1.04 & 71.04$\pm$0.71 & 49.24 \\
FPS & 18.57$\pm$1.68 & 33.46$\pm$2.63 & 48.56$\pm$1.34 & 58.37$\pm$2.06 & 64.07$\pm$1.09 & 70.83$\pm$0.68 & 48.98 \\
BADGE & 17.63$\pm$2.07 & 31.14$\pm$3.42 & 49.03$\pm$1.95 & 64.30$\pm$2.66 & 68.92$\pm$1.28 & 73.58$\pm$0.84 & 50.77 \\
DCoM & 21.19$\pm$0.94 & 41.88$\pm$1.57 & 62.09$\pm$0.88 & 67.71$\pm$1.19 & 71.83$\pm$0.62 & 75.62$\pm$0.41 & 56.72 \\
TypiClust & 21.76$\pm$1.12 & 39.25$\pm$1.84 & 59.94$\pm$0.97 & 69.71$\pm$1.36 & 72.15$\pm$0.73 & 75.89$\pm$0.48 & 56.45 \\
USL & 21.70$\pm$1.08 & 39.17$\pm$1.79 & 59.92$\pm$0.94 & 68.94$\pm$1.31 & 71.94$\pm$0.69 & 75.71$\pm$0.46 & 56.23 \\
ActiveFT & 21.68$\pm$0.99 & 41.06$\pm$1.66 & 61.39$\pm$0.91 & 70.61$\pm$1.24 & 73.67$\pm$0.66 & 77.24$\pm$0.43 & 57.61 \\
\rowours
\textbf{$\eps$-AS} & \textbf{22.00$\pm$0.63} & \textbf{43.12$\pm$0.78} & \textbf{62.30$\pm$0.52} & \textbf{71.46$\pm$0.74} & \textbf{75.38$\pm$0.41} & \textbf{79.13$\pm$0.27} & \textbf{58.90} \\
\midrule
$\Delta$ vs.\ best & \gain{20}{+0.24} & \gain{48}{+1.24} & \gain{20}{+0.21} & \gain{32}{+0.85} & \gain{60}{+1.71} & \gain{65}{+1.89} & \gain{50}{+1.29} \\
\bottomrule
\end{tabular}
\end{table}

\begin{table}[H]
\small
\centering
\caption{Caltech-101 main comparison (top-1 \%, mean$\pm$std).}
\label{tab:main_caltech}
\setlength{\tabcolsep}{3pt}
\begin{tabular}{@{}lccccc@{}}
\toprule
\rowhead
Method & $b{=}100$ & $b{=}200$ & $b{=}500$ & $b{=}1000$ & Avg \\
\midrule
Random & 62.55$\pm$1.75 & 76.79$\pm$1.97 & 87.34$\pm$0.25 & 90.54$\pm$0.57 & 79.31 \\
Coreset & 80.46$\pm$0.33 & 87.83$\pm$0.27 & 89.18$\pm$0.57 & 90.46$\pm$0.33 & 86.98 \\
FPS & 81.17$\pm$1.34 & 87.78$\pm$0.21 & 89.20$\pm$0.56 & 90.14$\pm$0.66 & 87.07 \\
BADGE & 69.86$\pm$3.19 & 87.06$\pm$0.85 & 90.61$\pm$0.70 & 91.09$\pm$0.47 & 84.66 \\
ProbCover & 73.46$\pm$1.37 & 80.80$\pm$2.01 & 83.78$\pm$0.87 & 81.03$\pm$0.92 & 79.77 \\
DCoM & 74.67$\pm$1.98 & 78.99$\pm$1.57 & 84.95$\pm$1.02 & 89.74$\pm$0.51 & 82.09 \\
TypiClust & 80.89$\pm$1.85 & 89.47$\pm$0.46 & 91.52$\pm$0.05 & 90.98$\pm$0.07 & 88.22 \\
USL & 80.66$\pm$1.96 & 89.58$\pm$0.59 & 91.53$\pm$0.21 & 91.33$\pm$0.18 & 88.28 \\
ActiveFT & 78.19$\pm$1.95 & 88.02$\pm$0.58 & 90.40$\pm$0.64 & 90.91$\pm$0.18 & 86.88 \\
\rowours
\textbf{$\eps$-AS} & \textbf{82.08$\pm$1.04} & \textbf{89.97$\pm$0.89} & \textbf{91.62$\pm$0.24} & \textbf{91.97$\pm$0.66} & \textbf{88.91} \\
\midrule
$\Delta$ vs.\ best & \gain{32}{+0.91} & \gain{20}{+0.39} & \gain{20}{+0.09} & \gain{32}{+0.64} & \gain{32}{+0.63} \\
\bottomrule
\end{tabular}
\end{table}

\begin{table}[H]
\small
\centering
\caption{STL-10 main comparison (top-1 \%, mean$\pm$std).}
\label{tab:main_stl10}
\setlength{\tabcolsep}{3pt}
\begin{tabular}{@{}lccccc@{}}
\toprule
\rowhead
Method & $b{=}50$ & $b{=}100$ & $b{=}200$ & $b{=}500$ & Avg \\
\midrule
Random & 73.60$\pm$2.65 & 82.02$\pm$1.51 & 85.77$\pm$2.78 & 93.28$\pm$1.98 & 83.67 \\
Coreset & 76.93$\pm$1.47 & 85.12$\pm$1.82 & 90.31$\pm$4.55 & 84.95$\pm$0.84 & 84.33 \\
FPS & 78.08$\pm$2.12 & 84.91$\pm$2.11 & 91.25$\pm$3.58 & 84.10$\pm$1.72 & 84.59 \\
BADGE & 82.92$\pm$2.07 & 92.32$\pm$1.32 & 96.89$\pm$0.52 & 97.37$\pm$1.56 & 92.38 \\
ProbCover & 90.09$\pm$1.33 & 95.78$\pm$0.81 & 94.83$\pm$2.34 & 89.16$\pm$3.44 & 92.47 \\
DCoM & 84.23$\pm$1.64 & 91.97$\pm$0.80 & 93.59$\pm$1.58 & 95.30$\pm$1.68 & 91.27 \\
TypiClust & 86.53$\pm$1.09 & 94.84$\pm$0.94 & 97.05$\pm$0.72 & 97.18$\pm$1.21 & 93.90 \\
USL & 86.46$\pm$0.92 & 94.84$\pm$0.99 & 97.74$\pm$0.51 & 97.57$\pm$0.43 & 94.15 \\
ActiveFT & 91.10$\pm$1.62 & 94.51$\pm$0.24 & 97.03$\pm$1.06 & 96.80$\pm$0.13 & 94.86 \\
\rowours
\textbf{$\eps$-AS} & \textbf{91.81$\pm$0.29} & \textbf{96.20$\pm$0.76} & 97.48$\pm$0.88 & \textbf{97.96$\pm$0.98} & \textbf{95.86} \\
\midrule
$\Delta$ vs.\ best & \gain{32}{+0.71} & \gain{20}{+0.42} & \drop{20}{$-$0.26} & \gain{20}{+0.39} & \gain{48}{+1.00} \\
\bottomrule
\end{tabular}
\end{table}

\FloatBarrier
\paragraph{Across-seed variability.}
Across the $22$ transfer cells, $\eps$-AS has the lowest average
standard deviation at $0.67$ percentage points, followed by USL at
$0.81$ and ActiveFT at $0.91$. Variability is highest in the
lowest-budget MNIST settings, so small mean differences should be
interpreted together with the reported dispersion.

\subsection{Intrinsic-Dimension Measurements}
\label{app:intrinsic_dim}

The method uses the empirical closest-anchor cost $\widehat m(b)$
directly and does not estimate intrinsic dimension. For diagnosis,
Figure~\ref{fig:dint_app} reports ordinary least-squares fits of
$\log\widehat m(b)$ against $\log b$ using at least eight log-spaced
budgets from $b=20$, with zero-cost endpoints excluded.

\begin{figure}[H]
\centering
\includegraphics[width=\linewidth]{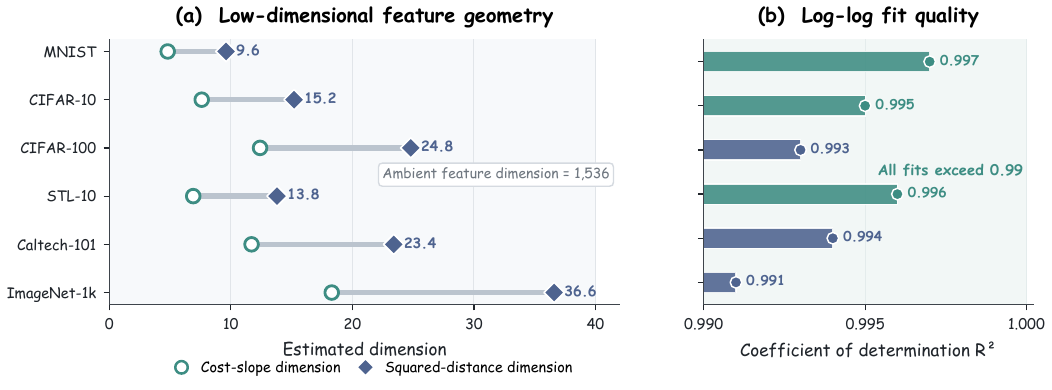}
\caption{Intrinsic-dimension diagnostics from the fitted relation
between $\log\widehat m(b)$ and $\log b$. Panel (a) shows the
cost-slope dimension $d_{\mathrm{cost}}$ and the corresponding
squared-distance dimension
$d_{\mathrm{int}}=2d_{\mathrm{cost}}$. Panel (b) reports the
coefficient of determination for each fit.}
\label{fig:dint_app}
\end{figure}

The fitted cost-slope parameters are between $4.8$ and $18.3$, and
the corresponding squared-distance dimensions are between $9.6$ and
$36.6$. Both are far below the ambient feature dimension $D=1536$.
The fit quality remains high across the full range of observed
dimensions, with $R^2$ between $0.991$ and $0.997$.
The conversion $d_{\mathrm{int}}=2d_{\mathrm{cost}}$ follows from
$C(z,\mu)=\|z-\mu\|_2^2/2$, for which the conventional
$d$-dimensional population quantization law is
$m_P(b)\asymp b^{-2/d}$.

\subsection{Numerical Sinkhorn Schedule}
\label{app:numerics}

We implement Sinkhorn iterations in the log domain with a numerical
floor $\eps_{\min}=10^{-4}$ on the regularizer, an iteration cap
$T_{\mathrm{sk}}=200$, and a tolerance $10^{-6}$ on the maximum
absolute change of the dual potentials between consecutive
iterations. The early-stop criterion fires well before
$T_{\mathrm{sk}}=200$ on every dataset and budget we measured. The
worst-case row-marginal residual at termination is reported in
Figure~\ref{fig:numerics_app}.

\begin{figure}[H]
\centering
\includegraphics[width=\linewidth]{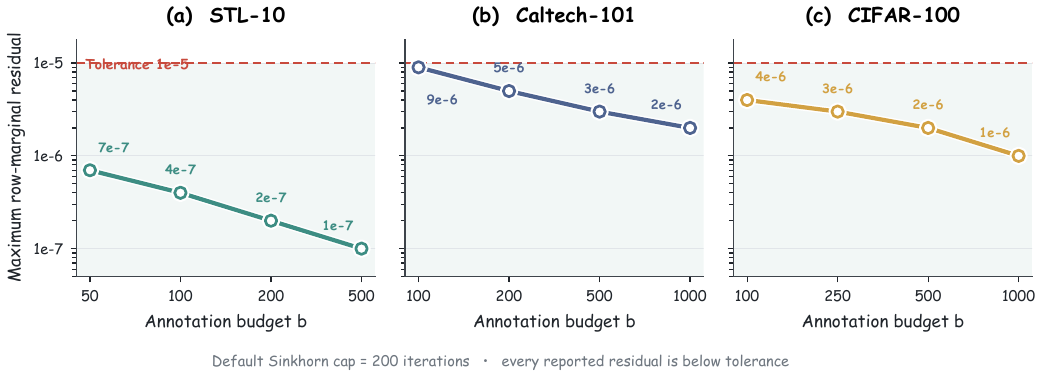}
\caption{Sinkhorn row-marginal residuals under the default cap
$T_{\mathrm{sk}}=200$. The panels use the dataset-specific budget
grids shown on their horizontal axes. Each point reports
$\max_i|\sum_k\pi^\star_{\eps^\star}(i,k)-\mu_U(i)|$ for the corresponding
budget. The dashed line marks a residual of $10^{-5}$.}
\label{fig:numerics_app}
\end{figure}

All reported residuals are below $10^{-5}$ and decrease as the
budget grows on each of the three datasets. The largest value is
$9\!\cdot\!10^{-6}$ on Caltech-101 at $b=100$, while the STL-10
residuals remain below $10^{-6}$ throughout. The exact equivariance
and asymptotic results concern the ideal rule
$\eps=c\,\widehat m(b)$ and the exact EOT optimum. The floor and
stopping tolerance above provide numerical safeguards. Cost-scale
equivariance is preserved by rescaling the floor together with the
cost. In an asymptotic implementation, the corresponding sufficient
condition is $\eps_{\min}=o(m_P(b))$.

\subsection{Runtime and Per-Cell Wall-Clock Breakdown}
\label{app:complexity}

The dominant selection cost of $\eps$-AS is the $k$-means step,
which is shared with TypiClust, BADGE-cold, and USL. The measured
Sinkhorn stage takes roughly one tenth as long as $k$-means in each
of the two breakdowns below.
Figure~\ref{fig:wallclock_breakdown_app} reports the per-cell
wall-clock breakdown on Caltech-101 (at $b{=}1000$) and on ImageNet-1k
(at $b{=}2000$) on a single NVIDIA H200 GPU.

\begin{figure}[H]
\centering
\includegraphics[width=\linewidth]{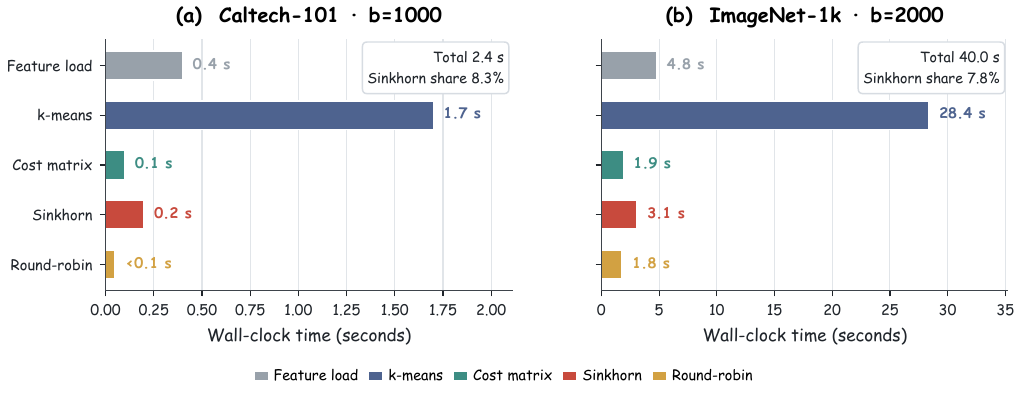}
\caption{Per-cell wall-clock breakdown of $\eps$-AS on Caltech-101
(at $b{=}1000$) and ImageNet-1k (at $b{=}2000$).
The $k$-means step includes $k$-means$++$ initialization followed
by twenty-five Lloyd iterations. The Sinkhorn step uses the
log-domain schedule of Appendix~\ref{app:numerics}. The total is
end-to-end and includes I/O of the cached features.}
\label{fig:wallclock_breakdown_app}
\end{figure}

The Sinkhorn iteration accounts for $8.3\%$ and $7.8\%$ of total
wall-clock in the two reported cells, respectively. The $k$-means
step is the largest component, accounting for $1.7$ seconds of the
$2.4$-second total on Caltech-101 and $28.4$ seconds of the
$40.0$-second total on ImageNet-1k. In comparison, Sinkhorn takes
only $0.2$ and $3.1$ seconds. The dominant stage is therefore shared
with several clustering-based baselines.

\begin{figure}[H]
\centering
\includegraphics[width=\linewidth]{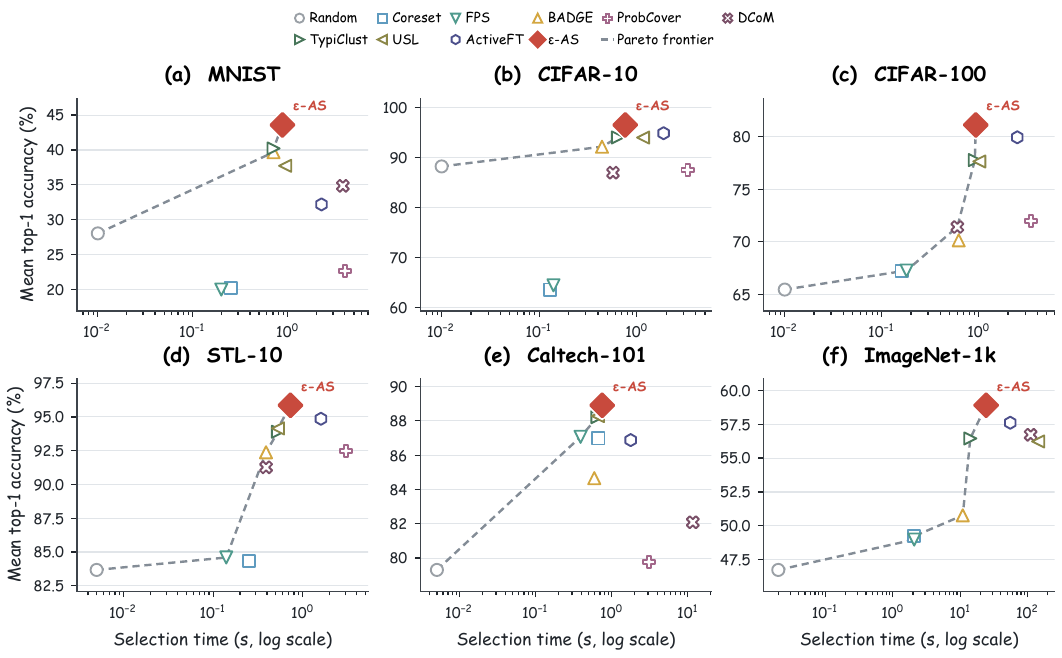}
\caption{Accuracy and selection-time trade-offs averaged over the
cold-start budgets and three seeds. The horizontal axis uses a
logarithmic scale, and the dashed curve traces the empirical Pareto
frontier. Selection time excludes the shared feature-extraction pass.
ProbCover is omitted from the ImageNet-1k panel because its dense
implementation is infeasible at that scale.}
\label{fig:time_all}
\end{figure}

Figure~\ref{fig:time_all} complements the two per-cell breakdowns
with selection-only time averaged across budgets. $\eps$-AS remains
below one second on the five smaller datasets and is faster than
ActiveFT on all six datasets. On ImageNet-1k, it raises mean accuracy
from $57.61\%$ to $58.90\%$ while reducing selection time from
$55.41$ to $24.28$ seconds, a $2.3\times$ speedup. It lies on the
accuracy-time Pareto frontier in every panel.

\subsection{Extended Ablations}
\label{app:ext_ablation}

The main paper reports the full method and five component ablations
on CIFAR-100 and Caltech-101. Figure~\ref{fig:abl_extended_app}
extends the same controlled comparison to STL-10 and CIFAR-10 using
a common protocol for every variant.

\begin{figure}[H]
\centering
\includegraphics[width=\linewidth]{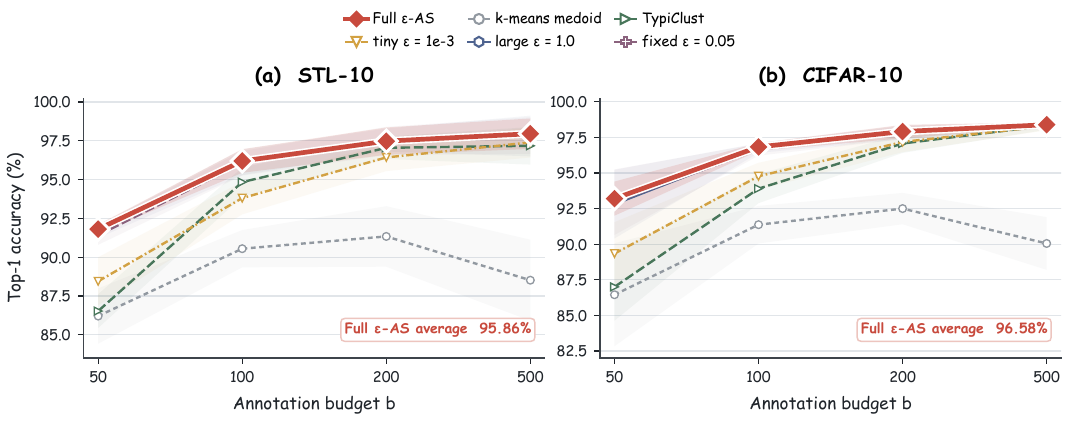}
\caption{Component ablations on STL-10 and CIFAR-10. Lines report
mean top-1 accuracy across three seeds, and shaded bands show one
standard deviation.}
\label{fig:abl_extended_app}
\end{figure}

Full $\eps$-AS has the highest budget-averaged accuracy on both
datasets, with $95.86\%$ on STL-10 and $96.58\%$ on CIFAR-10. It
leads three of four STL-10 budgets and all four CIFAR-10 budgets. At
$b=200$ on STL-10, the large fixed temperature exceeds the full
method by $0.07$ percentage points but does not improve the
dataset-level average. The medoid variant degrades at the largest
budget on both datasets, while TypiClust and the fixed-temperature
variants narrow the gap at larger budgets. The pattern supports the
cross-dataset stability of the normalized adaptive scale.

\FloatBarrier

\FloatBarrier
\twocolumn
\phantomsection
\addcontentsline{toc}{section}{References}
\bibliography{OT-CSAL}

\end{document}